\documentclass[a4paper,10pt,reqno]{amsart}
\usepackage[T1]{fontenc}
\usepackage{lmodern}
\usepackage[utf8]{inputenc}

\usepackage{amsmath,amssymb,graphicx, mathtools} 
\usepackage{amsthm,amsaddr}
\usepackage[a4paper,left=2.5cm,right=2.5cm,top=3cm,bottom=3cm]{geometry}
\usepackage{enumerate} 
\usepackage{enumitem}
\usepackage{tikz} 
\usetikzlibrary{arrows,decorations.markings}
\usepackage[colorlinks, linkcolor=red, urlcolor=blue, citecolor=blue]{hyperref}
\usepackage{url}
\usepackage{comment}
\usepackage{caption}
\usepackage{subcaption}

\DeclareMathOperator*{\argsup}{arg\,sup}

\newcommand{\norm}[1]{\lVert#1\rVert}

\theoremstyle{plain}
\newtheorem{thm}{Theorem}[section]
\newtheorem{cor}[thm]{Corollary}
\newtheorem{lem}[thm]{Lemma}
\newtheorem{prop}[thm]{Proposition}

\theoremstyle{definition}

\theoremstyle{remark}
\newtheorem{rmk}[thm]{Remark}

\numberwithin{equation}{section}

\makeatletter
\renewcommand\subsection{\@startsection{subsection}{2}%
  \z@{-.5\linespacing\@plus-.7\linespacing}{.5\linespacing}%
  {\normalfont\scshape}}
\renewcommand\subsubsection{\@startsection{subsubsection}{3}%
  \z@{.5\linespacing\@plus.7\linespacing}{-.5em}%
  {\normalfont\scshape}}
\makeatother

\makeatletter
\@namedef{subjclassname@2020}{%
  \textup{2020} Mathematics Subject Classification}
\makeatother

\usepackage[format=plain,
            labelfont=sc]{caption}
\usepackage[toc,page]{appendix} 
\usepackage{scalerel,stackengine}
\stackMath%
\newcommand\reallywidehat[1]{%
\savestack{\tmpbox}{\stretchto{%
  \scaleto{%
    \scalerel*[\widthof{\ensuremath{#1}}]{\kern.1pt\mathchar"0362\kern.1pt}%
    {\rule{0ex}{\textheight}}
  }{\textheight}%
}{2.4ex}}%
\stackon[-6.9pt]{#1}{\tmpbox}%
}

\makeatletter
\newsavebox\myboxA
\newsavebox\myboxB
\newlength\mylenA
\newcommand*\xoverline[2][0.75]{%
    \sbox{\myboxA}{$\m@th#2$}%
    \setbox\myboxB\null
    \ht\myboxB=\ht\myboxA%
    \dp\myboxB=\dp\myboxA%
    \wd\myboxB=#1\wd\myboxA
    \sbox\myboxB{$\m@th\overline{\copy\myboxB}$}
    \setlength\mylenA{\the\wd\myboxA}
    \addtolength\mylenA{-\the\wd\myboxB}%
    \ifdim\wd\myboxB<\wd\myboxA%
       \rlap{\hskip 0.5\mylenA\usebox\myboxB}{\usebox\myboxA}%
    \else
        \hskip -0.5\mylenA\rlap{\usebox\myboxA}{\hskip 0.5\mylenA\usebox\myboxB}%
    \fi}
\makeatother    

\usepackage{fancyhdr}

\newcommand\shortitle{}
\newcommand\name{}

\begin{document}

\title{Singular perturbations and hierarchical learning in two-layer neural networks}
\author{Cédric Gerbelot}
\address{Unité de Mathématiques Pures et Appliquées (UMPA), ENS Lyon}
\author{Jean-Christophe Mourrat}
\address{Unité de Mathématiques Pures et Appliquées (UMPA), ENS Lyon and CNRS}
\email{cedric.gerbelot-barrillon@ens-lyon.fr, jean-christophe.mourrat@ens-lyon.fr}
\thanks{JCM acknowledges the support of the ERC MSCA grant SLOHD (101203974).}
\date{\today}

\begin{abstract}
We study the population gradient flow of an infinitely wide two-layer neural network learning a misspecified single-index model in high dimension. The two layers are optimized 
jointly, with a perturbative parameter tuning the relative training speed between the first and second layer. This setting 
was considered by Berthier, Montanari and Zhou in \cite{berthier2024learning}, who conjectured a hierarchical learning scenario with explicit timescales as the second layer is trained faster than the first. 
In this paper, we prove that the constant and linear components of the hidden link function are indeed recovered within the predicted timescales, at sharp explicit thresholds. We then analyze the onset of learning of the quadratic component and show that the components learned at earlier stages continue to influence the dynamics in an essential way. Our proof is based on quantitative approximation results 
for singularly perturbed flows evolving near a manifold defined by integral constraints. At a phenomenological level, we also show that the empirical measure of the weights displays 
singular behaviour when reaching the quadratic component of the hidden link, with a small fraction of neurons growing significantly while the remaining ones rearrange to preserve the components already learned.
\end{abstract}

\maketitle

\section{Introduction}
In the effort to understand the success of gradient-based optimization methods for deep learning \cite{bengio2017deep}, it is crucial to find models 
that are amenable to theoretical analysis while retaining practical relevance. Among such models, two-layer neural networks remain a significant technical challenge 
despite a flurry of recent contributions: in the absence of convexity, it is difficult to find analytically tractable approximations 
of the dynamical system induced by a gradient flow on the parameters. In that regard, several approaches have appeared to go beyond the kernel regime \cite{jacot2018neural,du2019gradient} which, 
despite the appealing simplicity of the ensuing (linear) dynamics, falls short of explaining the ability of neural networks to perform feature learning \cite{chizat2019lazy,woodworth2020kernel,geiger2020disentangling}. Notably, 
kernel-based analyses of neural network training dynamics do not show how these architectures adapt to latent low-dimensional structure in the data, thereby avoiding the curse of dimensionality \cite{bach2017breaking,yehudai2019power,ghorbani2020neural}. We summarize in what follows, without claim of exhaustivity, two of these approaches that motivate the present work. \newline \indent
On the one hand, the series of papers \cite{mei2018mean,chizat2018global,rotskoff2022trainability}
proposes to study infinitely wide two-layer neural networks and shows that, upon an appropriate scaling of the network and parameter initialization, the evolution of the 
empirical measure of the weights may be described by a non-linear transport equation corresponding to a Wasserstein gradient flow on a convex functional. Although this approach 
allows one to provide significant generalization guarantees in certain settings \cite{chizat2020implicit}, the arguments often rely on infinite training time (see also \cite{wojtowytsch2020convergence}), and quantitative convergence results still largely rely on Polyak-Lojasiewicz assumptions, see e.g. \cite{chizat2022sparse}. \newline \indent
On the other hand, a vast body of work studies the high-dimensional dynamics of various instances of gradient methods on simple problems, with random design data models displaying latent low-dimensional structure. Examples include single \cite{yehudai2020learning,arous2021online,bietti2022learning,berthier2024learning} and 
multi-index regression \cite{goldt2019dynamics,damian2022neural,ba2022high,abbe2023sgd,arous2024high,dandi2024two,bietti2025learning}, as well as binary and multiclass classification \cite{refinetti2021classifying,ben2024high,glasgow2024sgd}. These approaches 
rely on reducing the high-dimensional dynamics to a low-dimensional (Euclidean) dynamical system capturing both the effect of empirical noise and population terms. Postponing a more detailed discussion to Section \ref{sec:related_works}, we 
mention for now that this has led to quantitative descriptions of a number of phenomena observed in practice such as hierarchical learning \cite{arpit2017closer}, where the neural network learns functions of increasing complexity on different timescales, with associated plateaus in the training and test error.    \newline \indent 
In the case of single index regression, if the link function is known, a single neuron 
is sufficient to learn the target function (i.e. the hidden index), and the gradient dynamics can be reduced to a one dimensional ordinary differential equation (ODE). The time for a parameter vector under gradient flow to achieve non-trivial correlation 
with the hidden index from an uninformative start is then governed by the local curvature of the loss landscape near the initial condition \cite{dudeja2018learning,arous2020algorithmic,arous2021online}. 
This, in turn, allows a sharp analysis of gradient methods with no oracle modification to the algorithm. When the activation function is not known however, a (fully parametrized) two-layer neural network 
is required to solve the corresponding semiparametric problem, and the ensuing effective dynamics is much harder to analyze. This has motivated the use of oracle-based proof methods such as layer-wise 
training \cite{daniely2020learning,bietti2022learning,lee2024neural}, or meta-algorithms operating directly in an appropriate Hilbert space \cite{pillaud2025joint}. Thus, understanding the joint learning dynamics of a two-layer 
neural network trained with gradient descent to solve a semiparametric regression task is still unresolved. \newline \indent 
Recently, the work \cite{berthier2024learning} took a significant step towards answering this problem. Blending elements from both the approaches described above, the authors 
study gradient flow on an infinitely wide two-layer neural network, in the high-dimensional limit, and in a regime where the second layer is trained faster than the first. The relative 
speed is controlled by a perturbative parameter that allows one to interpolate between the unchanged gradient flow and layerwise training. Focusing on the population limit, they derive solutions to the effective 
dynamics using a matched asymptotic expansion argument, a non-rigorous method from the theory of singularly perturbed dynamical systems \cite{holmes2012introduction}. These computations lead to a precise description of the hierarchical learning process for the constant 
and linear components of the unknown activation, while progressively recovering the hidden index, with explicit expressions for the corresponding timescales. An additional conjecture is then established for the
higher order components based on neglecting the terms associated to the ones learned previously and solving simplified dynamics that do not require a perturbative analysis. \newline \indent 
In this paper, we build a framework to rigorously study a family of singularly perturbed dynamical systems that includes the learning dynamics of two-layer neural networks,
providing mathematical grounding for the results of \cite{berthier2024learning} regarding the constant and linear components of the hidden activation along the way. We then provide an explicit representation for the learning dynamics of the non-linear components, showing that the terms associated to the previously learned ones cannot be neglected. At a phenomenological level, beyond the constant and linear components,  
the empirical measure of the weights displays singular behaviour as it moves towards learning the next components while retaining the information acquired at earlier stages of the learning process. At a technical level, 
we provide a general methodology to construct approximations of dynamical systems evolving near target sets defined by integral constraints, and to obtain time dependent bounds that adapt to the hierarchical nature of the dynamics for the validity of these
approximations.
\section{Setting}
Our setting is the same as that of \cite{berthier2024learning}, and we adopt the same notations. For the sake of clarity, we will not reproduce 
their main results explicitly or enumerate the assumptions required for them to hold. Rather, we will give an informal description of only the elements required for our results to be stated. We refer the reader to their paper for more precise statements, and to the second paragraph 
of Section~\ref{sec:related_works} for further discussion on this matter. \newline \indent
Consider the joint law $P_{x,y}$ on $\mathbb{R}^{d} \times \mathbb{R}$ defined by the conditional model
\begin{equation}
  y = \varphi(\langle u_{*},x \rangle),
\end{equation}
where $\varphi : \mathbb{R} \to \mathbb{R}$ is an unknown link function, $x$ is a random vector in $\mathbb{R}^{d}$ with i.i.d. standard normal entries and $u_{*} \in \mathbb{S}^{d-1}$ is a fixed, unknown vector on the unit sphere. We learn this 
target function using a two-layer neural network parametrized as follows 
\begin{equation}
  f(x,a,u) = \frac{1}{m}\sum_{i=1}^{m}a_{i}\sigma(\langle u_{i},x \rangle), \quad a_{1},\dots,a_{m} \in \mathbb{R}, \quad u_{1},\dots,u_{m} \in \mathbb{S}^{d-1},
\end{equation}
and where $\sigma : \mathbb{R} \to \mathbb{R}$ is an activation function.
The parameters of this network are trained with gradient flow on the population loss 
\begin{equation}
  \mathcal{R}(a,u) = \frac{1}{2}\mathbb{E}\left\{\left(y-f(x,a,u)\right)^{2}\right\},
\end{equation}
where we denote $(a,u) = (a_{1},...,a_{m},u_{1},...,u_{m})$ the vectors gathering the external and internal weights.
The gradient flow is considered on the product manifold $(\mathbb{R} \times \mathbb{S}^{d-1})^{m}$, so that a spherical gradient is considered for the internal weights, and an Euclidean gradient for the external ones, 
leading to the system of ordinary differential equations (ODEs)
\begin{align}
  \partial_t a_{i} &= -m\partial _{a_{i}}\mathcal{R}(a,u), \\
  \partial_{t}u_{i} &= -m\left(I_{d}-u_{i}u_{i}^{\top}\right)\nabla_{u_{i}}\mathcal{R}(a,u).
\end{align}
Additionally, the relative training speed between the first and second layer weights is tuned with a perturbative parameter $\varepsilon>0$, understood as small, as follows
\begin{align}
  \label{eq:base_syst_dyn1}
  \partial_t a_{i} &= -\frac{m}{\varepsilon}\partial _{a_{i}}\mathcal{R}(a,u), \\
  \partial_{t}u_{i} &= -m\left(I_{d}-u_{i}u_{i}^{\top}\right)\nabla_{u_{i}}\mathcal{R}(a,u).
  \label{eq:base_syst_dyn2}
\end{align}
Thus, as $\varepsilon$ becomes small, the second layer weights are tuned progressively faster than the first layer ones. At initialization, the vectors $a,u$ are assumed to be drawn i.i.d. and independently respectively from a measure $P_{A}$, assumed to be absolutely continuous w.r.t. the Lebesgue measure on $\mathbb{R}$, and the invariant measure on the d-dimensional unit sphere, denoted $\mathcal{U}(\mathcal{S}^{d-1})$.
The activations $\sigma,\phi$ are assumed to be square-integrable with respect to the standard Gaussian measure. The coefficients of their 
decompositions in Hermite polynomials are respectively denoted $\{\sigma_{k}\}_{k=0}^{\infty}$ and $\{\varphi_{k}\}_{k=0}^{\infty}$.
Letting, for each $1 \leq i \leq m$, $s_{i} = \langle w_{i},u_{*} \rangle$, the authors of \cite{berthier2024learning} show that, when taking $m,d$ to infinity, the dynamical system 
\eqref{eq:base_syst_dyn1}-\eqref{eq:base_syst_dyn2} can be reduced to a set of autonomous ODEs on the variables $\{(a_{i},s_{i})\}_{i=1}^{m}$. At initialization, the overlaps $\{s_{i}\}_{i=1}^{m}$ are 
then i.i.d. according to the distribution of the push-forward of $\mathcal{U}(\mathcal{S}^{d-1})$ through the Euclidean projection on $u_{*}$, denoted $P_{S}$. In turn, these variables may be reindexed 
using a probability space $(\Omega,\mathcal{F},\rho)$ and random processes $a(t,\omega),s(t,\omega)$ (for $\omega \in \Omega$ and $t \geq 0$), respectively taking values in $\mathbb{R}$ and $[-1,1]$, and following the evolution
\begin{align}
  \label{eq:main_mf_dyn1}
  \partial_{t}a(t,\omega) &= \frac{1}{\varepsilon}\sum_{k=0}^{\infty}\sigma_{k}s(t,\omega)^{k}\left(\varphi_{k}-\sigma_{k}\int a(t,\nu)s(t,\nu)^{k}d\rho(\nu)\right), \\
  \partial_{t}s(t,\omega) &= a(t,\omega)\left(1-s(t,\omega)^{2}\right)\sum_{k=1}^{\infty}k \sigma_{k}s(t,\omega)^{k-1}\left(\varphi_{k}-\sigma_{k}\int a(t,\nu)s(t,\nu)^{k}d\rho(\nu)\right),
  \label{eq:main_mf_dyn2}
\end{align}
which can be understood as a gradient flow on the cost function
\begin{equation}
  \frac{1}{2}\sum_{k=0}^{\infty}\left(\varphi_{k}-\sigma_{k}\int a(\omega)s(\omega)^{k}d\rho(\omega)\right)^{2}.
\end{equation}
The initial condition $a(0,\omega),s(0,\omega)$ is a pair of measurable maps such that pushing $\rho$ through them recovers $P_{A}\otimes P_{S}$, and,
since $P_{S}$ converges weakly to a point mass at $0$ as $d$ goes to infinity, we have $s(0,\omega) = 0$. Also, if we assume that $\rho$ is the uniform measure on $[0,1]$,
$a(0,\omega)$ is simply the inverse cumulative density function (CDF) of $P_{A}$.
We refer the reader to Section 4.3 and 4.4 of \cite{berthier2024learning}, and to the references therein, for a more detailed discussion on the links between this formulation and the one usually found 
in the mean-field theory of neural networks \cite{mei2019mean,chizat2018global,rotskoff2022trainability}.
The evolution \eqref{eq:main_mf_dyn1}-\eqref{eq:main_mf_dyn2} will be the main topic of our study, and, in particular, we seek to understand whether or not the dynamics reaches points where the terms $\left(\varphi_{k}-\sigma_{k}\int a(\omega)s(\omega)^{k}d\rho(\omega)\right)^{2}$ are small,
for increasing values of $k$. \newline 
\paragraph{\textbf{Notations}}
We will denote the $L^{\infty}$ norm on $[0,1]$ (or other subsets of the real line) by $\norm{.}$, and the Euclidean norm on $\mathbb{R}^{n}$ (for any integer $n$) by $\norm{.}_{2}$. Also,
for any piecewise continuous, integrable function $f$ on $\Omega$, we denote $\overline{f} = \int_{\Omega}f(\omega)d\rho(\omega)$, and $f_{\perp} = f-\overline{f}$. We will 
sometimes abbreviate the centered, standard Gaussian measure with $\mu$. Finally, for any real number $x$, we let $(x)_{+} = \max(0,x)$.
\newline \indent 
\begin{rmk}
In what follows, various initial value problems related to the evolution \eqref{eq:main_mf_dyn1}-\eqref{eq:main_mf_dyn2} will be considered, always for initial conditions 
around which the vector fields defining the dynamics are locally Lipschitz continuous. Thus, we will often omit to justify the existence and uniqueness of maximal solutions, which immediately follow 
from the Cauchy-Lipschitz theorem.
\end{rmk}
\section{Main results}
\label{sec:main_results}
This section gathers the results that are the most meaningful to describe the learning dynamics in an intuitive way. They are consequences of more general 
technical results that can be found in Section \ref{sec:proofs} and Appendix \ref{sec:app_B}. Our assumptions are as follows: \newline
\paragraph{\textbf{Assumption 1}}
\label{par:ass1}
  \begin{itemize}
    \item The function $a(0,.)$, which will be denoted $a_{init}$, is piecewise continuous from $[0,1]$ to $\mathbb{R}$ and there exists a constant $M$ such that $\norm{a(0,.)} \leq M$.
    \item $s(0,.)=0$. 
    \item The functions $\sigma,\sigma'$ and $\varphi$ are in $L^{2}(\mu)$. For each $k \geq 0$, $\sigma_{k}>0$ and $\varphi_{k} >0$. \\
  \end{itemize}
We will sometimes refer to the parameters $M,\{\sigma_{k}\}_{k=0}^{\infty},\{\varphi_{k}\}_{k=0}^{\infty}$ as the problem \emph{data}.
Until further notice, all the following statements concern the solution $a(t,.),s(t,.)$ to the initial value problem defined by the evolution 
\eqref{eq:main_mf_dyn1}-\eqref{eq:main_mf_dyn2}, with any initial condition such that Assumption 1 is verified. 
\subsection{Learning the constant and linear components}
Our first main result shows that the timescale on which the constant component (the coefficient $\varphi_{0}$) is learned to a precision close to $\varepsilon$ is of order $\frac{\varepsilon\log(\frac{1}{\varepsilon})}{\sigma_{0}^{2}}$.
\begin{prop}[Learning the constant component]
\label{prop:main_0}
There exists a positive constant $C_{0}$ such that, for any $\eta \in (0,1)$ and $\varepsilon$ sufficiently small, 
\begin{equation}
  \inf \left\{t \geq 0: \vert \overline{a}(t)- \frac{\varphi_{0}}{\sigma_{0}} \vert \leq \varepsilon^{\eta}+C_{0}\varepsilon\log(\frac{1}{\varepsilon})\right\} \leq \frac{\left(\eta \log(\frac{1}{\varepsilon})+\log(\vert \overline{a}(0)-\frac{\varphi_{0}}{\sigma_{0}} \vert)\right)_{+}}{\sigma_{0}^{2}}\varepsilon
\end{equation}
\end{prop}
The next result characterizes the timescale required to learn the linear component (the coefficient $\varphi_{1}$). The first part of the result 
shows that learning $\varphi_{1}$ up to a precision of order $o(1)$ is arbitrarily close to $\frac{1}{4\sigma_{1}\varphi_{1}}\sqrt{\varepsilon}\log(\frac{1}{\varepsilon})$. 
The second part shows that, on a longer timescale of order $\mathcal{O}(\sqrt{\varepsilon}\log(\frac{1}{\varepsilon}))$, the coefficient is learned up to a precision of $\varepsilon^{1/4}$.
\begin{prop}[Learning the linear component]
\label{prop:main_1}
  For any constant $c_{1}>0$, $\alpha \in (0,1)$ and $\varepsilon$ sufficiently small, 
  \begin{equation}
    \inf \left\{t \geq 0 : \overline{as}(t) = (1-\alpha)\frac{\varphi_{1}}{\sigma_{1}}\right\} \leq \left(\frac{1}{4\sigma_{1}\varphi_{1}}+c_{1}\right)\sqrt{\varepsilon}\log(\frac{1}{\varepsilon}).
  \end{equation}
  Furthermore, there exists a constant $C_{1}>\frac{1}{4\sigma_{1}\varphi_{1}}$ such that, for $\varepsilon$ sufficiently small, 
  \begin{equation}
    \inf \left\{t \geq 0 : \overline{as}(t) = (1-\varepsilon^{1/4})\frac{\varphi_{1}}{\sigma_{1}}\right\} \leq C_{1}\sqrt{\varepsilon}\log(\frac{1}{\varepsilon}).
  \end{equation}
\end{prop}
Propositions \ref{prop:main_0} and \ref{prop:main_1} are corollaries of the following more general result, that provides a quantitative approximation to the solution of
the dynamical system \eqref{eq:main_mf_dyn1}-\eqref{eq:main_mf_dyn2} on timescales up to order $\mathcal{O}(\sqrt{\varepsilon}\log(\frac{1}{\varepsilon}))$. To state this result,
we begin by defining the idealized system around which our approximation is constructed. Let $\overline{b}:\mathbb{R}_{+} \to \mathbb{R}$ be the function 
\begin{equation}
  \label{eq:b_bar}
  \overline{b}(t) = \frac{\varphi_{0}}{\sigma_{0}}+\left(\overline{a}(0)-\frac{\varphi_{0}}{\sigma_{0}}\right)\exp(-\frac{\sigma_{0}^{2}t}{\varepsilon}).
\end{equation}
$\overline{b}$ is the solution to the ODE obtained for the quantity $\overline{a}$ if only the term corresponding to the coefficient $\varphi_{0}$
is conserved in the dynamical system \eqref{eq:main_mf_dyn1}-\eqref{eq:main_mf_dyn2}. Similarly, with $\sigma'_{1} = \sigma_{1}\int_{0}^{1}a_{init,\perp}^{2}(\nu)d\nu$, let $\beta,\gamma$ be the solutions to the dynamical system 
\begin{align}
  \label{eq:main_bg1}
  \partial_{t}\beta(t) &= \frac{\sigma_{1}}{\varepsilon}\gamma(t)\left(\varphi_{1}-\sigma'_{1}\beta(t)\gamma(t)\right), \\
  \partial_{t}\gamma(t) &= \sigma_{1}\beta(t)\left(\varphi_{1}-\sigma'_{1}\beta(t)\gamma(t)\right),
  \label{eq:main_bg2}
\end{align}
with initial condition $\beta(0) = 1$ and $\gamma(0)=0$. This system can be integrated (almost) explicitly, and $\beta,\gamma$ are studied in detail in Appendix \ref{sec:app_A}. Now, define $b(t,.) = \beta(t)a_{init,\perp}(.)$ and $g(t,.) = \gamma(t)a_{init,\perp}(.)$.
The functions $b,g$ are the solution to the dynamical system \eqref{eq:main_mf_dyn1}-\eqref{eq:main_mf_dyn2} if only the terms corresponding to the coefficient $\varphi_{1}$ are kept. The following 
theorem provides a quantitative approximation result for $a(t,.),s(t,.)$ in the form of tubes with time-dependent radii around the curves defined by $\overline{b},b,g$.
\begin{thm}
  \label{thm:main_tube01}
  There exists a constant $D$ such that, for any $c \in [0,\frac{1}{8D}]$ and time $(t_{\varepsilon})_{\varepsilon>0}$ of the form 
  \begin{equation}
    t_{\varepsilon} = \left(\frac{1}{4\sigma_{1}\varphi_{1}}+c\right)\sqrt{\varepsilon}\log(\frac{1}{\varepsilon}),
  \end{equation}
  there exists a constant $C$ such that, for any $t \in [0,t_{\varepsilon}]$ and $\varepsilon$ sufficiently small,
  \begin{align}
    \label{eq:main_tube_mean}
    \vert \overline{a}-\overline{b} \vert(t) \leq Ct+C\varepsilon\beta^{2}(t), \quad \vert \overline{s}(t) \vert \leq Ct, \\
    \norm{(a_{\perp}-b)(t,.)}+\varepsilon^{-1/2}\norm{(s_{\perp}-g)(t,.)} \leq C\varepsilon^{\frac{1}{4}-cD}\beta(t).
    \label{eq:main_tube_perp}
  \end{align}
  Furthermore, there exists a constant $\delta \in (0,\frac{1}{8})$ such that, for any constant $K\geq\frac{1}{4\sigma_{1}\varphi_{1}}+\frac{1}{8D}$ and 
  any $t \in [t_{\varepsilon},K\sqrt{\varepsilon}\log(\frac{1}{\varepsilon})]$, and $\varepsilon$ sufficiently small, \eqref{eq:main_tube_mean} holds and
   \begin{align}
    \norm{(a_{\perp}-b)(t,.)}+\varepsilon^{-1/2}\norm{(s_{\perp}-g)(t,.)} \leq C\varepsilon^{\delta}\beta(t).
  \end{align}
\end{thm}
\begin{rmk}
  \label{rmk:bg_comp}
  One can show that, for all times, $\gamma(t) \leq \sqrt{\varepsilon}\beta(t)$, 
  justifying the sole use of $\beta$ to state the bounds on both $a_{\perp}$ and $s_{\perp}$ in Theorem \ref{thm:main_tube01}.
\end{rmk}
\begin{rmk}
  \label{rmk:time_adapt_bound}
  On timescales of order $\mathcal{O}(\varepsilon\log(\frac{1}{\varepsilon}))$, the terms depending on $\beta(t)$ are negligible, so that 
  the bounds relevant for the constant component are linear in time. As we progress towards the linear component, the dependence in $\beta(t)$ 
  becomes relevant, and we rigorously recover the corresponding timescales from \cite{berthier2024learning} down to the constants. 
\end{rmk}
\subsection{Learning the quadratic component}
In the remainder of this section, we will focus on the following simplification of the dynamical system \eqref{eq:main_mf_dyn1}-\eqref{eq:main_mf_dyn2}:
\begin{align}
  \label{eq:main_quad_dyn1}
  \partial_{t} a(t,\omega) &= \frac{\sigma_{0}}{\varepsilon}\left(\varphi_{0}-\sigma_{0}\overline{a}(t)\right)+\frac{\sigma_{1}}{\varepsilon}s(t,\omega)\left(\varphi_{1}-\sigma_{1}\overline{as}(t)\right)+\frac{\sigma_{2}\varphi_{2}}{\varepsilon}s^{2}(t,\omega), \\
  \partial_{t} s(t,\omega) &= a(t,\omega)\sigma_{1}\left(\varphi_{1}-\sigma_{1}\overline{as}(t)\right)+2\sigma_{2}\varphi_{2}a(t,\omega)s(t,\omega).
  \label{eq:main_quad_dyn2}
\end{align}
We now turn to the learning process of the quadratic part of $\varphi$, namely the coefficient $\varphi_{2}$. A core difficulty in this case is the absence 
of a clear decomposition of the solution to \eqref{eq:main_mf_dyn1}-\eqref{eq:main_mf_dyn2}, mainly due to the previously learned coefficients $\varphi_{0},\varphi_{1}$, the contribution of which in \eqref{eq:main_mf_dyn1}-\eqref{eq:main_mf_dyn2} does not vanish once they have been learned. Thus, we determine the 
correct idealized system to study the learning dynamics on a relevant timescale (in powers of $\varepsilon$). Such a timescale is identified in Lemma \ref{lem:pure_quad_sol} by studying the 
solution to \eqref{eq:main_mf_dyn1}-\eqref{eq:main_mf_dyn2} when only the terms corresponding to $\varphi_{2}$ are kept, showing that whatever approximation is obtained should hold for times of order at least $\varepsilon^{1/4}$.
We begin by describing the family of initial conditions that we will consider, as prescribed by the terminal values of the pair $a(t,.),s(t,.)$ in Theorem \ref{thm:main_tube01} (and its corollaries). In particular, upon
setting a time of the form $t_{0} = K\sqrt{\varepsilon}\log(\frac{1}{\varepsilon})$, which we will choose as our origin, in Corollary \ref{cor:stable_01}, we obtain that there exist positive constants $C$ and $\delta$ such that 
\begin{align}
    \label{eq:main_init_quad1}
    &\vert \varphi_{0}-\sigma_{0}\overline{a}(t_{0}) \vert \leq C\sqrt{\varepsilon}, \quad \vert \varphi_{1}-\sigma_{1}\overline{as}(t_{0}) \vert \leq C\varepsilon^{1/4}, \quad \vert \overline{s}(t_{0}) \vert \leq C\sqrt{\varepsilon}\log(\frac{1}{\varepsilon}), \\
    &\norm{a_{\perp}(t_{0},.)-\left(\frac{\varphi_{1}}{\sigma_{1}'}\right)^{1/2}\varepsilon^{-1/4}a_{init,\perp}(.)}+\varepsilon^{-1/2}\norm{s_{\perp}(t_{0},.)-\left(\frac{\varphi_{1}}{\sigma_{1}'}\right)^{1/2}\varepsilon^{1/4}a_{init,\perp}(.)} \leq C\varepsilon^{-1/4+\delta}.
    \label{eq:main_init_quad2}
\end{align}
Owing to the initial scale of $a,s$ prescribed above (respectively $\varepsilon^{-1/4}$ and $\varepsilon^{1/4}$), and the fact that the relevant timescale on which we seek to work is of order $\varepsilon^{1/4}$, it is convenient to consider the following change in time and scale to state our 
approximation result:
 \begin{align}
    \label{eq:main_change_scale_a}
    \hat{a}(t,\omega) = \varepsilon^{1/4}a(\varepsilon^{1/4} t,\omega), \\
    \hat{s}(t,\omega) = \varepsilon^{-1/4}s(\varepsilon^{1/4} t,\omega).
    \label{eq:main_change_scale_s}
  \end{align}
Our goal now becomes to study the pair $\hat{a},\hat{s}$ on a timescale of order one. 
Now, as done for the constant and linear component, we define our idealized dynamics, in this case as the solution to the system
  \begin{align}
    \label{eq:main_quad_ideal_dyn1}
    \partial_{t}z_{1}(t,\omega) &= \sigma_{0}\alpha_{0}(t)+\sigma_{1}\alpha_{1}(t)z_{2}(t,\omega)+\sigma_{2}\varphi_{2}z_{2}^{2}(t,\omega), \\
    \partial_{t}z_{2}(t,\omega) &= z_{1}(t,\omega)\left(\sigma_{1}\alpha_{1}(t)+2\sigma_{2}\varphi_{2}z_{2}(t,\omega)\right),
    \label{eq:main_quad_ideal_dyn2}
  \end{align}
with any initial condition verifying $\overline{z}_{1}(t_{0}) = \varepsilon^{1/4}\frac{\varphi_{0}}{\sigma_{0}}$ and $\overline{z_{1}z_{2}}(t_{0}) = \frac{\varphi_{1}}{\sigma_{1}}$, and where the coefficients $\alpha_{0}(t),\alpha_{1}(t)$ are given by  
\begin{align}
  \alpha_{0}(t) &= -\frac{1}{\sigma_{0}}\left(\sigma_{1}\alpha_1(t)\overline{z_{2}}(t)+\sigma_{2}\varphi_{2}\overline{z_{2}^{2}}(t)\right), \\
  \alpha_{1}(t) &= -\frac{2\sigma_{2}\varphi_{2}\overline{z_{1}^{2}z_{2}}+\sigma_{2}\varphi_{2}\overline{z_{2}^{3}}-\sigma_{2}\varphi_{2}\overline{z_{2}^{2}}\overline{z_{2}}}{\sigma_{1}\overline{z_{1}^{2}}+\sigma_{1}(\overline{z_{2}^{2}}-\overline{z_{2}}^{2})},
\end{align}
and respectively ensure that 
\begin{equation}
  \partial_{t} \overline{z_{1}}(t) = 0, \quad \mbox{and} \quad \partial_{t} \overline{z_{1}z_{2}}(t) = 0.
\end{equation}
The following result is a consequence of the more general Proposition \ref{prop:approx_dyn_set} (and of Corollary \ref{cor:nn_approx_dyn}), which provides an approximation 
result for families of singularly perturbed flows evolving near certain target sets defined by integral constraints. 
\begin{prop}
  \label{prop:quad_approx}
  Let $(a(t,.),s(t,.))$ denote the solution to \eqref{eq:main_quad_dyn1}-\eqref{eq:main_quad_dyn2} with any initial condition $a(t_{0},.),s(t_{0},.)$ verifying the bounds \eqref{eq:main_init_quad1}-\eqref{eq:main_init_quad2}. We 
  then have the following results.
  \begin{enumerate}[font={\bfseries},label={(\alph*)},ref=\theprop(\alph*)]
  \item \label{prop:quad_approxa} There exist continuous functions $z_{1}(t_{0},.),z_{2}(t_{0},.)$ and a constant $C_{0}$ such that
  \begin{align}
    &\overline{z_{1}}(t_{0}) = \varepsilon^{1/4}\frac{\varphi_{0}}{\sigma_{0}}, \quad \overline{z_{2}}(t_{0}) = 0, \quad \overline{z_{1}z_{2}}(t_{0}) = \frac{\varphi_{1}}{\sigma_{1}}, \\
    &\norm{\hat{a}(t_{0},.)-z_{1}(t_{0},.)}+\norm{\hat{s}(t_{0},.)-z_{2}(t_{0},.)} \leq C_{0}\varepsilon^{1/4}\log(\frac{1}{\varepsilon}).
  \end{align}
  \item \label{prop:quad_approxb} Let $(z_{1}(t,.),z_{2}(t,.))$ be the solution to \eqref{eq:main_quad_ideal_dyn1}-\eqref{eq:main_quad_ideal_dyn2} with initial condition $(z_{1}(t_{0},.),z_{2}(t_{0},.))$.
  For any constant $R>0$, consider the hitting time 
  \begin{equation}
    T_{R} = \inf\{t \geq t_{0} : (\norm{\hat{a}(t,.)}+\norm{\hat{s}(t,.)}) \vee(\norm{z_{1}(t,.)}+\norm{z_{2}(t,.)}) \geq R\}.
  \end{equation}
Then there exist constants $C,K$ depending only on the problem data and $R$ such that, for any fixed time horizon $T$ independent of $\varepsilon$ and any $t_{0} \leq t \leq T_{R} \wedge T$,
\begin{equation}
  \label{eq:main_err_approx_quad}
  \norm{\norm{(\hat{a}(t,.),\hat{s}(t,.))-(z_{1}(t,.),z_{2}(t,.))}_{2}}_{\infty} \leq \left(C_{0}\varepsilon^{1/4}\log(\frac{1}{\varepsilon})+C\varepsilon^{1/8}(t-t_{0})\right)\exp\left(K(t-t_{0})\right).
\end{equation}
\end{enumerate}
\end{prop}
At an intuitive level, the perturbative parameter $\varepsilon$ allows to pin the dynamics in a neighborhood of the sets defined by the previously learned components. While the auxiliary dynamical system \eqref{eq:main_quad_ideal_dyn1}-\eqref{eq:main_quad_ideal_dyn2} is difficult to study quantitatively, 
it precisely characterizes the effect of the previously learned coefficients on the remainder of the learning dynamics. In particular, it shows 
that the effect of these coefficients cannot be ignored, and that, upon reaching the quadratic term, the dynamics \eqref{eq:main_mf_dyn1}-\eqref{eq:main_mf_dyn2}
is not equivalent to the one obtained by keeping only the terms related to the coefficient $\varphi_{2}$. Through simulations, we show that 
the maps $a(t,\omega)$ and $s(t,\omega)$ become highly singular, as the corresponding empirical measures evolve towards learning the 
quadratic component, but must also retain values of $\overline{a},\overline{as}$ respectively close to $\frac{\varphi_{0}}{\sigma_{0}},\frac{\varphi_{1}}{\sigma_{1}}$.
\begin{rmk}
 A straightforward consequence of the bounds for the initial condition at Eq.\eqref{eq:main_init_quad1}-\eqref{eq:main_init_quad2} is that 
 there exists a constant $D$ depending only on $\varphi_{0},\varphi_{1},\sigma_{0},\sigma_{1},M$ such that 
 \begin{equation}
  (\norm{\hat{a}(t_{0},.)}+\norm{\hat{s}(t_{0},.)}) \vee(\norm{z_{1}(t_{0},.)}+\norm{z_{2}(t_{0},.)}) \leq D.
 \end{equation}
Thus, Proposition \ref{prop:quad_approxb} becomes informative when choosing $R > D$, and describes the evolution of $\hat{a},\hat{s}$ (or equivalently $a,s$) at least until they reach a large, order one multiple of their initial scale. The meaning of this hitting time and the required scale for 
 the pair $a,s$ to learn the coefficient $\varphi_{2}$ are discussed further in Subsection \ref{sec:proof_quad}.
\end{rmk}
\begin{rmk}
  \label{rmk:quad_approx2}
  While the error bound in Eq.\eqref{eq:main_err_approx_quad} is exponential in time, we argue that this does not constitute a technical bottleneck. 
  The error is governed by how close the original system \eqref{eq:main_quad_dyn1}-\eqref{eq:main_quad_dyn2} evolves to the sets defined by the learned
  coefficients $\varphi_{0},\varphi_{1}$, and by the possibility to construct initial conditions that exactly satisfy those constraints while 
  remaining close to the true initial condition, as done in Proposition \ref{prop:quad_approxa}. This is formalized in Lemma \ref{lem:quad_approx}, of which 
  Proposition \ref{prop:quad_approx} is a straightforward consequence. More precisely, as soon as the exponential factor in Eq.\eqref{eq:main_err_approx_quad} becomes 
  larger than a chosen threshold, we may restart the argument, provided the resulting error is negligible compared to $\hat{a},\hat{s}$ 
  (here they are of order one and the error of order $\varepsilon^{1/8}$). The local nature of the bounds in Proposition \ref{prop:approx_dyn_set} 
  also allows one to adapt the approximation bound \eqref{eq:main_err_approx_quad} to whatever scale is reached by $\hat{a},\hat{s}$. However, 
  this method rests on a quantitative understanding of the auxiliary dynamics \eqref{eq:main_quad_ideal_dyn1}-\eqref{eq:main_quad_ideal_dyn2}, which 
  remains challenging.
\end{rmk}
\section{Further related works}
\label{sec:related_works}
\paragraph{\textbf{Hierarchical learning}}
Hierarchical learning in two-layer neural networks is an active area of research within machine learning theory, with a significant part of the effort directed towards analytically tractable models 
that reproduce this phenomenon. Building on insights gathered on the well-specified single index problem \cite{yehudai2020learning,arous2021online,tan2023online} and the analysis of a
single step of gradient descent \cite{damian2022neural,ba2022high}, several works have given quantitative descriptions of hierarchical learning, taking multi-index 
models as a benchmark problem. Using layerwise training, the work \cite{abbe2023sgd} provides sample complexity guarantees for the recovery of the indices with online 
stochastic gradient descent (SGD) and data uniformly supported on the hypercube. The plateaus in training and test error are associated with the incremental recovery of the 
support of the index, with dimension-dependent timescales governed by the degree (in the sense of a polynomial decomposition) of the components of the link function carrying 
subspaces spanned by the hidden multi-index. Similar guarantees are proposed in \cite{dandi2024two} for Gaussian data with large gradient steps and correspondingly large data batches, 
also with layerwise training. Finally \cite{bietti2025learning} gives a richer description of the same hierarchical learning phenomenon using the continuous time and population limit
of a two-stage algorithm combining successive nonparametric regression steps with gradient steps. The phenomenology described in the present paper is different: the incremental learning
is not associated with the latent subspace being supported on components of increasing degree, but rather with the successive learning of the components themselves while non-trivial 
correlation with the hidden index is achieved as soon as the linear component is learned, and continues to increase thereafter. This simultaneous fitting of the coefficients of $\varphi$ and 
progressive recovery of the hidden index presents a core technical difficulty absent from the aforementioned works. Finally, other models 
for incremental learning have been studied, notably diagonal linear networks \cite{woodworth2020kernel,pesme2023saddle,berthier2023incremental} and deep linear networks \cite{saxe2014exact,jacot2021saddle}. 
The resulting dynamics are significantly different from that considered in the present paper, and the models are not directly comparable. \newline
\paragraph{\textbf{Mean-field theory of neural networks}}
Following the identification of the mean field limit for infinitely wide neural networks in \cite{chizat2018global,mei2018mean,rotskoff2022trainability}, several works 
have proposed quantitative analysis of the resulting trajectories for models trained on synthetic data involving low-dimensional latent representations, or some degree of symmetry. The work 
\cite{abbe2022merged} studies the recovery of multi-index models with input data uniformly distributed on the hypercube. Using the invariances of the model, 
the authors reduce the mean field gradient flow (i.e. the transport equation for the empirical measure of the weights) to a dimension free one and 
provide guarantees for the recovery of the support of the hidden indices. For two-layer networks with ReLU activation, a similar use of symmetries is proposed in \cite{hajjar2023symmetries}
which establishes invariance properties of the resulting gradient flow and provides exponential convergence guarantees in a setting reducing to a linear model. Also, the works 
\cite{veiga2022phase,arnaboldi2023high} combine the approaches of \cite{saad1995dynamics,goldt2019dynamics} with quantitative bounds for the validity of the mean-field limit from \cite{mei2019mean}
to obtain low-dimensional representations (or dimension-free ones at the level of distributions) of the trajectories of online SGD for wide two-layer neural networks.  None of these works provide quantitative convergence guarantees or 
a detailed analysis of how a two-layer neural network solves a semi-parametric problem when trained with gradient flow. We note that the bounds from \cite{berthier2024learning} (see Corollary 1 and Proposition 3), based on propagation of chaos results from \cite{mei2019mean}, are insufficient to show that 
the mean-field limit is valid on the timescales describing the behaviour of the evolution \eqref{eq:main_mf_dyn1}-\eqref{eq:main_mf_dyn2} outside of regimes where the ambient dimension and number of neurons diverge exponentially with $\frac{1}{\varepsilon}$. In the context of machine learning, 
uniform-in-time bounds for propagation of chaos in single layer neural networks have recently been proposed in \cite{mahankali2023beyond,glasgow2025mean,glasgow2026uniform}, with \cite{mahankali2023beyond} carrying out a complete convergence analysis of the resulting dynamics for a specific instance of a single index target function. Their results do not apply directly to our setting, 
which involves a two-layer neural network and a singular perturbation. We refer the reader to the related works section of the latter papers for a more detailed discussion on propagation of 
chaos for neural networks, and leave the development of methods leading to uniform-in-time bounds for our setting to later work.\newline
\paragraph{\textbf{The two-timescale regime}}
Two-timescale algorithms and their connections to singular ordinary differential equations appear in \cite{borkar1997stochastic}, and 
have since then been used in stochastic approximation \cite{konda2004convergence}, reinforcement learning \cite{dalal2018finite} and optimization \cite{hong2023two}.
In the context of two-layer neural networks, the paper \cite{marion2023leveraging} provides a detailed study of the training dynamics 
of a specific instance of a two-layer neural network solving a univariate non-parametric regression problem. By only tuning the second layer weights and biases, 
the authors provide quantitative convergence guarantees for the two-timescale population gradient flow learning a piecewise constant function, with explicit bounds on the required number of neurons and 
relative speed parameter $\varepsilon$ for the convergence to hold. In particular, their work shows experimentally that the gradient flow may fail to converge without a timescale separation between the 
training of the outer layer weights and biases. On the one hand, our study is asymptotic in the number of neurons and, although explicit constants for the parameter $\varepsilon$ can be extracted from our proof 
method, we do not track them. On the other hand, our target function is high-dimensional and more complex, with the learning architecture more general and closer to practical ones.
\section{Simulations}
Letting $h_{k}$ denote the $k$-th normalized Hermite polynomial, we illustrate the results from Section \ref{sec:main_results} through simulations with the following activation and 
target link function.
\begin{align}
  \sigma(x) &= 0.4h_{0}(x)+0.5h_{1}(x)+0.2h_{2}(x), \\
  \varphi(x) &= h_{0}(x)+2h_{1}(x)+h_{2}(x).
\end{align}
We begin by verifying the separation of timescales for the learning of each of the components $\varphi_{0},\varphi_{1},\varphi_{2}$ with simulations reminiscent of those 
proposed in \cite{berthier2024learning}, before turning to the approximation results presented in Theorem \ref{thm:main_tube01} and Proposition \ref{prop:quad_approx}. In all figures, 
we choose $\Omega = [0,1]$ and $\rho$ the corresponding uniform measure, along with $a_{init}(\omega) = -1+2\omega$, the inverse CDF of the uniform measure on $[-1,1]$.
\subsection{Learning the constant and linear components}
Figures \ref{fig:harmo01_fig1}, \ref{fig:harmo01_fig2} and \ref{fig:harmo01_fig3} illustrate the results of Section \ref{sec:main_results} for the constant and linear components. For clarity of exposition, we simulate the dynamical system \eqref{eq:main_mf_dyn1}-\eqref{eq:main_mf_dyn2} with only the first two coefficients of the above activation and link function.
\begin{figure}[h]
\centering
\includegraphics[scale=0.31]{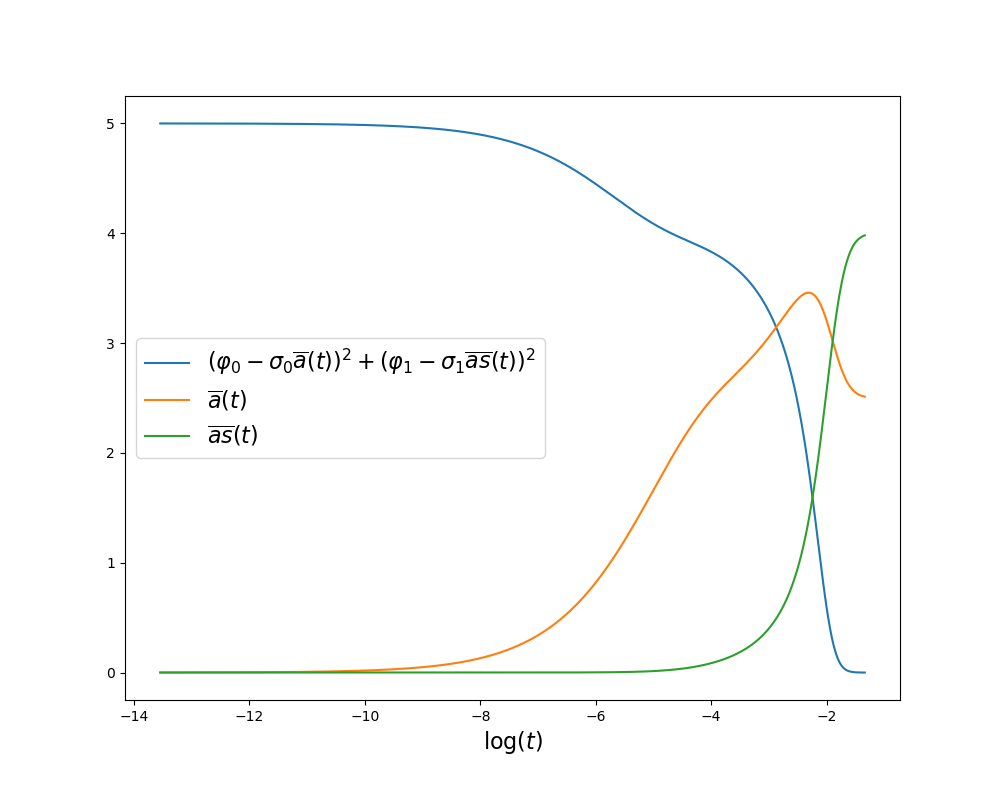}
\includegraphics[scale=0.31]{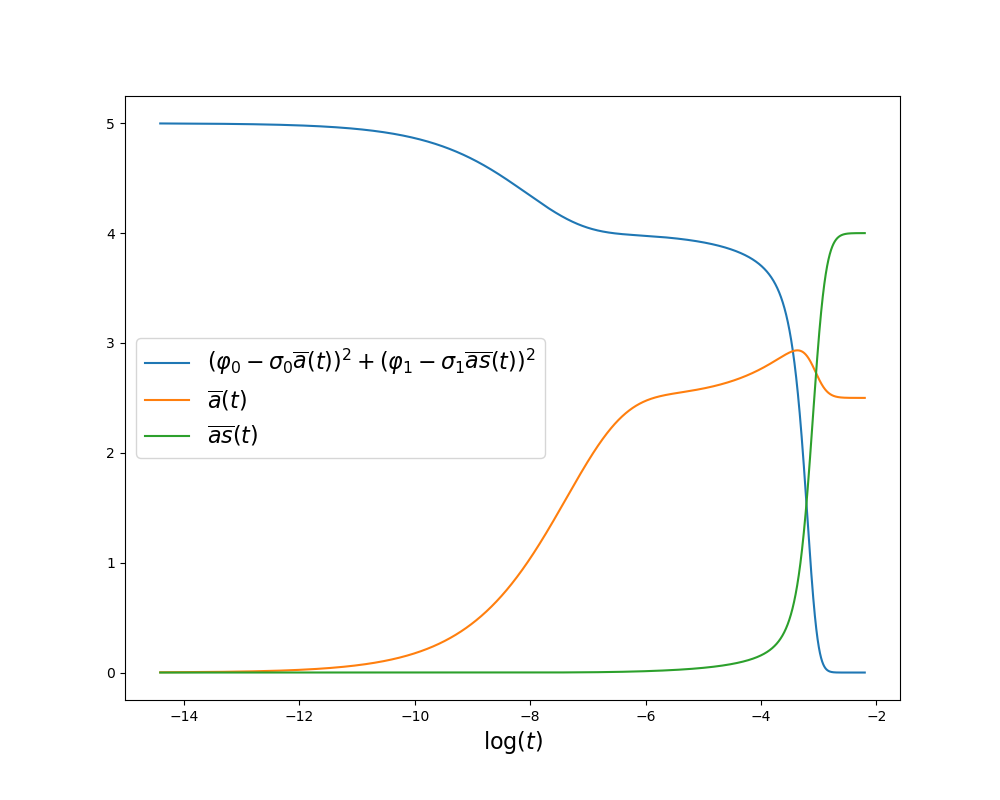}
\caption{Separation of timescales for the constant and linear components: evolution of the averaged quantities $\overline{a}$, $\overline{as}$, and cost function 
for the first two coefficients $\varphi_{0},\varphi_{1}$ as a function of time on a logarithmic scale, for $\varepsilon = 10^{-3}$ (left) and $\varepsilon = 10^{-4}$ (right).}
\label{fig:harmo01_fig1}
\end{figure}
\begin{figure}[h]
\centering
\includegraphics[scale=0.31]{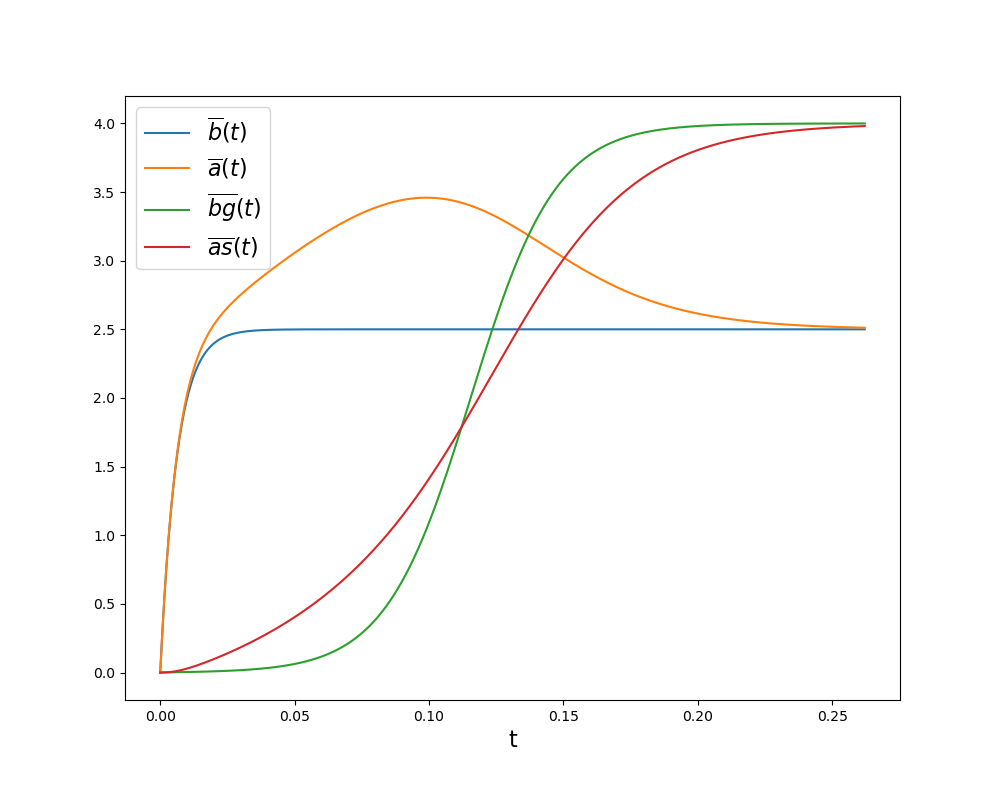}
\includegraphics[scale=0.31]{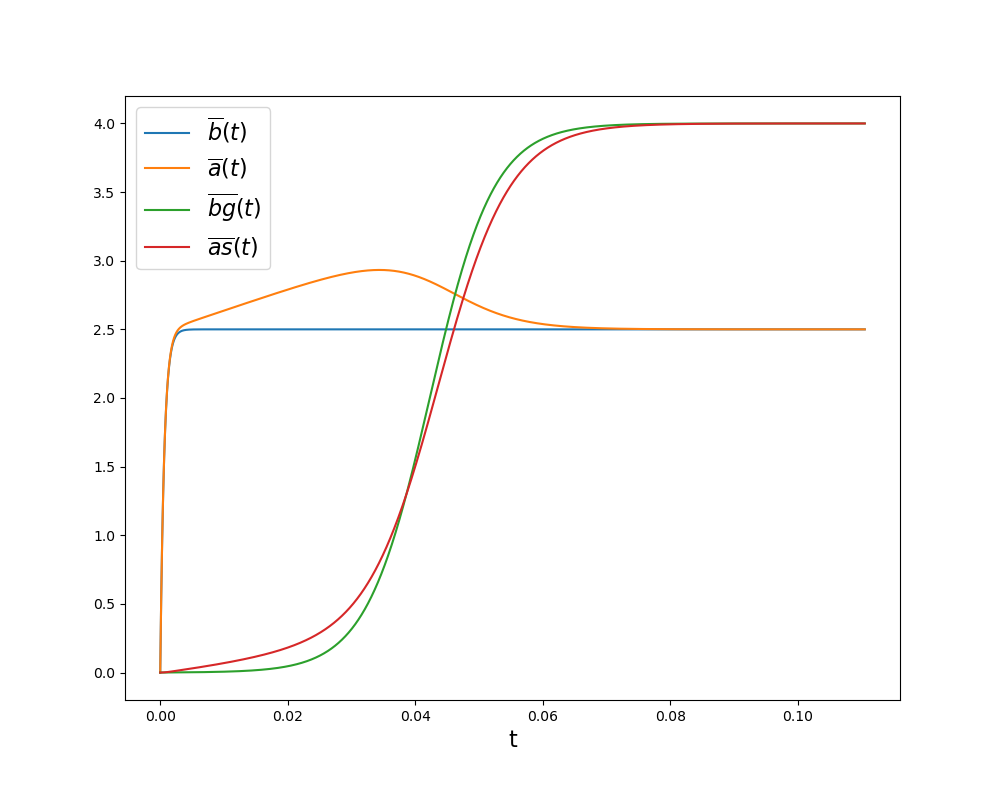}
\caption{Illustration of Theorem \ref{thm:main_tube01}: comparison of $\overline{a}$ (orange curve) with $\overline{b}$ (blue curve), and of $\overline{as}$ (red curve) with $\overline{bg}$ (green curve), 
for $\varepsilon = 10^{-3}$ (left) and $\varepsilon = 10^{-4}$ (right). On a timescale of order $\varepsilon\log(1/\varepsilon)$, the blue and orange curves are nearly indistinguishable, and, when moving to a timescale of order $\sqrt{\varepsilon}\log(1/\varepsilon)$, the error becomes larger, 
as predicted by the bound \eqref{eq:main_tube_mean} (see also Remark \ref{rmk:time_adapt_bound}). The curves collapse at longer times owing to the absence of higher-order terms in the simulation. Overall the relative error between the curves decreases with $\varepsilon$.}
\label{fig:harmo01_fig2}
\end{figure}
\begin{figure}[h]
\centering
\includegraphics[scale=0.31]{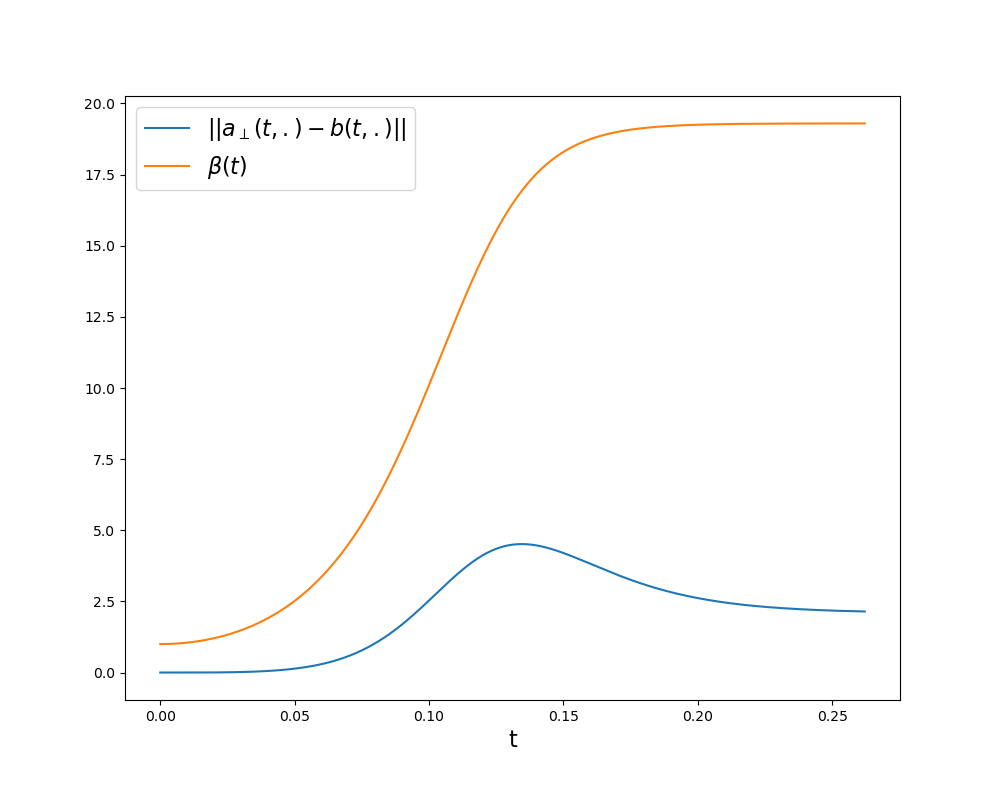}
\includegraphics[scale=0.31]{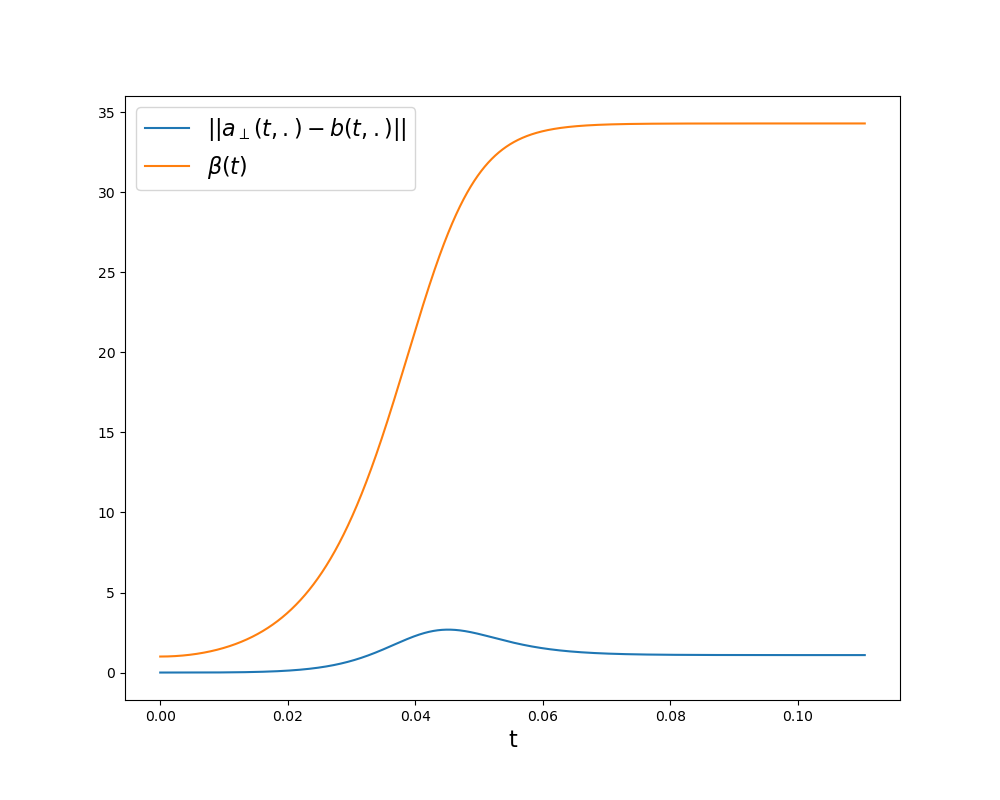}
\includegraphics[scale=0.31]{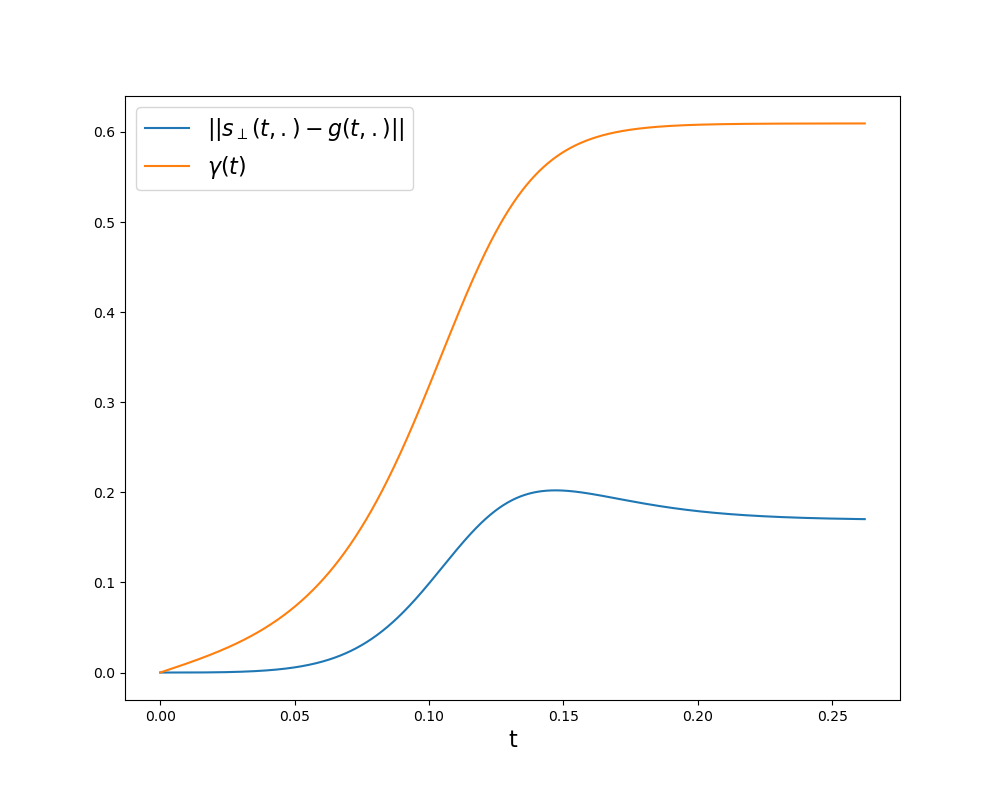}
\includegraphics[scale=0.31]{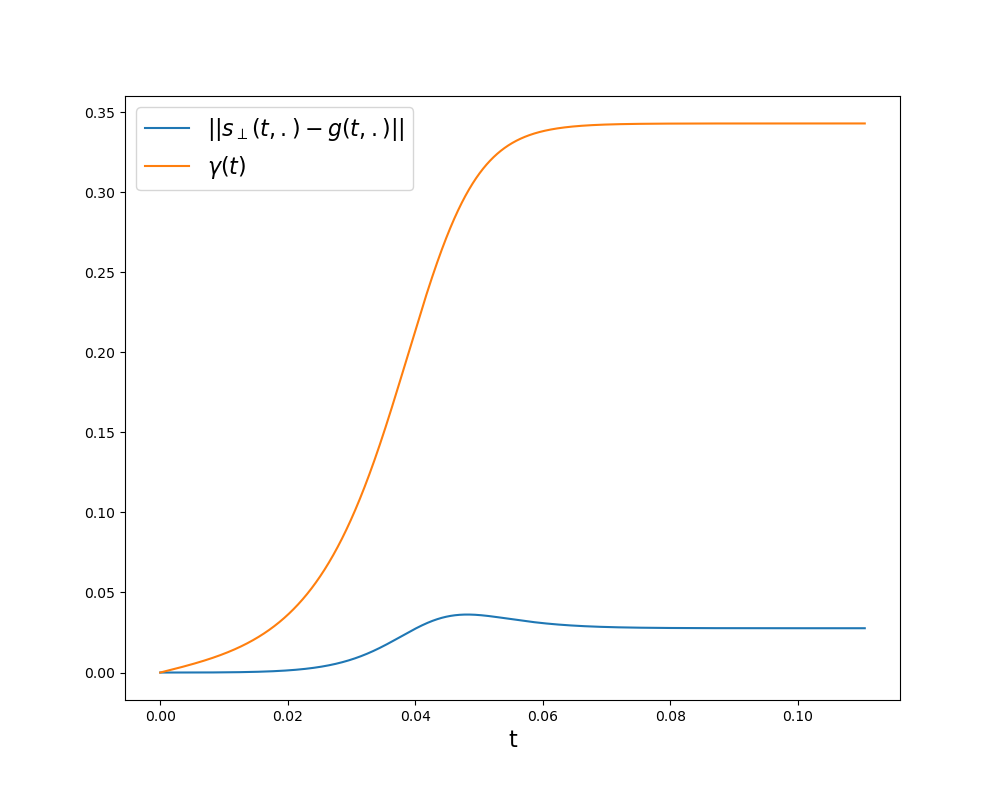}
\caption{Illustration of Theorem \ref{thm:main_tube01}: comparison of $\norm{a_{\perp}(t,.)-b(t,.)}$ with $\beta(t)$ (top two figures), and 
$\norm{s_{\perp}(t,.)-g(t,.)}$ with $\gamma(t)$ (bottom two figures), for $\varepsilon = 10^{-3}$ (left panel) and $\varepsilon = 10^{-4}$ (right panel). The 
approximation errors roughly behave as small multiples of $\beta(t)$, as predicted by the bound \eqref{eq:main_tube_perp} (see also Remark \ref{rmk:bg_comp}), with the error decreasing with $\varepsilon$.}
\label{fig:harmo01_fig3}
\end{figure}
\subsection{Learning the quadratic component}
Figures \ref{fig:harmo012_fig1}, \ref{fig:harmo012_fig2} and \ref{fig:harmo012_fig3} illustrate the results of Section \ref{sec:main_results} for the quadratic component.
\begin{figure}[h]
\centering
\includegraphics[scale=0.31]{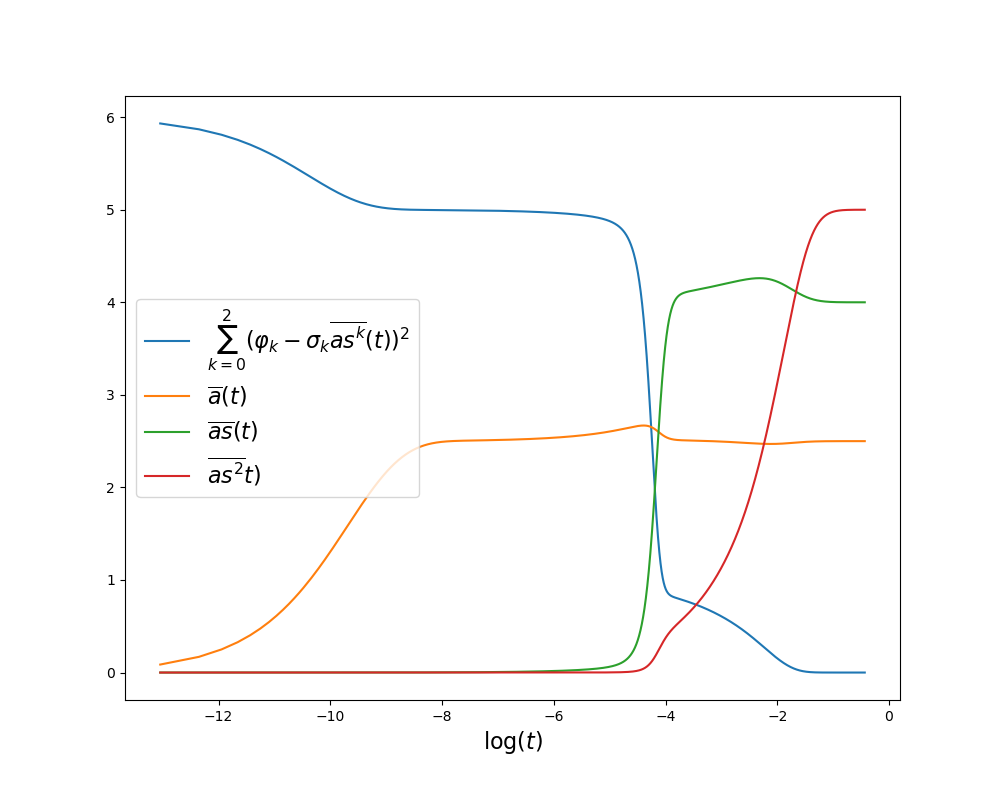}
\includegraphics[scale=0.31]{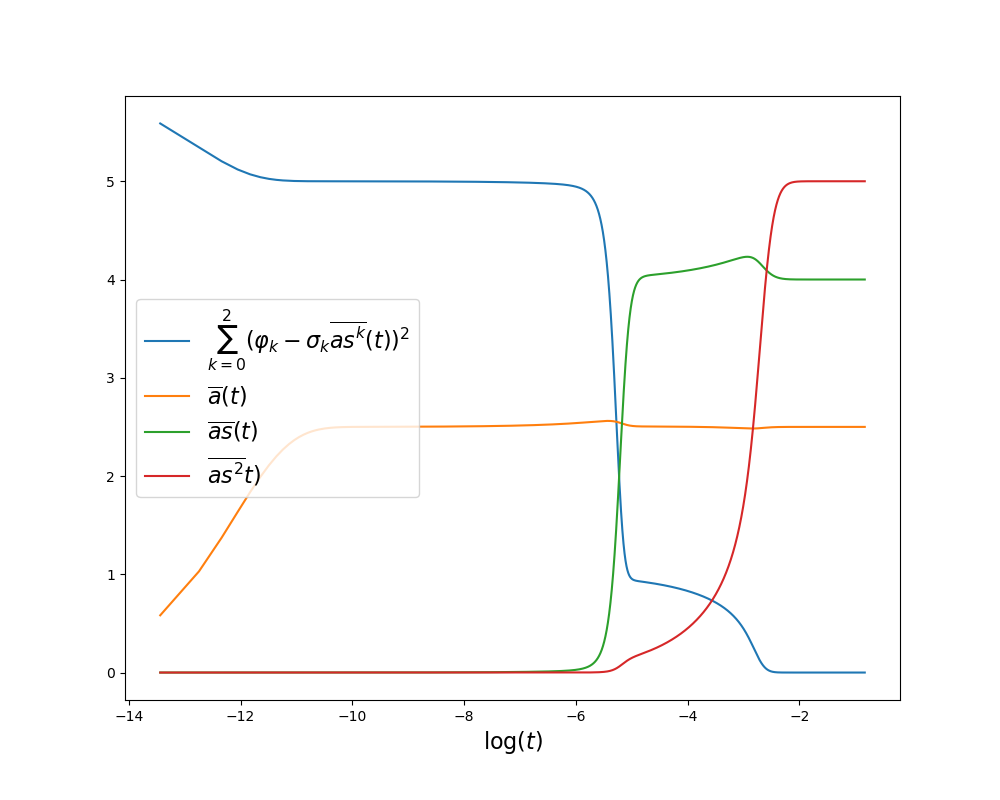}
\caption{Separation of timescales for the constant, linear and quadratic component: evolution of the quantities $\overline{a},\overline{as},\overline{as^{2}}$,
along with the cost function for the first three coefficients $\varphi_{0},\varphi_{1},\varphi_{2}$, as a function of time on a logarithmic scale, for $\varepsilon = 10^{-5}$ (left) and $\varepsilon = 10^{-6}$ (right). Note that, 
to better observe the separation, smaller values of $\varepsilon$ must be used than for the first two components.}
\label{fig:harmo012_fig1}
\end{figure}
In Figure \ref{fig:harmo012_fig2}, we compare the rescaled maps $\hat{a}(t,\omega),\hat{s}(t,\omega)$ defined at Eq.\eqref{eq:main_change_scale_a}-\eqref{eq:main_change_scale_s}, that obey 
the evolution 
\begin{align}
  \partial_{t} \hat{a}(t,\omega) &= \frac{\sigma_{0}}{\sqrt{\varepsilon}}\left(\varphi_{0}-\sigma_{0}\varepsilon^{-1/4}\overline{\hat{a}}(t)\right)+\frac{\sigma_{1}}{\varepsilon^{1/4}}\hat{s}(t,\omega)\left(\varphi_{1}-\sigma_{1}\overline{\hat{a}\hat{s}}(t)\right)+\sigma_{2}\varphi_{2}\hat{s}^{2}(t,\omega), \\
  \partial_{t} \hat{s}(t,\omega) &= \hat{a}(t,\omega)\frac{\sigma_{1}}{\varepsilon^{1/4}}\left(\varphi_{1}-\sigma_{1}\overline{\hat{a}\hat{s}}(t)\right)+2\sigma_{2}\varphi_{2}\hat{a}(t,\omega)\hat{s}(t,\omega),
\end{align}
the idealized maps $z_{1}(t,\omega),z_{2}(t,\omega)$, solutions to the dynamical system Eq.\eqref{eq:main_quad_ideal_dyn1}-\eqref{eq:main_quad_ideal_dyn2}, and the maps 
$\hat{a}_{q}(t,\omega),\hat{s}_{q}(t,\omega)$ corresponding to the solutions to Eq.\eqref{eq:main_mf_dyn1}-\eqref{eq:main_mf_dyn2} if only the coefficients $\sigma_{2},\varphi_{2}$
are kept (i.e. neglecting the effect of the learned coefficients $\varphi_{0},\varphi_{1}$), with the change in time and scale defined at \eqref{eq:main_change_scale_a}-\eqref{eq:main_change_scale_s}.
Namely, $\hat{a}_{q}(t,\omega),\hat{s}_{q}(t,\omega)$ are solutions to the dynamical system
\begin{align}
  \partial_{t} \hat{a}_{q}(t,\omega) &= \sigma_{2}\varphi_{2}\hat{s}_{q}^{2}(t,\omega), \\
  \partial_{t} \hat{s}_{q}(t,\omega) &= 2\sigma_{2}\varphi_{2}\hat{a}_{q}(t,\omega)\hat{s}_{q}(t,\omega).
\end{align}
The initial conditions for these maps are constructed by solving Eq.\eqref{eq:main_mf_dyn1}-\eqref{eq:main_mf_dyn2} until the first two coefficients are learned to a chosen precision.
To summarize these maps, we compare the quantities $\overline{\hat{a}\hat{s}^{2}},\overline{\hat{a}_{q}\hat{s}_{q}^{2}}$ and $\overline{z_{1}z_{2}^{2}}$, which govern 
the learning of the coefficient $\varphi_{2}$. Note that, owing to the change in time and scale, the target value for these quantities is $\varepsilon^{-1/4}\frac{\varphi_{2}}{\sigma_{2}}$.
\begin{figure}[h]
\centering
\includegraphics[scale=0.31]{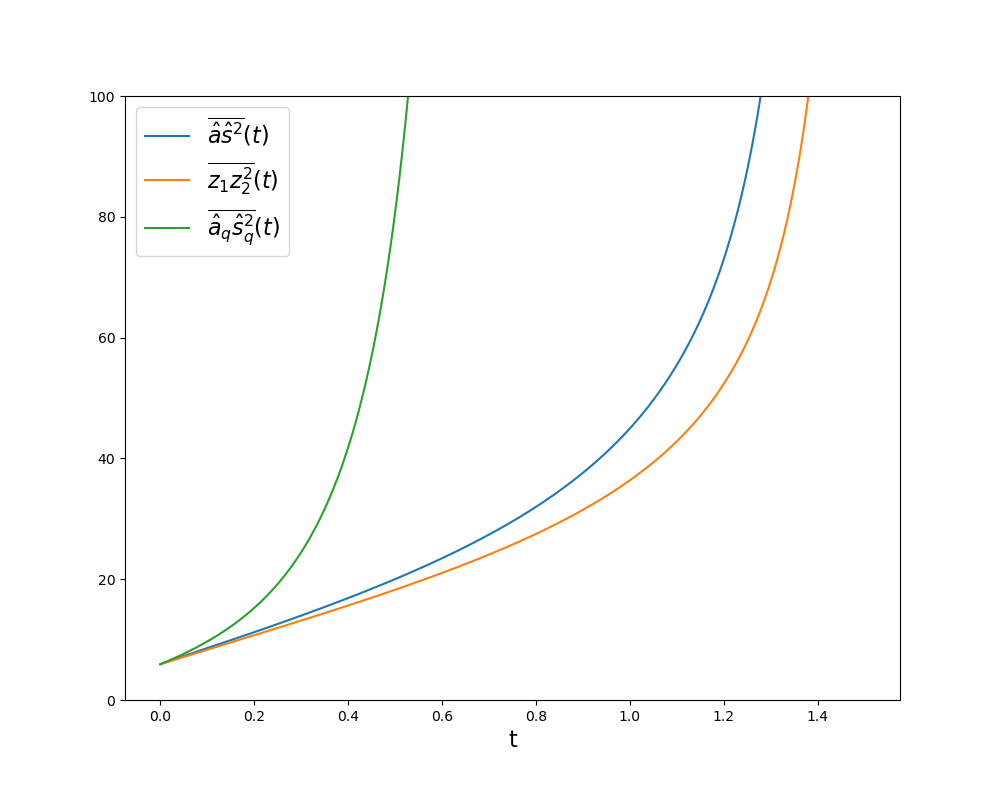}
\includegraphics[scale=0.31]{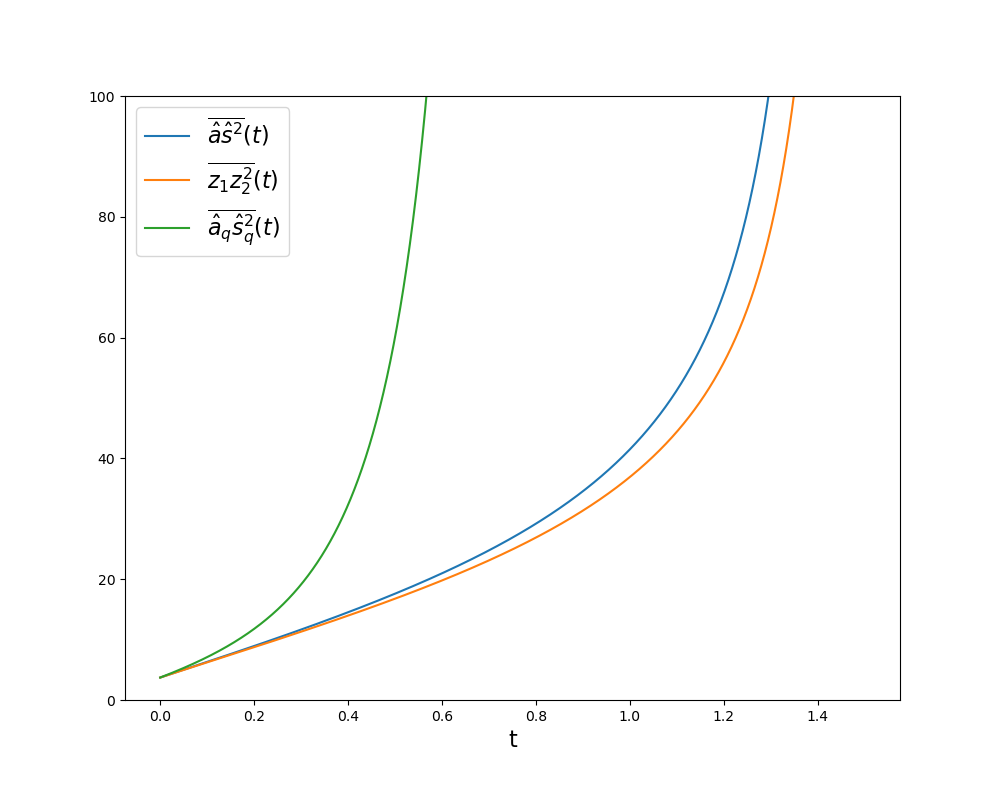}
\caption{Illustration of Proposition \ref{prop:quad_approx}: comparison of the evolution of $\overline{\hat{a}\hat{s}^{2}}$ (blue curve), $\overline{\hat{a}_{q}\hat{s}_{q}^{2}}$ (green curve) and $\overline{z_{1}z_{2}^{2}}$ (orange curve), for $\varepsilon = 10^{-5}$ (left) and $\varepsilon = 10^{-6}$ (right).
The explosion time for the quantity $\overline{\hat{a}_{q}\hat{s}_{q}^{2}}$ is significantly smaller than that of $\overline{\hat{a}\hat{s}^{2}}$, whose evolution is much closer to that of the idealized quantity $\overline{z_{1}z_{2}^{2}}$. The approximation becomes more accurate as $\varepsilon$ decreases.}
\label{fig:harmo012_fig2}
\end{figure}
Finally, in Figure \ref{fig:harmo012_fig3}, we show the evolution of the maps $z_{1}(t,\omega)$ and $z_{2}(t,\omega)$, from the initial condition described above. 
Due to our choice of $(\Omega,\rho)$, this initial condition corresponds to multiples (the terminal values of $\beta,\gamma$) of $a_{init} = -1+2\omega$.
\begin{figure}[h]
\centering
\includegraphics[scale=0.31]{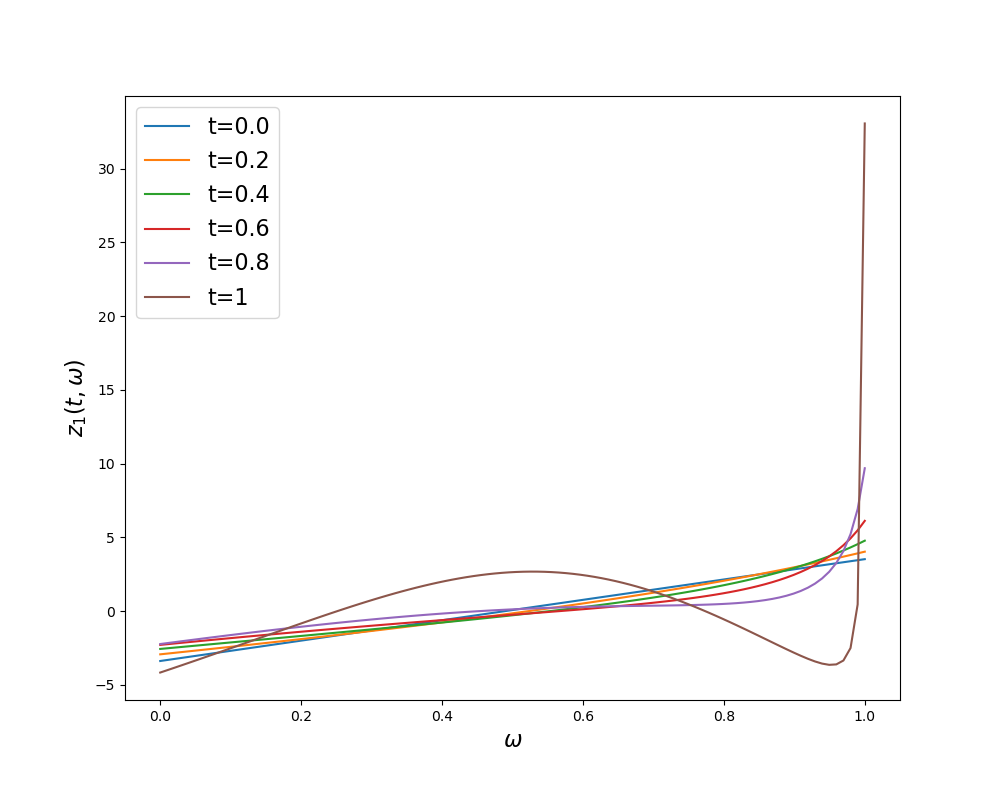}
\includegraphics[scale=0.31]{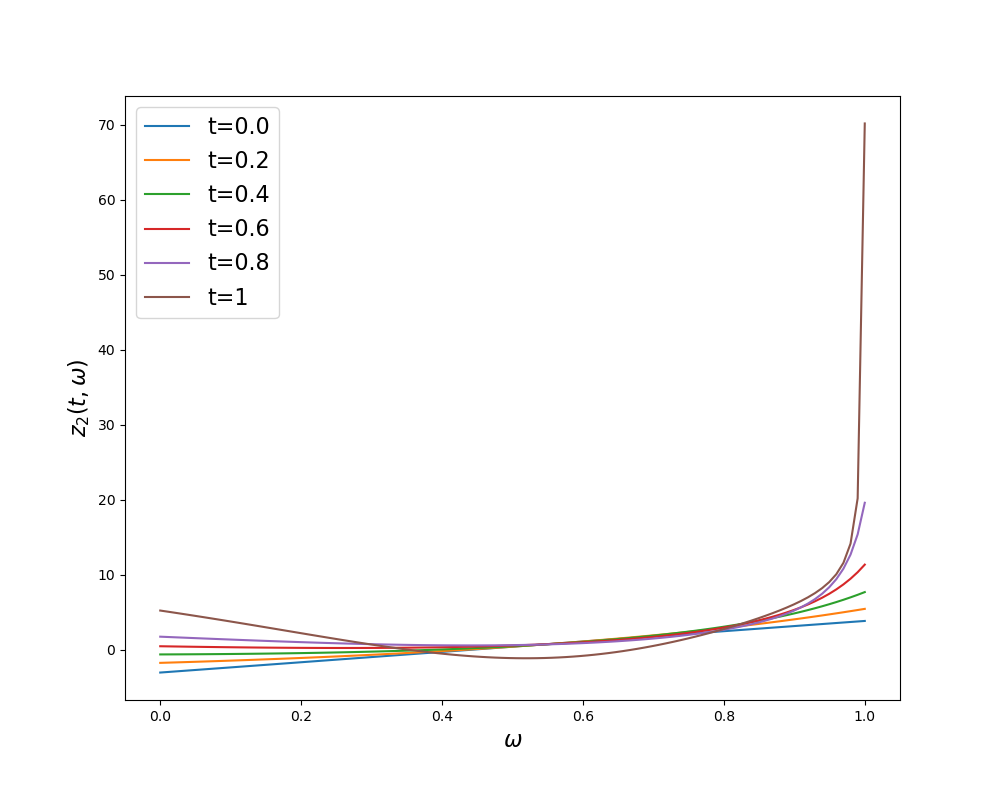}
\caption{Evolution of the maps $z_{1}(t,.),z_{2}(t,.)$, in the setting of Figure \ref{fig:harmo012_fig2}, with $\varepsilon = 10^{-6}$, for fractions of the total running time of Figure \ref{fig:harmo012_fig2}. A small fraction of neurons 
diverge towards learning the coefficient $\varphi_{2}$, while the remaining neurons adapt to the change to conserve the constraints on $\overline{z}_{1}$, $\overline{z_{1}z_{2}}$.}
\label{fig:harmo012_fig3}
\end{figure}
\section{Outline of the proof}
The purpose of this section is to give an informal description of the proof method. It revolves around the identification, at various timescales, of an idealized dynamical system that 
may be (almost) integrated explicitly and captures the phenomenology of the original one.
Owing to the initial condition, we expect to be able 
to successively approximate the evolution \eqref{eq:main_mf_dyn1}-\eqref{eq:main_mf_dyn2} by the terms corresponding to each harmonic at appropriate timescales. Two major difficulties appear:
firstly, obtaining a quantitative control on such approximations and the timescales on which they are valid; and secondly, summarizing the effect of coefficients once they have been learned. Indeed, we cannot 
expect that, once a certain harmonic has been learned, its effect becomes completely negligible as the system progresses to the next one. As we will see below, when studying the coefficients corresponding to the constant and linear part of $\varphi$, an explicit decomposition of the solution allows one to 
circumvent this second issue. Therefore we postpone dealing with this difficulty to Subsection \ref{subsec:sketch2}.
\subsection{Constructing tubes around idealized systems}
\label{subsec:sketch01}
Let us now focus on the first difficulty described above. A naive idea is to introduce sequences of thresholds for the variables $a,s$ (or related quantities such as their mean or supremum over $\Omega$), and use them to approximate \eqref{eq:main_mf_dyn1}-\eqref{eq:main_mf_dyn2} around the leading order terms 
up to the first hitting times of these thresholds. While this strategy can lead to preliminary results, it becomes quickly laborious as more complex dynamics are considered. A more efficient solution is to introduce tubes of varying radii around the curves defined 
by the idealized system, and show that the original system remains in those tubes on relevant timescales. We will now outline how to determine the radii of those tubes, namely using the knowledge obtained from integrating the idealized system, before formalizing this discussion 
in Section \ref{sec:proofs}. \newline \indent 
In what follows, $D$ will denote a constant that depends only on $\{\varphi_{k},\sigma_{k}\}_{k\geq 0}$ and $M$, and whose exact value may change from line to line. We begin by noting that, near initialization,
\begin{equation}
  \label{eq:abar_bound_1}
  \vert \partial_{t} \overline{a}(t) - \frac{\sigma_{0}}{\varepsilon}\left(\varphi_{0}-\sigma_{0}\overline{a}(t)\right) \vert \leq \frac{D}{\varepsilon}\norm{s(t,.)},
\end{equation}
so that we expect $\overline{a}(t)$ to remain close to the solution to the ODE 
\begin{equation}
  \begin{cases}
  &\frac{d}{dt}\overline{b}(t) = \frac{\sigma_{0}}{\varepsilon}\left(\varphi_{0}-\sigma_{0}\overline{b}(t)\right), \\
  &\overline{b}(0) = \overline{a}(0),
  \end{cases}
\end{equation}
namely 
\begin{equation}
  \overline{b}(t) = \frac{\varphi_{0}}{\sigma_{0}}+\left(\overline{a}(0)-\frac{\varphi_{0}}{\sigma_{0}}\right)\exp(-\frac{\sigma_{0}^{2}t}{\varepsilon}).
\end{equation}
This function immediately gives the required timescale to learn the constant component, provided we establish a sufficiently sharp control of $\overline{a}$ around $\overline{b}$.
Upon rewriting the bound \eqref{eq:abar_bound_1} as 
\begin{equation}
  \big\vert \partial_{t}(\overline{a}(t)-\overline{b}(t))+\frac{\sigma_{0}^{2}}{\varepsilon}\left(\overline{a}(t)-\overline{b}(t)\right) \big\vert \leq \frac{D}{\varepsilon}\norm{s(t,.)},
\end{equation}
we obtain the first tube for $\overline{a}$ around $\overline{b}$, the radius of which will depend on our control over $\norm{s(t,.)}$. 
Noting that we also have 
\begin{align}
  &\partial_{t} \norm{s(t,.)} \leq D \norm{a(t,.)}, \\
  &\partial_{t} \norm{a_{\perp}(t,.)} \leq \frac{D}{\varepsilon}\norm{s(t,.)},
\end{align}
near initialization, we may close this system of inequalities to obtain an upper bounding flow on $\norm{s(t,.)}$ and conclude with respect to the 
learning process of the coefficient $\varphi_{0}$.
Instead of doing so, we seek to refine our approximation to also include the learning process of $\varphi_{1}$. Upon writing the evolution of $a_{\perp},s_{\perp}$, we find:
\begin{align}
  \big \vert \partial_{t}a_{\perp}-\frac{\sigma_{1}}{\varepsilon}s_{\perp}\left(\varphi_{1}-\sigma_{1}\overline{as}\right) \big \vert &\leq \frac{D}{\varepsilon}\norm{s_{\perp}(t,.)}^{2}, \\
  \big \vert \partial_{t}s_{\perp}-\sigma_{1}a_{\perp}(\varphi_{1}-\sigma_{1}\overline{as})\big \vert &\leq D \norm{a_{\perp}(t,.)}\norm{s_{\perp}(t,.)}.
\end{align}
Assuming that $\overline{a}$ remains close to $\overline{b}$, $\overline{s}$ will remain of order $Dt$ (with $t$ a fractional power of $\varepsilon$) until $\overline{as}$ becomes of order one, 
which is necessary to learn the coefficient $\varphi_{1}$. Thus, the term $\overline{as}$ should be well approximated by $\overline{a_{\perp}s_{\perp}}$, and 
the functions $a_{\perp},s_{\perp}$ should remain close to the solutions to the system 
\begin{align}
  \partial_{t}b = \frac{\sigma_{1}}{\varepsilon}g(\varphi_{1}-\sigma_{1}\overline{bg}), \\
  \partial_{t}g = \sigma_{1}b(\varphi_{1}-\sigma_{1}\overline{bg}),
\end{align}
with initial condition $b(0,.) = a_{init,\perp}(.)$ and $g(0,.)=0$, which can be solved analytically. As before, we can use this solution to define tubes for $a_{\perp},s_{\perp}$ respectively around the curves 
$b,g$, and combine them with the tube for $\overline{a}$ around $\overline{b}$ along with the upper bound for $\overline{s}$ (i.e. a tube around $0$). The radii of the tubes 
are then determined by plugging the idealized solutions $\overline{b},b,g$ in the higher order terms of the evolution equations \eqref{eq:main_mf_dyn1}-\eqref{eq:main_mf_dyn2}, making it possible to complete
the proof upon introducing an appropriate stopping time to control the validity of the approximation. These steps are formalized in Subsection \ref{subsec:proofs_linear}.
\subsection{Approximating the learning dynamics near target sets}
\label{subsec:sketch2}
As we move to the quadratic component, there is no clear decomposition of the solution to \eqref{eq:main_mf_dyn1}-\eqref{eq:main_mf_dyn2} that leads 
to a simpler auxiliary dynamical system to study. Indeed, as shown by the bounds \eqref{eq:main_init_quad1}-\eqref{eq:main_init_quad2}, 
once the first two coefficients have been learned, the terms $\varphi_{0}-\sigma_{0}\overline{a},s(\varphi_{1}-\sigma_{1}\overline{as})$ in \eqref{eq:main_mf_dyn1} and $a(\varphi_{1}-\sigma_{1}\overline{as})$ in \eqref{eq:main_mf_dyn2} are 
respectively of the same order of magnitude as the terms $a^{2}$ and $as$ that drive the dynamics towards learning the coefficient $\varphi_{2}$. There is, however, no straightforward way to show that one dominates the other. We will show instead 
that the learning dynamics are well approximated by an auxiliary system in which the integral constraints defined by the coefficients $\varphi_{0},\varphi_{1}$ are exactly verified, and are enforced through time dependent coefficients akin to Lagrange multipliers. To provide some understanding of this procedure, 
let us look at a simplified example. With $f,g$ smooth vector fields on $\mathbb{R}^{d}$ and $h : \mathbb{R}^{d} \to \mathbb{R}_{+}$ a smooth function, consider the following dynamical system 
\begin{equation}
  \label{eq:toy_flow}
  \frac{d}{dt}z_{\varepsilon}(t) = \frac{h(z_{\varepsilon}(t))}{\varepsilon}g(z_{\varepsilon}(t))+f(z_{\varepsilon}(t)), 
\end{equation}
with any initial condition $z_{\varepsilon}(0)$ such that there exists a non-trivial time interval $[0,T_{0}]$ on which a solution exists. Assume that there exists a time $T_{1}$ and a continuous, positive,
increasing function $u(\varepsilon)$ such that $u(0) = 0$ and such that, for any $0 \leq t \leq T_{1}$, $h(z_{\varepsilon}(t)) \leq u(\varepsilon)$. We claim that, for any $t$ and $\delta$ such that 
$t+\delta \leq T_{0} \wedge T_{1}$, the increment $z_{\varepsilon}(t+\delta)-z_{\varepsilon}(t)$ is close to the tangent space at $z_{\varepsilon}(t)$ of the embedded submanifold 
\begin{equation}
  \mathcal{S} = \left\{z \in \mathbb{R}^{d} : h(z) = 0\right\}.
\end{equation}
Indeed, a second order Taylor expansion of $h(z_{\varepsilon}(t+\delta))$ around $t$ gives 
\begin{equation}
  \vert \langle \nabla h(z_{\varepsilon}(t)),z_{\varepsilon}(t+\delta)-z_{\varepsilon}(t) \rangle \vert \leq 2u(\varepsilon)+\mathcal{O}(\delta^{2}).
\end{equation}
Using this inequality and appropriately expanding $z_{\varepsilon}(t+\delta)$ around $z_{\varepsilon}(t)$ and, in turn, the flow Eq.\eqref{eq:toy_flow}, one may show that the considered solution 
to \eqref{eq:toy_flow} is close to that of the Cauchy problem 
\begin{equation}
  \begin{cases}
    \frac{d}{dt}z(t) = \alpha(t)g(z(t))+f(z(t)), \\
    z(0) = z_{0},
  \end{cases}
\end{equation}
where $z_{0}$ is such that $h(z_{0}) = 0$ (and is sufficiently close to $z_{\varepsilon}(0)$), and where the coefficient $\alpha(t)$ is given by 
\begin{equation}
  \alpha(t) = -\frac{\langle \nabla h(z(t)),f(z(t)) \rangle}{\langle \nabla h(z(t)),g(z(t)) \rangle},
\end{equation}
and imposes that $\frac{d}{dt}h(z(t)) = 0$. This operation is reminiscent of the orthogonal projection on a tangent space used to compute gradients on smooth manifolds, see e.g. \cite{boumal2023introduction}, but the projection is skewed in the direction of the vector field $g$. This intuition is formalized by Proposition \ref{prop:approx_dyn_set} in a more general setup (not specific to neural networks), which we specialize to the evolution 
\eqref{eq:main_mf_dyn1}-\eqref{eq:main_mf_dyn2} in Corollary \ref{cor:nn_approx_dyn}. This ultimately leads to Proposition \ref{prop:quad_approx} upon 
combining the aforementioned approximation result with the control over the quantities $\vert \varphi_{0}-\sigma_{0}\overline{a} \vert$ and $\vert \varphi_{1}-\sigma_{1}\overline{as} \vert$ established 
in Theorem \ref{thm:main_tube01}. We leave the construction of tubes, used to prove Theorem \ref{thm:main_tube01}, around the idealized system identified in 
Proposition \ref{prop:quad_approx}, for future work.
\begin{rmk}
  As a benefit of working with increments, the proof of Proposition \ref{prop:approx_dyn_set} does not require to estimate derivatives of the quantity $h(z_{\varepsilon}(t))$, for which we only have the control $h(z_{\varepsilon}(t)) \leq u(\varepsilon)$; these derivatives may display singular behaviour. 
\end{rmk}
\begin{rmk}
  We note that a Lagrange multiplier appears in the matched asymptotic expansion argument of \cite{berthier2024learning} when describing 
  the learning dynamics of the coefficient $\varphi_{1}$, to enforce the constraint $\overline{a} = \frac{\varphi_{0}}{\sigma_{0}}$. While the proof of Theorem \ref{thm:main_tube01} can also be carried out using the approximation result 
  from Proposition \ref{prop:approx_dyn_set} and the iteration argument outlined in Remark \ref{rmk:quad_approx2}, the decomposition $a = \overline{a}+a_{\perp}$ and 
  $s = \overline{s}+s_{\perp}$ naturally leads to the same auxiliary dynamics, allowing for a more direct proof. This, however, is only possible for the 
  coefficients $\varphi_{0},\varphi_{1}$, while Proposition \ref{prop:approx_dyn_set} allows one to identify the correct auxiliary system at any order. 
\end{rmk}
\section{Proofs}
\label{sec:proofs}
We begin with the constant and linear components, based on the construction of tubes outlined in Section \ref{subsec:sketch01}.
\subsection{Learning the constant and linear components}
\label{subsec:proofs_linear}
In this section, the definitions of $b,g,\gamma,\beta$ are those from Appendix \ref{sec:app_A}, and $\overline{b}$ is given by \eqref{eq:b_bar}.
\begin{thm}
  \label{thm:harmo_0+1_tube}
  There exists a constant $D<+\infty$ such that, for any $c \in [0,\frac{1}{8D}]$ and time $(t_{\varepsilon})_{\varepsilon>0}$ of the form 
  \begin{equation}
    t_{\varepsilon} = \left(\frac{1}{4\sigma_{1}\varphi_{1}}+c\right)\sqrt{\varepsilon}\log(\frac{1}{\varepsilon}),
  \end{equation}
  there exists a constant $C<+\infty$ such that the following holds for every $\varepsilon$ sufficiently small. For every $t\leq t_{\varepsilon}$, we have 
  \begin{align}
    \label{eq:bound_as_bar_proof}
    \vert \overline{a}-\overline{b} \vert(t) \leq Ct+C\varepsilon\beta^{2}(t), \quad \vert \overline{s}(t) \vert \leq Ct, \\
    \norm{(a_{\perp}-b)(t,.)}+\varepsilon^{-1/2}\norm{(s_{\perp}-g)(t,.)} \leq C\varepsilon^{\frac{1}{4}-cD}\beta(t).
    \label{eq:bound_as_perp_proof}
  \end{align}
\end{thm}
\begin{proof}
Let us define the hitting time 
\begin{align}
  T := \inf \left\{t \geq 0 : \norm{a}+\varepsilon^{-1/2}\norm{s} \geq C\beta \quad \mbox{or} \quad \norm{a_{\perp}-b}+\varepsilon^{-1/2}\norm{s_{\perp}-g} \geq \varepsilon^{\eta}\beta\right\},
\end{align}
where $C$ and $\eta$ are constants to be determined in the course of the proof. Throughout the argument, we write $D = D(\mathrm{data}) < +\infty$ for a large 
constant whose value may change from place to place; we stress in particular that the constant $D$ does not depend on the choice of the constant $C$. Although 
we will not always state it explicitly, all the derivations in this proof are performed for times $t \leq T \wedge t_{\varepsilon}$ only. \newline
\paragraph{\textbf{Controlling $\overline{a},\overline{s}$}}
Recall that 
\begin{equation}
  \partial_{t} \overline{a} = \frac{1}{\varepsilon}\sum_{k=0}^{\infty}\sigma_{k}\overline{s^{k}}\left(\varphi_{k}-\sigma_{k}\overline{as^{k}}\right).
\end{equation}
Then, note that, on the considered time interval 
\begin{equation}
  \norm{a}+\varepsilon^{-1/2}\norm{s} \leq C\beta \leq CD\varepsilon^{-1/4},
\end{equation}
where we used the fact that $\beta(t)$ is increasing and that $\beta(t_{\varepsilon}) \leq D\varepsilon^{-1/4}$.
Thus,
\begin{equation}
  \vert \partial_{t}\overline{a}-\left(\frac{\sigma_{0}}{\varepsilon}\left(\varphi_{0}-\sigma_{0}\overline{a}\right)+\frac{\sigma_{1}\overline{s}}{\varepsilon}\left(\varphi_{1}-\sigma_{1}\overline{as}\right)\right) \vert \leq \big \vert \frac{1}{\varepsilon}\sum_{k=2}^{\infty}\sigma_{k}\overline{s^{k}}\left(\varphi_{k}-\sigma_{k}\overline{as^{k}}\right) \big \vert.
\end{equation}
But then, for any $k \geq 2$, $\overline{as^{k}}\leq \overline{as^{2}} \leq C^{3}D^{3}\varepsilon^{1/4}$, so that, for $\varepsilon$ sufficiently small,
\begin{equation}
\frac{1}{\varepsilon}\sum_{k=2}^{\infty}\sigma_{k}\overline{s^{k}}\left(\varphi_{k}-\sigma_{k}\overline{as^{k}}\right) \leq C^{2}D\beta^{2}.
\end{equation}
Owing to the definition of $\overline{b}$, we may write 
\begin{equation}
  \vert \partial_{t}(\overline{a}-\overline{b})+\frac{\sigma_{0}}{\varepsilon}(\overline{a}-\overline{b}) \vert \leq \frac{\sigma_{1}\vert \overline{s}(t) \vert }{\varepsilon}\vert \varphi_{1}-\sigma_{1}\overline{as} \vert +C^{2}D\beta,
\end{equation}
which, upon noting that $\overline{as} \leq C^{2}\sqrt{\varepsilon}\beta^{2} \leq C^{2}D$, leads to 
\begin{equation}
  \vert \partial_{t}(\overline{a}-\overline{b})+\frac{\sigma_{0}}{\varepsilon}(\overline{a}-\overline{b}) \vert \leq \frac{C^{2}D \vert \overline{s} \vert}{\varepsilon}+C^{2}D\beta^{2}.
\end{equation}
Integrating the above and using the fact that $\beta$ is increasing, we find 
\begin{align}
  \vert \overline{a}(t)-\overline{b}(t) \vert &\leq e^{-\frac{\sigma_{0}}{\varepsilon}t}\int_{0}^{t}e^{\frac{\sigma_{0}}{\varepsilon}u}\left(\frac{C^{2}D \vert \overline{s}(u) \vert}{\varepsilon}+C^{2}D\beta^{2}(u)\right)du, \\
  & \leq e^{-\frac{\sigma_{0}}{\varepsilon}t}\left(\frac{C^{2}D \vert \overline{S}(t) \vert}{\varepsilon}+C^{2}D\beta^{2}(t)\right)\int_{0}^{t}e^{\frac{\sigma_{0}}{\varepsilon}u}du,
\end{align}
where we introduced $\overline{S}(t) = \sup_{r \leq t} \vert \overline{s}(r) \vert$. Finally, we obtain 
\begin{equation}
  \label{eq:bound_a_bar}
  \vert \overline{a}(t)-\overline{b}(t) \vert \leq C^{2}D\left(\overline{S}(t)+\varepsilon\beta^{2}(t)\right).
\end{equation}
Turning to $\overline{s}$, we write 
\begin{align}
  \partial_{t}\overline{s} = \overline{a}\sigma_{1}(\varphi_{1}-\sigma_{1}\overline{as})+\sum_{k=2}^{\infty}k\sigma_{k}\overline{as^{k-1}}\left(\varphi_{k}-\sigma_{k}\overline{as^{k}}\right)-\sum_{k=1}^{\infty}k\sigma_{k}\overline{as^{k+1}}\left(\varphi_{k}-\sigma_{k}\overline{as^{k}}\right).
\end{align}
Recalling that $\overline{as^{2}} \leq C^{3}D^{3}\varepsilon^{1/4}$, and that $\sum_{k=1}^{\infty}k\sigma_{k} < \infty$, we find that, for $\varepsilon$ sufficiently small 
\begin{equation}
  \vert \partial_{t}\overline{s}-\overline{a}\sigma_{1}\left(\varphi_{1}-\sigma_{1}\overline{as}\right) \vert \leq C^{2}D\sqrt{\varepsilon}\beta^{2}+C^{3}D\varepsilon^{1/4},
\end{equation}
so that 
\begin{equation}
  \vert \partial_{t} \overline{s} \vert \leq C^{2}D(|\overline{a}|+1).
\end{equation}
Using the decomposition $\vert \overline{a} \vert \leq \vert \overline{b} \vert +\vert \overline{a}-\overline{b} \vert$ along with the bound Eq.\eqref{eq:bound_a_bar}, we find
\begin{equation}
  \vert \overline{a} \vert \leq C^{2}D\left(\overline{S}+\varepsilon\beta^{2}+1\right),
\end{equation}
which, recalling that $\beta^{2} \leq \frac{D}{\sqrt{\varepsilon}}$, leads to
\begin{equation}
  \partial_{t}\overline{s} \leq C^{4}D\left(\overline{S}+1\right).
\end{equation}
This, along with Lemma \ref{lem:Lip_sup}, implies that
\begin{equation}
  \partial_{t}\overline{S} \leq C^{4}D\left(\overline{S}+1\right),
\end{equation}
and 
\begin{equation}
  \label{eq:bound_s_bar}
  \overline{S}(t) \leq C^{4}Dt.
\end{equation}
Finally, using the above bound in Eq.\eqref{eq:bound_a_bar}, we find 
\begin{equation}
  \vert \overline{a}(t)-\overline{b}(t) \vert \leq C^{4}D\left(t+\varepsilon\beta^{2}(t)\right).
\end{equation}
\paragraph{\textbf{Controlling $a_{\perp},s_{\perp}$}} We now write 
\begin{equation}
  \label{eq:dt_a_p}
  \partial_{t}a_{\perp} = \frac{1}{\varepsilon}\sigma_{1}s_{\perp}\left(\varphi_{1}-\sigma_{1}\overline{as}\right)+\frac{1}{\varepsilon}\sum_{k=2}^{\infty}\sigma_{k}\left(s^{k}-\overline{s^{k}}\right)\left(\varphi_{k}-\sigma_{k}\overline{as^{k}}\right),
\end{equation}
and 
\begin{align}
  \label{eq:dt_s_p}
  \partial_{t}s_{\perp} = \sigma_{1}a_{\perp}(\varphi_{1}-\sigma_{1}\overline{as})&+\sum_{k=2}^{\infty}k\sigma_{k}(as^{k-1}-\overline{as^{k-1}})(\varphi_{k}-\sigma_{k}\overline{as^{k}}) \notag \\
  &-\sum_{k=1}^{\infty}k\sigma_{k}(as^{k+1}-\overline{as^{k+1}})(\varphi_{k}-\sigma_{k}\overline{as^{k}}).
\end{align}
Now note that $\overline{as} = \overline{a_{\perp}s_{\perp}}+\bar{a}\bar{s}$, so that 
\begin{align}
  \partial_{t}a_{\perp} &= \frac{1}{\varepsilon}\sigma_{1}s_{\perp}(\varphi_{1}-\sigma_{1}\overline{a_{\perp}s_{\perp}})+e_{1}(t), \\
  \partial_{t}s_{\perp} &= \sigma_{1}a_{\perp}(\varphi_{1}-\sigma_{1}\overline{a_{\perp}s_{\perp}})+e_{2}(t),
\end{align}
where 
\begin{align}
  e_{1}(t) &= -\frac{1}{\varepsilon}\sigma^{2}_{1}s_{\perp}\bar{a}\bar{s}+\frac{1}{\varepsilon}\sum_{k=2}^{\infty}\sigma_{k}\left(s^{k}-\overline{s^{k}}\right)\left(\varphi_{k}-\sigma_{k}\overline{as^{k}}\right), \\
  e_{2}(t) &= -\sigma^{2}_{1}a_{\perp}\bar{a}\bar{s}+\sum_{k=2}^{\infty}k\sigma_{k}(as^{k-1}-\overline{as^{k-1}})(\varphi_{k}-\sigma_{k}\overline{as^{k}})
  -\sum_{k=1}^{\infty}k\sigma_{k}(as^{k+1}-\overline{as^{k+1}})(\varphi_{k}-\sigma_{k}\overline{as^{k}}).
\end{align}
To establish upper bounds on $e_{1},e_{2}$, we use \eqref{eq:bound_a_bar} and \eqref{eq:bound_s_bar} along with the definition of $T$ to find 
\begin{align}
  s_{\perp}\bar{a}\bar{s} \leq C\sqrt{\varepsilon}\beta C^{4}D\left(t+\varepsilon\beta^{2}(t)\right)C^{4}Dt, \\
  a_{\perp}\bar{a}\bar{s} \leq C\beta C^{4}D\left(t+\varepsilon\beta^{2}(t)\right)C^{4}Dt.
\end{align}
Similarly, for any $t \leq T$ and $\varepsilon$ sufficiently small,
\begin{align}
  \label{eq:up_bound_e1_e2_1}
  \vert \frac{1}{\varepsilon}\sum_{k=2}^{\infty}\sigma_{k}\left(s^{k}-\overline{s^{k}}\right)\left(\varphi_{k}-\sigma_{k}\overline{as^{k}}\right) \vert &\leq C^{2}D\beta^{2}, \\
  \vert \sum_{k=2}^{\infty}k\sigma_{k}(as^{k-1}-\overline{as^{k-1}})(\varphi_{k}-\sigma_{k}\overline{as^{k}}) \vert &\leq C^{2}D\sqrt{\varepsilon}\beta^{2}, \\
  \vert \sum_{k=1}^{\infty}k\sigma_{k}(as^{k+1}-\overline{as^{k+1}})(\varphi_{k}-\sigma_{k}\overline{as^{k}}) \vert &\leq C^{3}D\varepsilon\beta^{3}.
  \label{eq:up_bound_e1_e2_2}
\end{align}
Now, since $t \leq D\sqrt{\varepsilon}\log(\frac{1}{\varepsilon})$ and $\beta(t) \leq D\varepsilon^{-1/4}$, we reach that, for $\varepsilon$ sufficiently small, 
\begin{equation}
  \vert e_{1}(t) \vert+\varepsilon^{-1/2}\vert e_{2}(t) \vert \leq C^{2}D\beta^{2}(t).
\end{equation}
To linearize $a_{\perp},s_{\perp}$ around $b,g$, we introduce $u:= a_{\perp}-b$ and $v = s_{\perp}-g$, so that 
\begin{equation}
  \overline{a_{\perp}s_{\perp}} = \overline{bg}+\overline{ug}+\overline{bv}+\overline{uv},
\end{equation}
and 
\begin{align}
  \partial_{t}a_{\perp} &= \frac{1}{\varepsilon}\sigma_{1}s_{\perp}(\varphi_{1}-\sigma_{1}\overline{bg})-\frac{1}{\varepsilon}\sigma_{1}^{2}s_{\perp}(\overline{ug}+\overline{bv}+\overline{uv})+e_{1}(t), \\
  \partial_{t}s_{\perp} &= \sigma_{1}a_{\perp}(\varphi_{1}-\sigma_{1}\overline{bg})-\sigma_{1}^{2}a_{\perp}(\overline{ug}+\overline{bv}+\overline{uv})+e_{2}(t).
\end{align}
Owing to Proposition \ref{prop:app_bg}, we know that $\norm{g} \leq D\sqrt{\varepsilon}\beta$ and $\norm{b} \leq D\beta$, so that 
\begin{equation}
  \vert \overline{ug}+\overline{bv} \vert \leq D\beta \left(\norm{u}\sqrt{\varepsilon}+\norm{v}\right),
\end{equation}
and, letting $Z = \norm{u}+\varepsilon^{-1/2}\norm{v}$, 
\begin{equation}
  \label{eq:bound_ug_bv_bar}
  \vert \overline{ug}+\overline{bv} \vert \leq D\beta \sqrt{\varepsilon}Z.
\end{equation}
Now,
\begin{align}
  \partial_{t} (u+\varepsilon^{-1/2}v) &= \frac{1}{\varepsilon}\sigma_{1}v(\varphi_{1}-\sigma_{1}\overline{bg})-\frac{1}{\varepsilon}\sigma_{1}^{2}s_{\perp}(\overline{ug}+\overline{bv}+\overline{uv})+e_{1}(t) \\
  &+\varepsilon^{-1/2}\left(\sigma_{1}u(\varphi_{1}-\sigma_{1}\overline{bg})-\sigma_{1}^{2}a_{\perp}(\overline{ug}+\overline{bv}+\overline{uv})+e_{2}(t)\right).
\end{align}
or equivalently
\begin{align}
  \partial_{t}\left(u+\varepsilon^{-1/2}v\right) \leq \frac{\sigma_{1}}{\sqrt{\varepsilon}}(u+\varepsilon^{-1/2}v)(\varphi_{1}-\sigma_{1}\overline{bg})-\sigma_{1}^{2}\left(\frac{1}{\sqrt{\varepsilon}}a_{\perp}+\frac{1}{\varepsilon}s_{\perp}\right)(\overline{ug}+\overline{bv}+\overline{uv})+e_{1}(t)+\varepsilon^{-1/2}e_{2}(t).
\end{align}
Decomposing $a_{\perp},s_{\perp}$ further (to eliminate the constant $C$ in leading order terms), we note that $\norm{a_{\perp}} \leq \norm{b}+\norm{u}$ and $\norm{s_{\perp}} \leq \norm{g}+\norm{v}$.
Using the bound Eq.\eqref{eq:bound_ug_bv_bar}, along with the fact that $\frac{\norm{b}}{\sqrt{\varepsilon}}+\frac{\norm{g}}{\varepsilon} \leq \frac{D\beta}{\sqrt{\varepsilon}}$,
we reach
\begin{align}
  \partial_{t}Z \leq \frac{\sigma_{1}}{\sqrt{\varepsilon}}Z(\varphi_{1}-\sigma_{1}\overline{bg})&+\sigma_{1}^{2}D\beta^{2}Z+\sigma_{1}^{2}D\beta\sqrt{\varepsilon} Z\left(\frac{\norm{u}}{\sqrt{\varepsilon}}+\frac{\norm{v}}{\varepsilon}\right) \\
  &+\sigma_{1}^{2}\varepsilon^{-1/2}C\beta(\norm{u}\norm{v})+e_{1}(t)+\varepsilon^{-1/2}e_{2}(t).
\end{align}
Using the threshold on $u,v$ defined in $T$, and the bounds obtained on $e_{1}(t),e_{2}(t)$, we find 
\begin{align}
  \partial_{t}Z &\leq \frac{\sigma_{1}}{\sqrt{\varepsilon}}(\varphi_{1}-\sigma_{1}\overline{bg})Z+\sigma_{1}^{2}D\beta^{2}Z+\sigma_{1}^{2}D\beta^{2}Z\varepsilon^{\eta}+\sigma_{1}^{2}\varepsilon^{2\eta}C\beta^{3}+C^{2}D\beta^{2}.
\end{align}
Thus, as long as $\eta \geq \frac{1}{8}$, we find 
\begin{equation}
  \label{eq:Z_IDE}
  \partial_{t} Z \leq \frac{\sigma_{1}}{\sqrt{\varepsilon}}(\varphi_{1}-\sigma_{1}\overline{bg})Z+D\beta^{2}Z+C^{2}D\beta^{2}.
\end{equation}
Letting $t^{\circ}_{\varepsilon} = \frac{1}{4\sigma_{1}\varphi_{1}}\sqrt{\varepsilon}\log(\frac{1}{\varepsilon})$, we begin by considering 
the interval $[0,t^{\circ}_{\varepsilon}]$, for which it will be sufficient to study the differential inequality
\begin{equation}
  \label{eq:Z_IDE_cut}
  \partial_{t} Z \leq \frac{\sigma_{1}\varphi_{1}}{\sqrt{\varepsilon}}Z+D\beta^{2}Z+C^{2}D\beta^{2}.
\end{equation}
Defining
\begin{equation}
\lambda_{\varepsilon}(t) = \int_{0}^{t}\left(\frac{\sigma_{1}\varphi_{1}}{\sqrt{\varepsilon}}+D\beta^{2}(u)\right)du,
\end{equation}
we thus have 
\begin{equation}
  Z(t) \leq \exp(\lambda_{\varepsilon}(t))\int_{0}^{t}\exp(-\lambda_{\varepsilon}(u))C^{2}D\beta^{2}(u)du.
\end{equation}
Then, 
\begin{align}
  \lambda_{\varepsilon}(t)-\frac{\sigma_{1}\varphi_{1}}{\sqrt{\varepsilon}}t &\leq D\int_{0}^{t}\exp\left(\frac{2\sigma_{1}\varphi_{1}u}{\sqrt{\varepsilon}}\right)du \\
  & \leq D\sqrt{\varepsilon}\exp\left(\frac{2\sigma_{1}\varphi_{1}t}{\sqrt{\varepsilon}}\right).
\end{align}
Now, 
\begin{align}
  \int_{0}^{t}\exp\left(-\lambda_{\varepsilon}(u)\right)\beta^{2}(u)du &\leq \int_{0}^{t}\exp(-\lambda_{\varepsilon}(u))\exp\left(\frac{2\sigma_{1}\varphi_{1}u}{\sqrt{\varepsilon}}\right)du \\
  &\leq \int_{0}^{t}\exp\left(\frac{\sigma_{1}\varphi_{1}u}{\sqrt{\varepsilon}}\right)du \\
  &\leq D\sqrt{\varepsilon}\exp\left(\frac{\sigma_{1}\varphi_{1}t}{\sqrt{\varepsilon}}\right).
\end{align}
We thus reach 
\begin{equation}
  Z(t) \leq C^{2}D\sqrt{\varepsilon}\exp\left(\frac{2\sigma_{1}\varphi_{1}}{\sqrt{\varepsilon}}t+D\sqrt{\varepsilon}\exp\left(\frac{2\sigma_{1}\varphi_{1}}{\sqrt{\varepsilon}}t\right)\right).
\end{equation}
Noting that, on the considered time interval, $\frac{1}{D}\exp\left(\frac{\sigma_{1}\varphi_{1}t}{\sqrt{\varepsilon}}\right) \leq \beta(t) \leq D\exp\left(\frac{\sigma_{1}\varphi_{1}t}{\sqrt{\varepsilon}}\right)$,
we finally reach
\begin{align}
Z(t) &\leq C^{2}D\beta(t)\sqrt{\varepsilon}\exp\left(\frac{\sigma_{1}\varphi_{1}t}{\sqrt{\varepsilon}}\right) \\
&\leq C^{2}D\varepsilon^{1/4}\beta(t). 
\end{align}
Thus, as long as $\eta \in [\frac{1}{8},\frac{1}{4}]$, the proof is complete on this time interval. We now turn to the time interval 
$[t^{\circ}_{\varepsilon},\tilde{t}^{\circ}_{\varepsilon}]$ for some $\tilde{t}^{\circ}_{\varepsilon} \geq t^{\circ}_{\varepsilon}$ to be set later. Going back to the differential inequality Eq.\eqref{eq:Z_IDE}, we introduce 
\begin{equation}
  \Gamma_{\varepsilon}(t) = \int_{0}^{t}\left(\frac{\sigma_{1}}{\sqrt{\varepsilon}}(\varphi_{1}-\sigma_{1}\overline{bg})(u)+D\beta^{2}(u)\right)du,
\end{equation}
so that
\begin{equation}
  \label{eq:Z_bound_G}
  Z(t) \leq \exp(\Gamma_{\varepsilon}(t))\int_{0}^{t}\exp(-\Gamma_{\varepsilon}(u))C^{2}D\beta^{2}(u)du.
\end{equation}
Beginning with an upper bound on $\Gamma_{\varepsilon}(t)$, we write  
\begin{align}
  \Gamma_{\varepsilon}(t) &= \int_{0}^{t^{\circ}_{\varepsilon}}\left(\frac{\sigma_{1}}{\sqrt{\varepsilon}}(\varphi_{1}-\sigma_{1}\overline{bg})(u)+D\beta^{2}(u)\right)du+\int_{t^{\circ}_{\varepsilon}}^{t}\left(\frac{\sigma_{1}}{\sqrt{\varepsilon}}(\varphi_{1}-\sigma_{1}\overline{bg})(u)+D\beta^{2}(u)\right)du \\
  &= \Gamma_{\varepsilon}(t^{\circ}_{\varepsilon})+\Gamma_{\varepsilon}(t)-\Gamma_{\varepsilon}(t^{\circ}_{\varepsilon}).
\end{align}
Using the results proven for the interval $[0,t^{\circ}_{\varepsilon}]$, we have 
\begin{equation}
  \Gamma_{\varepsilon}(t^{\circ}_{\varepsilon}) \leq \frac{1}{4}\log(\frac{1}{\varepsilon})+D.
\end{equation}
Then, note that for any $t \in [t^{\circ}_{\varepsilon},\infty)$, $\beta(t) \leq D \varepsilon^{-1/4}$. Furthermore,
Corollary \ref{cor:bet_gam_bounds} gives that 
\begin{equation}
  0 \leq \varphi_{1}-\sigma_{1}\overline{bg} \leq \frac{D}{\sqrt{\varepsilon}}\exp\left(-\frac{2\sigma_{1}\varphi_{1}}{\sqrt{\varepsilon}}t\right)+D\varepsilon.
\end{equation}
Thus,
\begin{align}
  \Gamma_{\varepsilon}(t)-\Gamma_{\varepsilon}(t^{\circ}_{\varepsilon}) &\leq \frac{D}{\varepsilon}\int_{t^{\circ}_{\varepsilon}}^{t}\exp\left(-\frac{2\sigma_{1}\varphi_{1}}{\sqrt{\varepsilon}}u\right)du+\frac{D}{\sqrt{\varepsilon}}(t-t^{\circ}_{\varepsilon}) \\
  &\leq D+\frac{D}{\sqrt{\varepsilon}}(t-t^{\circ}_{\varepsilon}),
\end{align}
and we reach the upper bound 
\begin{equation}
  \label{eq:bound_G_1}
  \exp(\Gamma_{\varepsilon}(t)) \leq D\varepsilon^{-1/4}\exp\left(\frac{D}{\sqrt{\varepsilon}}(t-t^{\circ}_{\varepsilon})\right).
\end{equation}
Turning to the integral term in Eq.\eqref{eq:Z_bound_G}, we begin with the term
\begin{align}
  \int_{0}^{t^{\circ}_{\varepsilon}}\exp(-\Gamma_{\varepsilon}(u))C^{2}D\beta^{2}(u)du &\leq C^{2}D\int_{0}^{t^{\circ}_{\varepsilon}}\exp\left(\int_{0}^{u}\frac{\sigma_{1}}{\sqrt{\varepsilon}}\left(\varphi_{1}+\sigma_{1}\overline{bg}\right)\right)(u)du \\
  &\leq C^{2}D\int_{0}^{t^{\circ}_{\varepsilon}}\exp\left(\frac{\sigma_{1}\varphi_{1}}{\sqrt{\varepsilon}}u+\frac{\sigma_{1}\sigma_{1}'}{\sqrt{\varepsilon}}\int_{0}^{u}\beta(r)\gamma(r)dr\right) \\
  &\leq C^{2}D\int_{0}^{t^{\circ}_{\varepsilon}}\exp\left(\frac{\sigma_{1}\varphi_{1}}{\sqrt{\varepsilon}}u+\frac{\sigma_{1}\sigma_{1}'}{4}\int_{0}^{u}\exp\left(\frac{2\sigma_{1}\varphi_{1}}{\sqrt{\varepsilon}}r\right)dr\right),
\end{align}
where we used Corollary \ref{cor:bet_gam_bounds} in the third line.
Since $u \leq t^{\circ}_{\varepsilon}$, it holds that $\int_{0}^{u}\exp\left(\frac{2\sigma_{1}\varphi_{1}}{\sqrt{\varepsilon}}r\right)dr \leq D$,
and we reach 
\begin{equation}
  \label{eq:bound_G_2}
  \int_{0}^{t^{\circ}_{\varepsilon}}\exp(-\Gamma_{\varepsilon}(u))C^{2}D\beta^{2}(u)du \leq C^{2}D\varepsilon^{1/4}.
\end{equation}
Continuing with the term $\int_{t^{\circ}_{\varepsilon}}^{t}\exp(-\Gamma_{\varepsilon}(u))C^{2}D\beta^{2}(u)du$,  we require a lower bound on $\Gamma_{\varepsilon}(u)$, for \\
$t \in [t^{\circ}_{\varepsilon},\infty)$. We then write
\begin{align}
  \Gamma_{\varepsilon}(t^{\circ}_{\varepsilon}) &\geq \frac{\sigma_{1}\varphi_{1}}{\sqrt{\varepsilon}}t^{\circ}_{\varepsilon}-\frac{\sigma_{1}}{\sqrt{\varepsilon}}\sigma'_{1}\int_{0}^{t^{\circ}_{\varepsilon}}\beta(u)\gamma(u)du \\
  &\geq \frac{1}{4}\log(\frac{1}{\varepsilon})-D\int_{0}^{t^{\circ}_{\varepsilon}}\exp(\frac{2\sigma_{1}\varphi_{1}}{\sqrt{\varepsilon}}u)du \\
  &\geq \frac{1}{4}\log(\frac{1}{\varepsilon})-D,
\end{align}
where we used Corollary \ref{cor:bet_gam_bounds} again in the second line.
Then, noting that, for any $t \in [t^{\circ}_{\varepsilon},\infty)$, $\beta(t) \geq \frac{1}{D}\varepsilon^{-1/4}$, we may write
\begin{align}
  \Gamma_{\varepsilon}(t)-\Gamma_{\varepsilon}(t^{\circ}_{\varepsilon}) &\geq \frac{1}{D}\int_{t^{\circ}_{\varepsilon}}^{t}\beta^{2}(r)dr \\
  &\geq \frac{1}{D}\varepsilon^{-1/2}(t-t^{\circ}_{\varepsilon}).
\end{align}
We thus reach 
\begin{align}
  \label{eq:bound_G_3}
  \int_{t^{\circ}_{\varepsilon}}^{t}\exp(-\Gamma_{\varepsilon}(u))C^{2}D\beta^{2}(u)du &\leq C^{2}D\varepsilon^{-1/2}\exp\left(-\Gamma_{\varepsilon}(t_{\varepsilon}^{\circ})\right)\int_{t^{\circ}_{\varepsilon}}^{t}\exp\left(-\frac{1}{D\sqrt{\varepsilon}}(u-t^{\circ}_{\varepsilon})\right)du \\
  &\leq C^{2}D\varepsilon^{1/4}.
\end{align}
Combining the bounds obtained at Eq.\eqref{eq:bound_G_1},\eqref{eq:bound_G_2},\eqref{eq:bound_G_3}, and recalling that, for $t \in [t^{\circ}_{\varepsilon},\infty)$, $\beta(t) \geq \frac{1}{D}\varepsilon^{-1/4}$, we obtain
\begin{equation}
  Z(t) \leq C^{2}D\beta(t)\varepsilon^{1/4}\exp\left(\frac{D}{\sqrt{\varepsilon}}(t-t^{\circ}_{\varepsilon})\right).
\end{equation}
We now set $\tilde{t}^{\circ}_{\varepsilon} = t^{\circ}_{\varepsilon}+c\sqrt{\varepsilon}\log(\frac{1}{\varepsilon})$ for some small constant $c$ to be determined in such a way that the above bound becomes 
\begin{equation}
  Z(t) \leq C^{2}D\beta(t)\varepsilon^{1/4-cD}.
\end{equation}
Thus, as long as $cD \leq \frac{1}{8}$, the proof is complete for the time interval $[0,\tilde{t}^{\circ}_{\varepsilon}]$.
\end{proof}
We now turn to a stability result regarding the approximation guarantees established in Theorem \ref{thm:harmo_0+1_tube}. A core limitation in the proof of the latter 
comes from the term $D\beta^{2}Z$ in the differential inequality \eqref{eq:Z_IDE}, which comes from controlling $\vert \overline{as}-\overline{bg} \vert$. This term grows unbounded when, in fact, the difference $\vert \overline{as}-\overline{bg} \vert$ becomes arbitrarily small at times of order $K\sqrt{\varepsilon}\log(\frac{1}{\varepsilon})$, for $K$ sufficiently large.
\begin{thm}
  \label{thm:stable_01}
  Let $t = t^{\circ}_{\varepsilon} = \frac{1}{4\sigma_{1}\varphi_{1}}\sqrt{\varepsilon}\log(\frac{1}{\varepsilon})$. There exists a positive constant $\delta$ depending only on the problem data such that, for any $K\geq\frac{1}{4\sigma_{1}\varphi_{1}}$ independent of $\varepsilon$ and time horizon $t_{\varepsilon} = K\sqrt{\varepsilon}\log(\frac{1}{\varepsilon})$, there exists a constant $C$ depending only on the problem data and $K$ such that, for $\varepsilon$ sufficiently small and 
  any $t\in [t^{\circ}_{\varepsilon},t_{\varepsilon}]$,
   \begin{align}
    \label{eq:stable_01_bounds_0}
    \vert \overline{a}-\overline{b} \vert(t) \leq Ct+C\varepsilon\beta^{2}(t), \quad \vert \overline{s}(t) \vert \leq Ct, \\
    \norm{(a_{\perp}-b)(t,.)}+\varepsilon^{-1/2}\norm{(s_{\perp}-g)(t,.)} \leq C\varepsilon^{\delta}\beta(t).
    \label{eq:stable_01_bounds_1}
  \end{align}
\end{thm}
\begin{proof}
  Using the notations of Theorem \ref{thm:harmo_0+1_tube}, we consider the hitting time
\begin{align}
  T := \inf \bigg\{t \geq 0 : \norm{a}+\varepsilon^{-1/2}\norm{s} \geq C\beta \quad \mbox{or} \quad \vert \varphi_{1}-\sigma_{1}\overline{as} \vert \geq C\bigg\},
\end{align}
for some constant $C$ to be set later. First, note that the argument used to establish the inequalities
\begin{align}
  \vert \overline{a}-\overline{b} \vert(t) &\leq Ct+C\varepsilon\beta^{2}(t), \\
  \vert \overline{s}(t) \vert &\leq Ct,
\end{align}
in the proof of Theorem \ref{thm:harmo_0+1_tube} is valid for any $0 \leq t \leq T$.
Thus, for any $0 \leq t \leq t_{\varepsilon} \wedge T$, we find 
\begin{equation}
  \label{eq:approx_phi0}
  \vert \frac{\varphi_{0}}{\sigma_{0}}- \overline{a}(t)\vert \leq KC\sqrt{\varepsilon}\log(\frac{1}{\varepsilon})+C\sqrt{\varepsilon}.
\end{equation}
We then write 
\begin{align}
  a\partial_{t}s &= a^{2}(1-s^{2})\sum_{k=1}^{\infty}k\sigma_{k}s^{k-1}\left(\varphi_{k}-\sigma_{k}\overline{as^{k}}\right) \\
  s\partial_{t}a &= \frac{1}{\varepsilon}\sum_{k=0}^{\infty}\sigma_{k}s^{k+1}\left(\varphi_{k}-\sigma_{k}\overline{as^{k}}\right),
\end{align}
so that 
\begin{align}
  \big \vert \partial_{t} \overline{as}-\frac{1}{\varepsilon}\sigma_{0}\overline{s}\left(\varphi_{0}-\sigma_{0}\overline{a}\right)-\left(\overline{a^{2}}+\frac{1}{\varepsilon}\overline{s^{2}}\right)\sigma_{1}\left(\varphi_{1}-\sigma_{1}\overline{as}\right) \big \vert \leq C^{3}D\sqrt{\varepsilon}\beta^{3},
\end{align}
and thus
\begin{align}
  \big \vert \partial_{t} \overline{as}-\left(\overline{a^{2}}+\frac{1}{\varepsilon}\overline{s^{2}}\right)\sigma_{1}\left(\varphi_{1}-\sigma_{1}\overline{as}\right) \big \vert \leq C^{3}D\sqrt{\varepsilon}\beta^{3}+DKC^{2}\log^{2}(\frac{1}{\varepsilon}).
\end{align}
Recalling that there exists a constant $D$ depending only on the problem data such that, for any $t \geq 0$, $\beta(t) \leq D\varepsilon^{-1/4}$, we find 
\begin{equation}
  \big \vert \partial_{t}\left(\varphi_{1}-\sigma_{1}\overline{as}\right)^{2}+\sigma_{1}^{2}\left(\overline{a^{2}}+\frac{1}{\varepsilon}\overline{s^{2}}\right)\left(\varphi_{1}-\sigma_{1}\overline{as}\right)^{2} \big \vert \leq C^{3}D \varepsilon^{-1/4}\vert \varphi_{1}-\sigma_{1}\overline{as} \vert.
\end{equation}
Allowing the exact value of $D$ to change from line to line in what follows, recall from Proposition \ref{prop:app_bg} that $\beta(t^{\circ}_{\varepsilon}) \geq \frac{1}{D}\varepsilon^{-1/4}$.
This, along with Theorem \ref{thm:harmo_0+1_tube}, gives the existence of a (strictly positive) constant $c$ such that, for any $t$ in the interval 
$[t^{\circ}_{\varepsilon},t^{\circ}_{\varepsilon}+c\sqrt{\varepsilon}\log(\frac{1}{\varepsilon})]$ and $\varepsilon$ sufficiently small,
\begin{equation}
  \overline{a^{2}}+\frac{1}{\varepsilon}\overline{s^{2}} \geq \frac{1}{D}\varepsilon^{-1/2},
\end{equation}
giving, on the same time interval, 
\begin{equation}
   \partial_{t}\left(\varphi_{1}-\sigma_{1}\overline{as}\right)^{2} \leq -\frac{1}{D\sqrt{\varepsilon}}\left(\varphi_{1}-\sigma_{1}\overline{as}\right)^{2}+C^{3}D\varepsilon^{-1/4}\big \vert \varphi_{1}-\sigma_{1}\overline{as} \big \vert.
\end{equation}
Lemma \ref{lem:Ber_1/2} started at $t = t^{\circ}_{\varepsilon}$ then gives 
\begin{equation}
  \label{eq:stable_phi1}
  \big \vert \varphi_{1}-\sigma_{1}\overline{as} \big \vert \leq \varphi_{1}\exp\left(-\frac{(t-t^{\circ}_{\varepsilon})}{D\sqrt{\varepsilon}}\right)+C^{3}D^{2}\varepsilon^{1/4}\left(1-\exp\left(-\frac{(t-t^{\circ}_{\varepsilon})}{D\sqrt{\varepsilon}}\right)\right).
\end{equation}
Now, note that Eq.\eqref{eq:dt_a_p},\eqref{eq:dt_s_p} and the bounds at \eqref{eq:up_bound_e1_e2_1}-\eqref{eq:up_bound_e1_e2_2} are valid for any $0 \leq t \leq T$, so that we may write 
\begin{align}
  \partial_{t}a_{\perp} &= \frac{1}{\varepsilon}\sigma_{1}s_{\perp}\left(\varphi_{1}-\sigma_{1}\overline{as}\right)+e_{1}(t) \\
  \partial_{t}s_{\perp} &= \sigma_{1}a_{\perp}\left(\varphi_{1}-\sigma_{1}\overline{as}\right)+e_{2}(t),
\end{align}
for some continuous functions $e_{1}(t),e_{2}(t)$ which verify $\sup_{t \leq T\wedge t_{\varepsilon}} \vert e_{1}(t) \vert +\varepsilon^{-1/2} \vert e_{2}(t) \vert \leq C^{2}D\beta^{2}$.
Then, adding and subtracting $\overline{bg}$,
\begin{align}
  \partial_{t}a_{\perp} &= \frac{\sigma_{1}}{\varepsilon}s_{\perp}\left(\varphi_{1}-\sigma_{1}\overline{bg}\right)-\frac{\sigma_{1}^{2}}{\varepsilon}s_{\perp}\left(\overline{as}-\overline{bg}\right)+e_{1}(t), \\
  \partial_{t}s_{\perp} &= \sigma_{1}a_{\perp}\left(\varphi_{1}-\sigma_{1}\overline{bg}\right)-\sigma_{1}^{2}a_{\perp}\left(\overline{as}-\overline{bg}\right)+e_{2}(t).
\end{align}
Recalling the definition $Z = \norm{u}+\varepsilon^{-1/2}\norm{v}$, with $u = a_{\perp}-b$ and $v = s_{\perp}-g$, we have, for any $0 \leq t \leq T$,
\begin{equation}
  \label{eq:partial_Z_stable}
  \partial_{t}Z \leq \frac{\sigma_{1}}{\sqrt{\varepsilon}}\left(\left(\varphi_{1}-\sigma_{1}\overline{bg}\right)+\sigma_{1}\vert \overline{as}-\overline{bg}\vert\right)Z+\frac{D}{\sqrt{\varepsilon}}\beta\vert \overline{as}-\overline{bg}\vert+e_{1}+\varepsilon^{-1/2}e_{2}.
\end{equation}
At this point, we have established that the bounds \eqref{eq:stable_phi1} and \eqref{eq:partial_Z_stable} are valid on the time interval \\
$[t^{\circ}_{\varepsilon},t^{\circ}_{\varepsilon}+c\sqrt{\varepsilon}\log(\frac{1}{\varepsilon})]$. Thus, at time 
$\overline{t}^{\circ}_{\varepsilon} = t^{\circ}_{\varepsilon}+c\sqrt{\varepsilon}\log(\frac{1}{\varepsilon})$, the following estimate holds:
\begin{align}
  \vert \varphi_{1}-\sigma_{1}\overline{as}\vert &\leq \varphi_{1}\varepsilon^{c/D}+C^{3}D\varepsilon^{1/4},
\end{align}
and, recalling that Corollary \ref{cor:bet_gam_bounds} implies the bound $0 \leq \varphi_{1}-\sigma_{1}\overline{bg} \leq \frac{D}{\sqrt{\varepsilon}}\exp(-\frac{2\sigma_{1}\varphi_{1}}{\sqrt{\varepsilon}}t)+D\varepsilon$,
we find that
\begin{equation}
  \vert \varphi_{1}-\sigma_{1}\overline{bg} \vert \leq D\varepsilon^{2\sigma_{1}\varphi_{1}c \wedge 1}.
\end{equation}
We conclude that there exists an exponent $\delta>0$, depending only on the problem data, such that, at $t=\overline{t}^{\circ}_{\varepsilon}$, 
\begin{equation}
  \vert \overline{as}-\overline{bg} \vert \leq D\varepsilon^{\delta}.
\end{equation}
Let us now introduce the hitting time 
\begin{equation}
  T' = \inf\left\{t \geq \overline{t}^{\circ}_{\varepsilon} : \overline{a^2}(t)+\frac{1}{\varepsilon}\overline{s^2}(t) \leq \frac{1}{D}\varepsilon^{-1/2}\right\}. 
\end{equation}
Clearly $T'>\overline{t}^{\circ}_{\varepsilon}$, and the inequality \eqref{eq:stable_phi1} remains valid on the time interval $[\overline{t}^{\circ}_{\varepsilon},T' \wedge T\wedge t_{\varepsilon}]$.
Noting that the r.h.s. of the inequality \eqref{eq:stable_phi1} is decreasing for any $t \geq \overline{t}^{\circ}_{\varepsilon}$, we find that, for any $t$ in the time interval $[\overline{t}^{\circ}_{\varepsilon},T' \wedge T\wedge t_{\varepsilon}]$ and 
$\varepsilon$ sufficiently small,
\begin{equation}
  \partial_{t}Z \leq D\varepsilon^{-1/2+\delta}Z+\frac{D}{\varepsilon^{3/4}}\left(\varphi_{1}\exp\left(-\frac{t}{D\sqrt{\varepsilon}}\right)+C^{3}D^{2}\varepsilon^{1/4}\right).
\end{equation}
Integrating the above, we find that, for any $t \in [\overline{t}^{\circ}_{\varepsilon},T' \wedge T\wedge t_{\varepsilon}]$,
\begin{align*}
  Z(t) &\leq Z(\overline{t}^{\circ}_{\varepsilon})\exp\left(\varepsilon^{-1/2+\delta}(t-\overline{t}^{\circ}_{\varepsilon})\right)+\frac{D}{\varepsilon^{3/4}}\int_{\overline{t}^{\circ}_{\varepsilon}}^{t}\exp\left(\varepsilon^{-1/2+\delta}(t-s)\right)\left(D\exp\left(-\frac{s}{D\sqrt{\varepsilon}}\right)+C^{3}D^{2}\varepsilon^{1/4}\right)ds,
\end{align*}
and, recalling from Theorem \ref{thm:harmo_0+1_tube} that $Z(\overline{t}^{\circ}_{\varepsilon}) \leq C\varepsilon^{1/8}\beta(\overline{t}^{\circ}_{\varepsilon})$,
\begin{align*}
  Z(t) &\leq C\varepsilon^{1/8}\beta(\overline{t}^{\circ}_{\varepsilon})\exp\left(\frac{\varepsilon^{\delta}}{\sqrt{\varepsilon}}t\right)+\frac{D}{\varepsilon^{3/4}}\exp\left(\frac{\varepsilon^{\delta}}{\sqrt{\varepsilon}}t\right)\left(D^{2}\sqrt{\varepsilon}\left[\exp\left(-\frac{\overline{t}_{\varepsilon}^{\circ}}{D\sqrt{\varepsilon}}\right)-\exp\left(-\frac{t}{D\sqrt{\varepsilon}}\right)\right]\right), \\
  &+\frac{C^{3}D^{3}}{\varepsilon^{1/2}}\exp\left(\frac{\varepsilon^{\delta}}{\sqrt{\varepsilon}}t\right)\left(\frac{\sqrt{\varepsilon}}{\varepsilon^{\delta}}\left[\exp\left(-\frac{\varepsilon^{\delta}\overline{t}_{\varepsilon}^{\circ}}{\sqrt{\varepsilon}}\right)-\exp\left(-\frac{\varepsilon^{\delta}t}{\sqrt{\varepsilon}}\right)\right]\right).
\end{align*}
Thus, for any $K>0$ independent of $\varepsilon$ and $\varepsilon$ sufficiently small, we find
\begin{align}
  Z(t_{\varepsilon}) &\leq C\varepsilon^{1/8}\beta(\overline{t}^{\circ}_{\varepsilon})+D\varepsilon^{-1/4+\delta}+C^{3}D^{3}\varepsilon^{-\delta}, \\
  &\leq D\varepsilon^{\delta}\beta(t_{\varepsilon}),
\end{align}
where we have used that $ \frac{1}{D}\varepsilon^{-1/4} \leq \beta(t) \leq D\varepsilon^{-1/4}$ for any $t \in [t^{\circ}_{\varepsilon},\infty)$. 
From this bound, we deduce that the desired bounds hold for any $t$ in $[\overline{t}^{\circ}_{\varepsilon},T' \wedge T\wedge t_{\varepsilon}]$, and, by definition of $Z$, 
that $T \wedge T' > K\sqrt{\varepsilon}\log(\frac{1}{\varepsilon})$, concluding the proof.
\end{proof}
As a corollary, we obtain the following result on the stability of the learned components $\varphi_{0},\varphi_{1}$, which also removes the $\log(\frac{1}{\varepsilon})$ factor in 
the bound \eqref{eq:approx_phi0}.
\begin{cor}
  \label{cor:stable_01}
  There exist constants $C,\delta$ and $K^{\circ}>\frac{1}{4\sigma_{1}\varphi_{1}}$ independent of $\varepsilon$, such that, for any $K>K^{\circ}$, $t_{\varepsilon} = K\sqrt{\varepsilon}\log(\frac{1}{\varepsilon})$, $t \in [K^{\circ}\sqrt{\varepsilon}\log(\frac{1}{\varepsilon}),t_{\varepsilon}]$ and $\varepsilon$ sufficiently small,
  \begin{align}
    &\vert \varphi_{0}-\sigma_{0}\overline{a}(t) \vert \leq C\sqrt{\varepsilon}, \quad \vert \varphi_{1}-\sigma_{1}\overline{as}(t) \vert \leq C\varepsilon^{1/4}, \quad \vert \overline{s}(t) \vert \leq Ct, \\
    &\norm{a_{\perp}(t,.)-\left(\frac{\varphi_{1}}{\sigma_{1}'}\right)^{1/2}\varepsilon^{-1/4}a_{init,\perp}(.)}+\varepsilon^{-1/2}\norm{s_{\perp}(t,.)-\left(\frac{\varphi_{1}}{\sigma_{1}'}\right)^{1/2}\varepsilon^{1/4}a_{init,\perp}(.)} \leq C\varepsilon^{-1/4+\delta}.
  \end{align}
\end{cor}
\begin{proof}
  With $T$ the same stopping time as the one introduced in the proof of Theorem \ref{thm:harmo_0+1_tube}, we have the upper bound \eqref{eq:stable_phi1},
  valid for any $t_{\varepsilon}^{\circ} \leq t \leq T \wedge t_{\varepsilon}$. Thus, the part of the statement regarding $\vert \varphi_{1}-\sigma_{1}\overline{as} \vert$ is true for some time $K_{1}\sqrt{\varepsilon}\log(\frac{1}{\varepsilon})$, with $\frac{1}{4\sigma_{1}\varphi_{1}} < K_{1} < K$. Now, for any $t \in [K_{1}\sqrt{\varepsilon}\log(\frac{1}{\varepsilon}),t_{\varepsilon}]$, $\vert \varphi_{1}-\sigma_{1}\overline{as} \vert \leq C\varepsilon^{1/4}$ and 
  $\vert \overline{s}(t) \vert \leq CK_{1}\sqrt{\varepsilon}\log(\frac{1}{\varepsilon})$, we find  
  \begin{align}
   \vert \partial_{t}\overline{a} - \frac{\sigma_{0}}{\varepsilon}(\varphi_{0}-\sigma_{0}\overline{a}(t))\vert  \leq CDK_{1}\varepsilon^{-1/4}\log(\frac{1}{\varepsilon})+C^{2}D\varepsilon^{-1/2},
  \end{align}
  so that, for $\varepsilon$ sufficiently small,
  \begin{align}
    \partial_{t} \left(\varphi_{0}-\sigma_{0}\overline{a}\right)^{2} \leq -\frac{\sigma_{0}}{\varepsilon}\left(\varphi_{0}-\sigma_{0}\overline{a}\right)^{2}+\frac{C^{2}D}{\sqrt{\varepsilon}}\vert \varphi_{0}-\sigma_{0}\overline{a}\vert.
  \end{align}
  Lemma \ref{lem:Ber_1/2} started at $t = K_{1}\sqrt{\varepsilon}\log(\frac{1}{\varepsilon})$ then gives 
  \begin{align}
    \vert \varphi_{0}-\sigma_{0}\overline{a}(t) \vert &\leq C\sqrt{\varepsilon}\log(\frac{1}{\varepsilon})\exp\left(-\frac{\sigma_{0}}{2\varepsilon}\left(t-K_{1}\sqrt{\varepsilon}\log(\frac{1}{\varepsilon})\right)\right) \notag \\
    &\hspace{1cm}+C^{2}D\sqrt{\varepsilon}\left(1-\exp\left(-\frac{\sigma_{0}}{2\varepsilon}\left(t-K_{1}\sqrt{\varepsilon}\log(\frac{1}{\varepsilon})\right)\right)\right),
  \end{align}
  implying the result for $\vert \varphi_{0}-\sigma_{0}\overline{a} \vert$, for any $K^{\circ}>K_{1}$ independent of $\varepsilon$. Finally, choosing any $K^{\circ} \geq \frac{1}{2\sigma_{1}\varphi_{1}}$,
  Corollary \ref{cor:bet_gam_bounds} implies that there exists a constant $D$ depending only on $\varphi_{1},\sigma_{1},\sigma_{1}',M$ such that, for any $t \geq K^{\circ}\sqrt{\varepsilon}\log(\frac{1}{\varepsilon})$,
  \begin{equation}
    \norm{b(t,.)-\left(\frac{\varphi_{1}}{\sigma_{1}'}\right)^{1/2}\varepsilon^{-1/4}a_{init,\perp}(.)}+\varepsilon^{-1/2}\norm{g(t,.)-\left(\frac{\varphi_{1}}{\sigma_{1}'}\right)^{1/2}\varepsilon^{1/4}a_{init,\perp}(.)} \leq D\varepsilon^{1/4},
  \end{equation}
  implying the last part of the statement.
\end{proof}
Theorem \ref{thm:main_tube01} is a direct consequence of Theorem \ref{thm:harmo_0+1_tube} and \ref{thm:stable_01}. We now give the proofs of Proposition \ref{prop:main_0} and \ref{prop:main_1}.
\begin{proof}[Proof of Proposition \ref{prop:main_0}]
  We begin by noting that, for any $\eta \in (0,1)$, 
\begin{equation}
  \vert \overline{b}(t)-\frac{\varphi_{0}}{\sigma_{0}} \vert \leq \varepsilon^{\eta} \iff \exp\left(\frac{\sigma_{0}^{2}}{\varepsilon}t\right) \geq \frac{\vert \overline{a}(0)-\frac{\varphi_{0}}{\sigma_{0}} \vert}{\varepsilon^{\eta}},
\end{equation}
so that 
\begin{equation}
  \inf\left\{t \geq 0 : \vert \overline{b}(t)-\frac{\varphi_{0}}{\sigma_{0}} \vert \leq \varepsilon^{\eta} \right\} \leq \frac{\left(\eta\log(\frac{1}{\varepsilon})+\log (\vert \overline{a}(0)-\frac{\varphi_{0}}{\sigma_{0}} \vert )\right)_{+}}{\sigma_{0}^{2}}\varepsilon.
\end{equation}
Letting $t_{0}$ denote the right-hand-side of the above inequality, this time horizon falls within the interval of validity of Theorem \ref{thm:harmo_0+1_tube}, so there exists a constant $C$ depending only on the problem data such that 
\begin{equation}
  \vert \overline{a}-\overline{b} \vert (t_{0}) \leq Ct_{0}+C\varepsilon\beta^{2}(t_{0}).
\end{equation}
Using the bounds on $\beta(t)$ given at Proposition \ref{prop:app_bg} and keeping the leading order terms in $\varepsilon$ concludes the proof.
\end{proof}
\begin{proof}[Proof of Proposition \ref{prop:main_1}]
  First note that, for all $t \geq 0$,
  \begin{align}
  \vert \overline{as}(t)-\overline{bg}(t) \vert &\leq \vert \overline{a_{\perp}s_{\perp}}(t) -\overline{bg}(t) \vert +\vert \bar{a}(t)\bar{s}(t) \vert, \notag \\
  &\leq \norm{a_{\perp}(t,.)}\norm{s_{\perp}(t,.)-g(t,.)}+\norm{g(t,.)}\norm{a_{\perp}(t,.)-b(t,.)}+\vert \overline{a}(t) \vert\vert \overline{s}(t) \vert.
  \end{align}
  Owing to Theorem \ref{thm:harmo_0+1_tube} and \ref{thm:stable_01}, for any constant $K$ and $0 \leq t \leq K\sqrt{\varepsilon}\log(\frac{1}{\varepsilon})$, for $\varepsilon$ sufficiently small,
  there exists a constant $C$ depending only on the problem data and $K$, and a constant $\delta \in (0,\frac{1}{4})$ depending on the problem data, such that 
  \begin{align}
    &\vert \overline{a}(t) \vert \leq C, \quad \vert \overline{s}(t) \vert \leq C\sqrt{\varepsilon}\log(\frac{1}{\varepsilon}), \\
    &\norm{a_{\perp}(t,.)} \leq C\beta(t), \quad \norm{g(t,.)} \leq C\sqrt{\varepsilon}\beta(t), \\
    &\norm{s_{\perp}(t,.)-g(t,.)} \leq C\varepsilon{1/2+\delta}\beta(t), \quad \norm{a_{\perp}(t,.)-b(t,.)} \leq C\varepsilon^{\delta}\beta(t).
  \end{align}
  In what follows, we will allow $C$ to change from line to line (as long as it only depends on the problem data and $K$).
  Now, recall from Proposition \ref{prop:app_bg} that there exists a constant $D$ depending only on the problem data such that, for all $t \geq 0$, $\beta(t) \leq D\varepsilon^{-1/4}$, so that we reach, 
  for any $0 \leq t \leq K\sqrt{\varepsilon}\log(\frac{1}{\varepsilon})$, 
  \begin{equation}
    \vert \overline{as}(t)-\overline{bg}(t) \vert \leq C\varepsilon^{\delta}.
  \end{equation}
  Using Corollary \ref{cor:bet_gam_bounds} at time $t_{1} = \left(\frac{1}{4\sigma_{1}\varphi_{1}}+c_{1}\right)\sqrt{\varepsilon}\log(\frac{1}{\varepsilon})$, we find that 
  \begin{equation}
    \vert \frac{\varphi_{1}}{\sigma_{1}}-\overline{bg}(t_{1}) \vert \leq C\varepsilon^{2\sigma_{1}\varphi_{1}c_{1}}.
  \end{equation}
  We conclude that 
  \begin{equation}
    \vert \overline{as}(t_{1}) -\frac{\varphi_{1}}{\sigma_{1}}\vert \leq C\varepsilon^{\delta \wedge 2\sigma_{1}\varphi_{1}c_{1}}, 
  \end{equation}
  giving the first part of Theorem \ref{thm:main_tube01}. The second part is a direct consequence of Corollary \ref{cor:stable_01}.
\end{proof}
\subsection{Learning the quadratic component}
\label{sec:proof_quad}
We now turn to the learning process for the term $\varphi_{2}$, with the initial condition prescribed by Theorem \ref{thm:stable_01} and Corollary \ref{cor:stable_01}. Recall that we study the following truncated version of the system \eqref{eq:main_mf_dyn1}-\eqref{eq:main_mf_dyn2}:
\begin{align}
  \label{eq:quad_dyn1}
  \partial_{t} a(t,\omega) &= \frac{\sigma_{0}}{\varepsilon}\left(\varphi_{0}-\sigma_{0}\overline{a}(t)\right)+\frac{\sigma_{1}}{\varepsilon}s(t,\omega)\left(\varphi_{1}-\sigma_{1}\overline{as}(t)\right)+\frac{\sigma_{2}\varphi_{2}}{\varepsilon}s^{2}(t,\omega), \\
  \partial_{t} s(t,\omega) &= a(t,\omega)\sigma_{1}\left(\varphi_{1}-\sigma_{1}\overline{as}(t)\right)+2\sigma_{2}\varphi_{2}a(t,\omega)s(t,\omega).
  \label{eq:quad_dyn2}
\end{align}
Choosing our starting time at $t_{0} = K^{\circ}\sqrt{\varepsilon}\log(\frac{1}{\varepsilon})$, Corollary \ref{cor:stable_01} implies that the first two terms in the r.h.s. of 
\eqref{eq:quad_dyn1} are of the same order of magnitude as the third, and the first term in the r.h.s. of \eqref{eq:quad_dyn2} is of the same order as the second. 
Thus, we cannot neglect them, and will use Proposition \ref{prop:approx_dyn_set} to determine the correct idealized system. To obtain sufficiently sharp bounds for the error terms inherent 
to the approximation result of Proposition \ref{prop:approx_dyn_set}, we will need to make changes in time and scale. 
In turn, to better understand the timescale on which we need to control the solution to \eqref{eq:quad_dyn1}-\eqref{eq:quad_dyn2}, we first study the following idealized dynamical system 
\begin{align}
  \label{eq:pure_quad_1}
  \partial_{t} a_{q}(t,\omega) &= \frac{\sigma_{2}\varphi_{2}}{\varepsilon}s_{q}^{2}(t,\omega), \\
  \partial_{t} s_{q}(t,\omega) &= 2\sigma_{2}\varphi_{2}a_{q}(t,\omega)s_{q}(t,\omega),
  \label{eq:pure_quad_2}
\end{align}
with any initial condition $a_{q}(t_{0},.),s_{q}(t_{0},.)$ such that, for some constant $C$, 
\begin{align}
    \label{eq:init_quad1}
    &\vert \varphi_{0}-\sigma_{0}\overline{a_{q}}(t_{0}) \vert \leq C\sqrt{\varepsilon}, \quad \vert \varphi_{1}-\sigma_{1}\overline{a_{q}s_{q}}(t_{0}) \vert \leq C\varepsilon^{1/4}, \quad \vert \overline{s_{q}}(t_{0}) \vert \leq C\sqrt{\varepsilon}\log(\frac{1}{\varepsilon}), \\
    &\norm{a_{q,\perp}(t_{0},.)-\left(\frac{\varphi_{1}}{\sigma_{1}'}\right)^{1/2}\varepsilon^{-1/4}a_{init,\perp}(.)}+\varepsilon^{-1/2}\norm{s_{q,\perp}(t_{0},.)-\left(\frac{\varphi_{1}}{\sigma_{1}'}\right)^{1/2}\varepsilon^{1/4}a_{init,\perp}(.)} \leq C\varepsilon^{-1/4+\delta}.
    \label{eq:init_quad2}
\end{align}
Note that the Cauchy-Lipschitz theorem ensures the existence of a time $T_{q} > t_{0}$ such that, for all $\omega \in \Omega$, 
there exists a continuously differentiable solution to \eqref{eq:pure_quad_1}-\eqref{eq:pure_quad_2} on $[t_0, T_q]$. The following result shows that 
a significant fraction of neurons change scale on a timescale of order $\varepsilon^{1/4}$.
\begin{lem}
  \label{lem:pure_quad_sol}
  Let $t_{0} = K^{\circ}\sqrt{\varepsilon}\log(\frac{1}{\varepsilon})$, with $K^{\circ}$ as in Corollary \ref{cor:stable_01}.
  Then, for any 
  $t_{0} \leq t \leq T_{q}$, the solution to \eqref{eq:pure_quad_1}-\eqref{eq:pure_quad_2} with initial condition verifying \eqref{eq:init_quad1}-\eqref{eq:init_quad2} verifies 
  \begin{equation}
    \label{eq:quad_link_clean}
    a_{q}^{2}(t,\omega)-a_{q}(t_{0},\omega)^{2} = \frac{1}{2\varepsilon}\left(s_{q}^{2}(t,\omega)-s_{q}(t_{0},\omega)^{2}\right),
  \end{equation}
  Furthermore, there exists a constant $K$ depending only on the problem data such that, for any $c \in (0,1)$, $\varepsilon$ sufficiently small and any $\omega$ in the set 
  \begin{equation}
    \mathcal{U}(c) = \left\{\omega : a_{init,\perp}(\omega) \geq \frac{K}{c}\varepsilon^{1/4} \right\},
  \end{equation}
  $a_{q}(t,\omega)$ verifies
  \begin{equation}
    \label{eq:quad_pure_a}
    \frac{a_{q}(t_{0},\omega)}{1-(1-c)\sigma_{2}\varphi_{2}a_{q}(t_{0},\omega)(t-t_{0})} \leq a_{q}(t,\omega) \leq \frac{a_{q}(t_{0},\omega)}{1-(2+c)\sigma_{2}\varphi_{2}a_{q}(t_{0},\omega)(t-t_{0})}.
  \end{equation}
\end{lem}
\begin{proof}
  For any $t_{0} \leq t \leq T_{q}$, 
  \begin{align}
    \partial_{t} a_{q}(t,\omega)^{2} = \frac{2\sigma_{2}\varphi_{2}}{\varepsilon}a_{q}(t,\omega)s_{q}^{2}(t,\omega), \\
    \partial_{t} s_{q}(t,\omega)^{2} = 4\sigma_{2}\varphi_{2}a_{q}(t,\omega)s_{q}^{2}(t,\omega),
  \end{align}
  so that,
  \begin{equation}
    \label{eq:quad_as_link}
    a_{q}(t,\omega)^{2} = \frac{1}{2\varepsilon}s_{q}(t,\omega)^{2}-\frac{1}{2\varepsilon}s_{q}(t_{0},\omega)^{2}+a_{q}(t_{0},\omega)^{2}.
  \end{equation}
  Now, recall that, owing to Theorem \ref{thm:stable_01} and Corollary \ref{cor:stable_01}, there exist constants $C,\delta$ independent of $\varepsilon$ such that 
  \begin{align}
    \vert \overline{a}_{q}(t_{0})-\frac{\varphi_{0}}{\sigma_{0}} \vert \leq C\sqrt{\varepsilon}, \quad \vert \overline{s}_{q}(t_{0}) \vert \leq CK^{\circ}\sqrt{\varepsilon}\log(\frac{1}{\varepsilon}), \\
    \norm{(a_{q,\perp}-b)(t_{0},.)}+\varepsilon^{-1/2}\norm{(s_{q,\perp}-g)(t_{0},.)} \leq C\varepsilon^{\delta}\beta(t_{0}).
  \end{align}
  To compare $s_{q}(t_{0},\omega)$ and $a_{q}(t_{0},\omega)$, we write 
  \begin{align}
    a_{q}(t_{0},\omega) = b(t_{0},\omega)+a_{q,\perp}(t_{0},\omega)-b(t_{0},\omega)+\overline{a_{q}}(t_{0}), \\
    s_{q}(t_{0},\omega) = g(t_{0},\omega)+s_{q,\perp}(t_{0},\omega)-g(t_{0},\omega)+\overline{s_{q}}(t_{0}),
  \end{align}
  so that, allowing $C$ to change from line to line, 
  \begin{align}
    \label{eq:init_q_1}
    \norm{a_{q}(t_{0},.)-\frac{\varphi_{0}}{\sigma_{0}}-\beta(t_{0})a_{init,\perp}(.)} &\leq C\varepsilon^{\delta}\beta(t_{0})+C\sqrt{\varepsilon}, \\
    \norm{s_{q}(t_{0},.)-\gamma(t_{0})a_{init,\perp}(.)} &\leq C\sqrt{\varepsilon}\log(\frac{1}{\varepsilon})+C\varepsilon^{1/2+\delta}\beta(t_{0}).
    \label{eq:init_q_2}
  \end{align}
  Choosing any $K^{\circ} > \frac{1}{4\sigma_{1}\varphi_{1}}$, Proposition \ref{prop:app_bg} then gives that 
  \begin{equation}
    \theta(t_{0}) = \frac{1}{2}\log(\frac{4\varphi_{1}}{\sigma_{1}'\sqrt{\varepsilon}}+o(\frac{1}{\sqrt{\varepsilon}})),
  \end{equation}
  so that there exists a constant $D$ depending only on $\sigma_{1},\varphi_{1},\sigma_{1}'$ such that 
  \begin{equation}
   \vert \beta(t_0) -\frac{1}{2}e^{\theta(t_{0})} \vert \leq D\varepsilon^{1/4}\quad \mbox{and} \quad \vert \gamma(t_0) - \frac{\sqrt{\varepsilon}}{2}e^{\theta(t_{0})} \vert \leq D\varepsilon^{3/4}.
  \end{equation}
  Combining the bounds \eqref{eq:init_q_1}, \eqref{eq:init_q_2} with the above, keeping the leading order in $\varepsilon$ and absorbing the constant $D$ in $C$, we find that 
  \begin{equation}
    \norm{s_{q}(t_{0},.)-\sqrt{\varepsilon}a_{q}(t_{0},.)} \leq C\varepsilon^{3/4},
  \end{equation}
  which, upon recalling Eq.\eqref{eq:quad_as_link}, implies in turn  
  \begin{align}
    \norm{a_{q}(t,.)^{2}-\frac{1}{2\varepsilon}s_{q}(t,.)^{2}-\frac{1}{2}a_{q}(t_{0},.)^{2}} &\leq C \norm{\frac{s_{q}(t_{0},.)}{\sqrt{\varepsilon}}-a_{q}(t_{0},.)}\norm{\frac{s_{q}(t_{0},.)}{\sqrt{\varepsilon}}+a_{q}(t_{0},.)}, \\
    &\leq C.
  \end{align}
  Then, for any $\omega \in \Omega$, we may bound the evolution of $a_{q}(t,\omega)$ as follows
  \begin{equation}
  2\sigma_{2}\varphi_{2}a_{q}^{2}(t,\omega)-\sigma_{2}\varphi_{2}a_{q}(t_{0},\omega)^{2}-C  \leq \partial_{t}a_{q}(t,\omega) \leq 2\sigma_{2}\varphi_{2}a_{q}^{2}(t,\omega)-\sigma_{2}\varphi_{2}a_{q}(t_{0},\omega)^{2}+C,
  \end{equation}
  and, since $a_{q}(t,\omega)$ is increasing for any $\omega$, $a_{q}(t,\omega)^{2}$ is increasing for any $\omega$ in the set 
  \begin{equation}
    \mathcal{A} = \left\{\omega : a_{q}(t_{0},\omega)\geq 0\right\}.
  \end{equation}
We may then write, for any $\omega \in \mathcal{A}$,
  \begin{equation}
  \sigma_{2}\varphi_{2}a_{q}^{2}(t,\omega)-C \leq \partial_{t}a_{q}(t,\omega) \leq 2\sigma_{2}\varphi_{2}a_{q}^{2}(t,\omega)+C.
  \end{equation}
  Thus, for $\varepsilon$ sufficiently small, the upper and lower bounds at \eqref{eq:quad_pure_a} hold for any $\omega$ in the set 
  \begin{equation}
    \mathcal{B} = \left\{\omega : a_{q}(t_{0},\omega)^{2} \geq \frac{C}{c\sigma_{2}\varphi_{2}}\right\}.
  \end{equation}
  The bound \eqref{eq:init_q_1} then ensures that $\mathcal{U}(c) \subset \mathcal{A} \cap \mathcal{B}$, concluding the proof.
\end{proof}
Now, note that, to learn the coefficient $\varphi_{2}$, we require the quantity $\overline{as^{2}}$ to be of order one, which, owing to Eq.\eqref{eq:quad_link_clean}, 
requires $a_{q}$ to be of order $\varepsilon^{-1/3}$ and $s_{q}$ of order $\varepsilon^{1/6}$. The bounds from Eq.\eqref{eq:quad_pure_a} then show that the required time 
for $a_{q}$ (or $s_q$) to make any significant change, e.g. reach a large multiple of its initial condition or the required scale $\varepsilon^{-1/3}$ will be of order $\varepsilon^{1/4}$.
Hence, we would like to control the solution to \eqref{eq:quad_dyn1}-\eqref{eq:quad_dyn2} on a timescale of order $\varepsilon^{1/4}$. The following lemma shows 
that the appropriate idealized system to study is the one introduced in Section \ref{sec:main_results} at Eq.\eqref{eq:main_quad_ideal_dyn1}-\eqref{eq:main_quad_ideal_dyn2}:
  \begin{align}
    \label{eq:quad_ideal_dyn1}
    \partial_{t}z_{1}(t,\omega) &= \sigma_{0}\alpha_{0}(t)+\sigma_{1}\alpha_{1}(t)z_{2}(t,\omega)+\sigma_{2}\varphi_{2}z_{2}^{2}(t,\omega), \\
    \partial_{t}z_{2}(t,\omega) &= z_{1}(t,\omega)\left(\sigma_{1}\alpha_{1}(t)+2\sigma_{2}\varphi_{2}z_{2}(t,\omega)\right),
    \label{eq:quad_ideal_dyn2}
\end{align}
with  
\begin{align}
  \alpha_{0}(t) &= -\frac{1}{\sigma_{0}}\left(\sigma_{1}\alpha_1(t)\overline{z_{2}}(t)+\sigma_{2}\varphi_{2}\overline{z_{2}^{2}}(t)\right), \\
  \alpha_{1}(t) &= -\frac{2\sigma_{2}\varphi_{2}\overline{z_{1}^{2}z_{2}}+\sigma_{2}\varphi_{2}\overline{z_{2}^{3}}-\sigma_{2}\varphi_{2}\overline{z_{2}^{2}}\overline{z_{2}}}{\sigma_{1}\overline{z_{1}^{2}}+\sigma_{1}(\overline{z_{2}^{2}}-\overline{z_{2}}^{2})}. 
\end{align}
We also recall that, by definition of $\alpha_{0},\alpha_{1}$, any solution to the above verifies 
\begin{equation}
  \partial_{t} \overline{z_{1}}(t) = 0, \quad \mbox{and} \quad \partial_{t} \overline{z_{1}z_{2}}(t) = 0.
\end{equation}
To obtain a sharper control on the error incurred by the application of Proposition \ref{prop:approx_dyn_set}, we recall the following change in time and scale 
introduced in Section \ref{sec:main_results}:
 \begin{align}
    \label{eq:change_scale_a}
    \hat{a}(t,\omega) = \varepsilon^{1/4}a(\varepsilon^{1/4} t,\omega), \\
    \hat{s}(t,\omega) = \varepsilon^{-1/4}s(\varepsilon^{1/4} t,\omega),
    \label{eq:change_scale_s}
  \end{align}
which leads us to study the evolution of $\hat{a},\hat{s}$ on a timescale of order one.
\begin{lem}
  \label{lem:quad_approx}
  Let $(a(t,.),s(t,.))$ denote the solution to \eqref{eq:quad_dyn1}-\eqref{eq:quad_dyn2} with any initial condition $a(t_{0},.)$, \\$s(t_{0},.)$ verifying the bounds \eqref{eq:init_quad1}-\eqref{eq:init_quad2}. We 
  then have the following results.
  \begin{itemize}
  \item There exist continuous functions $z_{1}(t_{0},.),z_{2}(t_{0},.)$ and constants $C_{1},C_{2}$ such that
  \begin{align}
    &\overline{z_{1}}(t_{0}) = \varepsilon^{1/4}\frac{\varphi_{0}}{\sigma_{0}}, \quad \overline{z_{2}}(t_{0}) = 0, \quad \overline{z_{1}z_{2}}(t_{0}) = \frac{\varphi_{1}}{\sigma_{1}}, \\
    &\norm{\hat{a}(t_{0},.)-z_{1}(t_{0},.)} \leq C_{1}\varepsilon^{3/4}, \quad \norm{\hat{s}(t_{0},.)-z_{2}(t_{0},.)} \leq C_{2}\varepsilon^{1/4}\log(\frac{1}{\varepsilon}).
  \end{align}
  Furthermore, there exists a constant $D$ depending only on the problem data such that 
  \begin{equation}
    (\norm{\hat{a}(t_{0},.)}+\norm{\hat{s}(t_{0},.)} ) \vee(\norm{z_{1}(t_{0},.)}+\norm{z_{2}(t_{0},.)}) \leq D.
  \end{equation}
  \item Let $(z_{1}(t,.),z_{2}(t,.))$ be the solution to \eqref{eq:quad_ideal_dyn1}-\eqref{eq:quad_ideal_dyn2} with initial condition $(z_{1}(t_{0},.),z_{2}(t_{0},.))$.
  For any constant $C>D$, consider the hitting time 
  \begin{equation}
    T_{1} = \inf\{t \geq t_{0} : (\norm{\hat{a}(t,.)}+\norm{\hat{s}(t,.)} ) \vee(\norm{z_{1}(t,.)}+\norm{z_{2}(t,.)}) \geq C\}.
  \end{equation}
Then there exist constants $C_{5},C_{6}$ depending only on the problem data and $C$ such that, for any fixed time horizon $T$ independent of $\varepsilon$ and any $t_{0} \leq t \leq T_{1} \wedge T$,
\begin{equation}
  \norm{z_{\varepsilon}(t,.)-z(t,.)} \leq \left(\norm{z_{\varepsilon}(t_{0},.)-z(t_{0},.)}+C_{5}\varepsilon^{1/8}(t-t_{0})\right)\exp(C_{6}(t-t_{0})).
\end{equation}
\end{itemize}
\end{lem}
\begin{proof}
  We begin by constructing an appropriate initial condition $z(t_{0},.) = (z_{1}(t_{0},.),z_{2}(t_{0},.))$. Recalling that 
  $\hat{a}(t_{0},.) = \varepsilon^{1/4}a(t_{0},.)$ and $\hat{s}(t_{0},.) = \varepsilon^{-1/4}s(t_{0},.)$, we have
\begin{align}
  \vert \overline{\hat{a}}(t_{0})-\varepsilon^{1/4}\frac{\varphi_{0}}{\sigma_{0}} \vert \leq C_{2}\varepsilon^{3/4}, \\
  \vert \overline{\hat{a}\hat{s}}(t_{0})-\frac{\varphi_{1}}{\sigma_{1}} \vert \leq C_{2}\varepsilon^{1/4}.
\end{align}
Then, we choose 
\begin{align}
  z_{1}(t_{0},.) = \hat{a}_{\perp}(t_{0},.)+\varepsilon^{1/4}\frac{\varphi_{0}}{\sigma_{0}}, \\
  z_{2}(t_{0},.) = \frac{\varphi_{1}}{\sigma_{1}\overline{\hat{a}_{\perp}\hat{s}_{\perp}}(t_{0})}\hat{s}_{\perp}(0,.),
\end{align}
so that 
\begin{align}
  \overline{z_{1}}(t_{0}) &= \varepsilon^{1/4}\frac{\varphi_{0}}{\sigma_{0}}, \\
  \overline{z_{1}z_{2}}(t_{0}) &= \frac{\varphi_{1}}{\sigma_{1}},
\end{align}
and 
\begin{align}
  \norm{z_{1}(t_{0},.)-\hat{a}(t_{0},.)} &\leq C_{2}\varepsilon^{3/4}.
\end{align}
Also,
\begin{align}
  \norm{z_{2}(t_{0},.)-\hat{s}(t_{0},.)} \leq \frac{\vert \varphi_{1}-\sigma_{1}\overline{\hat{a}\hat{s}}(t_{0})\vert}{\sigma_{1}\overline{\hat{a}\hat{s}}(t_{0})}\norm{\hat{s}_{\perp}(t_{0},.)}+ \vert \overline{\hat{s}}(t_{0}) \vert,
\end{align}
which, upon recalling that there exists $\delta>0$ independent of $\varepsilon$ such that,
\begin{align}
  \norm{\hat{s}_{\perp}(t_{0},.)-\left(\frac{\varphi_{1}}{\sigma_{1}'}\right)^{1/2}a_{init,\perp}(.)} \leq C\varepsilon^{\delta} \quad \mbox{and} \quad \vert \overline{\hat{s}}(t_{0}) \vert \leq C\varepsilon^{1/4}\log(\frac{1}{\varepsilon}),
\end{align}
implies that there exists a constant $C'$ such that 
\begin{equation}
  \norm{z_{2}(t_{0},.)-\hat{s}(t_{0},.)} \leq C'\varepsilon^{1/4}\log(\frac{1}{\varepsilon}).
\end{equation}
  Now, note that, owing to Theorem \ref{thm:stable_01}, there exists a constant $D$ depending only on $\sigma_{1},\varphi_{1},\sigma_{1}',M$ such that, for $\varepsilon$ sufficiently small,
  \begin{equation}
    \norm{a(t_{0},.)}+\varepsilon^{-1/2}\norm{s(t_{0},.)} \leq D\varepsilon^{-1/4}.
  \end{equation}
  Let $T_{1}$ be the hitting time 
  \begin{equation}
    T_{1} = \inf\{t \geq t_{0} : \norm{a(t,.)}+\varepsilon^{-1/2}\norm{s(t,.)} \geq C_{1}\varepsilon^{-1/4}\},
  \end{equation}
  for any constant $C_{1} >D$ (that does not depend on $\varepsilon$), and let $T$ be a time horizon independent of $\varepsilon$.
  Then, for any $t_{0} \leq t \leq T\varepsilon^{1/4} \wedge T_{1}$ consider the following change in time and scale in the dynamical system \eqref{eq:quad_dyn1}-\eqref{eq:quad_dyn2}:
  \begin{align}
    \hat{a}(t,\omega) = \varepsilon^{1/4}a(\varepsilon^{1/4} t,\omega), \\
    \hat{s}(t,\omega) = \varepsilon^{-1/4}s(\varepsilon^{1/4} t,\omega).
  \end{align}
  Thus, we are now interested in controlling the evolution of the pair $\hat{a},\hat{s}$ on a timescale of order one, i.e. independent of $\varepsilon$.
  The evolution \eqref{eq:quad_dyn1}-\eqref{eq:quad_dyn2} then becomes 
  \begin{align}
    \partial_{t} \hat{a}(t,\omega) &= \frac{\sigma_{0}}{\sqrt{\varepsilon}}\left(\varphi_{0}-\sigma_{0}\overline{a}(t)\right)+\frac{\sigma_{1}}{\varepsilon^{1/4}}\hat{s}(t,\omega)\left(\varphi_{1}-\sigma_{1}\overline{as}(t)\right)+\sigma_{2}\varphi_{2}\hat{s}^{2}(t,\omega), \\
    \partial_{t} \hat{s}(t,\omega) &= \hat{a}(t,\omega)\frac{\sigma_{1}}{\varepsilon^{1/4}}\left(\varphi_{1}-\sigma_{1}\overline{as}(t)\right)+2\sigma_{2}\varphi_{2}\hat{a}(t,\omega)\hat{s}(t,\omega),
  \end{align}
  and we may rewrite $T_{1}$ as 
  \begin{equation}
    T_{1} = \inf\{t \geq t_{0} : \norm{\hat{a}(t,.)}+\norm{\hat{s}(t,.)} \geq C_{1}\}.
  \end{equation}
Now, an argument similar to the one employed to prove the upper bound Eq.\eqref{eq:stable_phi1} in the proof of Theorem \ref{thm:stable_01} and Corollary \ref{cor:stable_01} shows that 
there exist constants $C_{3},C_{4}$ depending only on the problem data and $C_{1}$ such that, for any $t_{0} \leq t \leq T_{1}$,
\begin{align}
  \vert \varphi_{0}-\sigma_{0}\overline{a}(t) \vert \leq C_{3}\sqrt{\varepsilon}, \quad \mbox{and} \quad \vert \varphi_{1}-\sigma_{1}\overline{as}(t) \vert \leq C_{4}\varepsilon^{1/4}.
\end{align}
We now note that, by construction,  
  \begin{equation}
    \norm{z_{1}(t_{0},.)}+ \norm{z_{2}(t_{0},.)} \leq D,
  \end{equation}
and introduce the hitting time
\begin{equation}
    T_{2} = \inf\{t \geq t_{0} : \norm{z_{1}(t,.)}+\norm{z_{2}(t,.)} \geq C_{1}\}.
\end{equation}
Letting 
\begin{align}
z(t,\omega) &= (z_{1}(t,\omega),z_{2}(t,\omega)), \\
z_{\varepsilon}(t,\omega) &= (\hat{a}(t,\omega),\hat{s}(t,\omega)),
\end{align}
we apply Corollary \ref{cor:nn_approx_dyn} to obtain that there exist constants $C_{5},C_{6}$ independent of $\varepsilon$ such that, for any $t_{0} \leq t \leq T_{1} \wedge T_{2} \wedge T$,
\begin{equation}
  \norm{\norm{z_{\varepsilon}(t,.)-z(t,.)}_{2}} \leq \left(\norm{\norm{z_{\varepsilon}(t_{0},.)-z(t_{0},.)}_{2}}+C_{5}\varepsilon^{1/8}(t-t_{0})\right)\exp(C_{6}(t-t_{0})),
\end{equation}
concluding the proof. Proposition \ref{prop:quad_approx} is an immediate consequence of Lemma \ref{lem:quad_approx}.
\end{proof}
\appendix
\section{Auxiliary results on idealized systems}
\label{sec:app_A}
\subsection{Idealized systems for the constant and linear components}
Consider the dynamical system 
\begin{align}
  \partial_{t}b(t,\omega) &= \frac{\sigma_{1}}{\varepsilon}g(t,\omega)\left(\varphi_{1}-\sigma_{1}\overline{bg}(t)\right), \\
  \partial_{t}g(t,\omega) &= \sigma_{1}b(t,\omega)\left(\varphi_{1}-\sigma_{1}\overline{bg}(t)\right),
\end{align}
with initial condition $b(0,.) = a_{init,\perp}(.)$ and $g(0,.) = 0$. For these dynamics, we have that $b$ and $g$ will remain proportional 
to $a_{init,\perp}(.)$ for all times, so that $b(t,.) = \beta(t)a_{init,\perp}(.)$ and $g(t,.) = \gamma(t)a_{init,\perp}(.)$, where $\beta(t),\gamma(t)$ are the solutions to the 
dynamical system 
\begin{align}
  \label{eq:beta_gamma1}
  \partial_{t}\beta(t) &= \frac{\sigma_{1}}{\varepsilon}\gamma(t)\left(\varphi_{1}-\sigma'_{1}\beta(t)\gamma(t)\right), \\
  \partial_{t}\gamma(t) &= \sigma_{1}\beta(t)\left(\varphi_{1}-\sigma'_{1}\beta(t)\gamma(t)\right),
  \label{eq:beta_gamma2}
\end{align}
where we have introduced $\sigma'_{1} = \sigma_{1}\int_{0}^{1}a^{2}_{init,\perp}(\nu)d\nu$, and with $\beta(0) = 1$ and $\gamma(0) = 0$. The solution 
to the system \eqref{eq:beta_gamma1}-\eqref{eq:beta_gamma2} is characterized in the following proposition. 
\begin{prop}
  \label{prop:app_bg}
For every $t \geq 0$, we have the representation 
\begin{equation}
  \beta(t) = \cosh \theta(t), \quad \gamma(t) = \sqrt{\varepsilon}\sinh \theta(t),
\end{equation}
where $\theta(t)$ is the solution to the ODE 
\begin{equation}
  \label{eq:theta_ode}
  \partial_{t}\theta = \frac{\sigma_{1}\varphi_{1}}{\sqrt{\varepsilon}}-\frac{\sigma_{1}\sigma'_{1}}{2}\sinh 2\theta(t),
\end{equation}
with $\theta(0) =0$.
Furthermore, for every $t\geq 0$, we have the upper and lower bounds
\begin{equation}
l(t) \leq \theta(t) \leq u(t),
\end{equation}
where 
\begin{align}
  l(t) &= \frac{\sigma_{1}\varphi_{1}}{\sqrt{\varepsilon}}t-\frac{1}{2}\log\left(1+\frac{\sigma'_{1}\sqrt{\varepsilon}}{4\varphi_{1}}\left(e^{\frac{2\sigma_{1}\varphi_{1}}{\sqrt{\varepsilon}}t}-1\right)\right), \\
  u(t) &= \left(\frac{\sigma_{1}\varphi_{1}}{\sqrt{\varepsilon}}+\frac{\sigma_{1}\sigma'_{1}}{4}\right)t-\frac{1}{2}\log\left(1+\frac{\sigma'_{1}\sqrt{\varepsilon}}{4\varphi_{1}}\left(\frac{1}{1+\frac{\sigma'_{1}\sqrt{\varepsilon}}{4\varphi_{1}}}\right)\left(e^{\left(\frac{2\sigma_{1}\varphi_{1}}{\sqrt{\varepsilon}}+\frac{\sigma_{1}\sigma'_{1}}{2}\right)t}-1\right)\right).
\end{align}
\end{prop}
\begin{proof}
  A direct calculation using \eqref{eq:beta_gamma1}-\eqref{eq:beta_gamma2} shows that 
  \begin{align}
    \partial_{t}\big(\beta^{2}(t)-\varepsilon^{-1} \gamma^{2}(t)\big) &= 2\beta(t)\partial_{t}\beta(t)-2\varepsilon^{-1}\gamma(t)\partial_{t}\gamma(t) \\
    &=0,
  \end{align}
  and, since $\beta(0) = 1$ and $\gamma(0) = 0$, we find that $\beta^{2}-\varepsilon^{-1}\gamma^{2} = 1$ for all times. This allows us 
  to reparametrize the solution using a hyperbolic angle $\theta(t)$ such that $\beta(t) = \cosh \theta(t)$ and $\gamma(t) = \sqrt{\varepsilon}\sinh \theta(t)$.
  Substituting this into \eqref{eq:beta_gamma1}, we find 
  \begin{equation}
    \left(\sinh \theta(t)\right)\partial_{t}\theta = \varepsilon^{-1}\sigma_{1}(\sqrt{\varepsilon}\sinh \theta(t))\left(\varphi_{1}-\sigma'_{1} \cosh \theta(t)\sqrt{\varepsilon}\sinh \theta(t)\right).
  \end{equation}
  Dividing by $\sinh(\theta)$ (for $t>0$) and using the identity $2\sinh \theta \cosh \theta = \sinh 2\theta$, we obtain the scalar ODE \eqref{eq:theta_ode}. \newline 
  \indent 
  Now, the latter implies that, for all times, 
  \begin{equation}
  \frac{\sigma_{1}\varphi_{1}}{\sqrt{\varepsilon}}-\frac{\sigma_{1}\sigma'_{1}}{4}e^{2\theta(t)}  \leq \partial_{t} \theta(t) \leq \frac{\sigma_{1}\varphi_{1}}{\sqrt{\varepsilon}}-\frac{\sigma_{1}\sigma'_{1}}{4}\left(e^{2\theta(t)}-1\right),
  \end{equation}
  so that 
  \begin{equation}
  l(t) \leq \theta(t) \leq u(t),
  \end{equation}
  where $l,u$ are respectively the solutions of 
  \begin{align}
    \partial_{t} l(t) &= \frac{\sigma_{1}\varphi_{1}}{\sqrt{\varepsilon}}-\frac{\sigma_{1}\sigma'_{1}}{4}e^{2l(t)}, \quad l(0) = 0, \\
    \partial_{t} u(t) &= \frac{\sigma_{1}\varphi_{1}}{\sqrt{\varepsilon}}-\frac{\sigma_{1}\sigma'_{1}}{4}\left(e^{2u(t)}-1\right), \quad u(0) = 0.
  \end{align}
  Focusing on $l(t)$, let us introduce $\tilde{l}(t) = l(t)-\frac{\sigma_{1}\varphi_{1}}{\sqrt{\varepsilon}}t$, so that $\tilde{l}(t)$ solves
  \begin{equation}
    \partial_{t}\tilde{l}(t) = -\frac{\sigma_{1}\sigma'_{1}}{4}e^{2\tilde{l}(t)}e^{\frac{2\sigma_{1}\varphi_{1}}{\sqrt{\varepsilon}}t}, \quad \tilde{l}(0) = 0.
  \end{equation}
  To solve this ODE, we define
  \begin{align}
    F(x) &= -\int_{0}^{x}e^{-2y}dy \\
          &= \frac{1}{2}\left(e^{-2x}-1\right), \\
    G(t) &= \int_{0}^{t}\frac{\sigma_{1}\sigma'_{1}}{4}e^{\frac{2\sigma_{1}\varphi_{1}}{\sqrt{\varepsilon}}s} ds \\
    &= \frac{\sigma'_{1}\sqrt{\varepsilon}}{8\varphi_{1}}\left(e^{\frac{2\sigma_{1}\varphi_{1}}{\sqrt{\varepsilon}}t}-1\right),
  \end{align}
so that 
\begin{align}
  \tilde{l}(t) = F^{-1}(G(t)) = -\frac{1}{2}\log\left(1+\frac{\sigma'_{1}\sqrt{\varepsilon}}{4\varphi_{1}}\left(e^{\frac{2\sigma_{1}\varphi_{1}}{\sqrt{\varepsilon}}t}-1\right)\right),
\end{align}
yielding the expression for $l(t)$.
A similar procedure applied to $\tilde{u}(t) = u(t)-\left(\frac{\sigma_{1}\varphi_{1}}{\sqrt{\varepsilon}}+\frac{\sigma_{1}\sigma'_{1}}{4}\right)t$ gives the expression for $u(t)$.
\end{proof}
\begin{cor}
  \label{cor:bet_gam_bounds}
  For every $t \geq 0$, we have 
  \begin{equation}
    \vert \beta(t)\gamma(t) \vert \leq \frac{\sqrt{\varepsilon}}{4}\exp\left(\frac{2\sigma_{1}\varphi_{1}t}{\sqrt{\varepsilon}}\right),
  \end{equation}
  as well as 
  \begin{equation}
    0 \leq \varphi_{1}-\sigma'_{1}\beta(t)\gamma(t) \leq \varphi_{1}\left(\frac{4\varphi_{1}-\sigma_{1}'\sqrt{\varepsilon}}{4\varphi_{1}+\sigma_{1}'\sqrt{\varepsilon}\left(\exp\left(\frac{2\sigma_{1}\varphi_{1}t}{\sqrt{\varepsilon}}\right)-1\right)}\right)+\frac{\sigma_{1}'\sqrt{\varepsilon}}{4}\left(\frac{\sigma_{1}'\sqrt{\varepsilon}}{4\varphi_{1}}+\exp\left(-\frac{2\sigma_{1}\varphi_{1}t}{\sqrt{\varepsilon}}\right)\right).
  \end{equation}
\end{cor}
\begin{proof}
The first result is immediate from Eq.\eqref{eq:theta_ode}. For the second one, we write 
\begin{align}
    \varphi_{1}-\sigma'_{1}\beta(t)\gamma(t) &\leq \varphi_{1}-\sigma'_{1}\sqrt{\varepsilon}\cosh(l(t))\sinh(l(t)) \\
    &= \varphi_{1}-\frac{\sigma'_{1}\sqrt{\varepsilon}}{2}\sinh(2l(t)) \\
    &\leq \varphi_{1}-\frac{\sigma_{1}'\sqrt{\varepsilon}}{4}\left(\frac{\exp(\frac{2\sigma_{1}\varphi_{1}t}{\sqrt{\varepsilon}})}{1+\frac{\sigma_{1}'\sqrt{\varepsilon}}{4\varphi_{1}}\left(\exp(\frac{2\sigma_{1}\varphi_{1}t}{\sqrt{\varepsilon}})-1\right)}-\frac{1+\frac{\sigma_{1}'\sqrt{\varepsilon}}{4\varphi_{1}}\left(\exp(\frac{2\sigma_{1}\varphi_{1}t}{\sqrt{\varepsilon}})-1\right)}{\exp(\frac{2\sigma_{1}\varphi_{1}t}{\sqrt{\varepsilon}})}\right), \\
\end{align}
implying the desired.
\end{proof}
\section{An approximation result for flows near target sets}
\label{sec:app_B}
Let $\mathcal{O}$ be an open set in $\mathbb{R}^{d}$, $m$ be a fixed positive integer, $\{h_{i}\}_{i=1}^{m}$ be a collection of twice continuously differentiable functions on $\mathcal{O}$ valued in $\mathbb{R}_{+}$, and $\{g_{i}\}_{i=1}^{m},f$ be a collection of continuously differentiable vector fields on $\mathcal{O}$.
In what follows, we let $\mathcal{S}$ denote the set
\begin{equation}
  \mathcal{S} = \left\{ z \in \mathcal{C}(\Omega, \mathbb{R}^{d}) : \thickspace \forall \thickspace 1 \leq i \leq m, \int_{\Omega}h_{i}(z(\omega))d\rho(\omega) = 0 \right\}.
\end{equation}
For each $1 \leq i \leq m$, let $v_{i} : \mathbb{R}_{+} \to \mathbb{R}_{+}$ be a continuous, non-decreasing, positive function such that $v_{i}(0) = 0$.
For any $\varepsilon>0$, consider the (infinite dimensional) system of ordinary differential equations 
\begin{align}
  \label{eq:model_dyn}
  \partial_{t}{z}_{\varepsilon}(t,\omega) &= \sum_{i=1}^{m}\frac{g_{i}(z_{\varepsilon}(t,\omega))}{v_{i}(\varepsilon)}\int_{\Omega}h_{i}(z_{\varepsilon}(t,\nu))d\rho(\nu)+f(z_{\varepsilon}(t,\omega)), \\
  z_{\varepsilon}(0,\omega) &= z_{\varepsilon,0}(\omega) \in \mathcal{C}_{c}(\Omega,\mathbb{R}^{d}).
  \label{eq:model_dyn_init}
\end{align}
Since, for any $\omega \in \Omega$, the right-hand-side of \eqref{eq:model_dyn} is locally Lipschitz on $\mathcal{O}$, the Cauchy-Lipschitz theorem guarantees the 
existence and uniqueness of a local solution $z_{\varepsilon}(t,\omega) \in \mathcal{C}^{1}([0,T_{0}(\omega)],\mathbb{R}^{d})$ to the initial value problem \eqref{eq:model_dyn}-\eqref{eq:model_dyn_init},
for some time horizon $T_{0}(\omega)$. In what follows, we will denote $T_{0} = \inf_{\omega \in \Omega} T_{0}(\omega)$. For any $u \in \mathcal{C}(\Omega,\mathbb{R}^{d})$, 
we let $\boldsymbol{G}(u) \in \mathbb{R}^{m \times m}$ and $\boldsymbol{F}(u) \in \mathbb{R}^{m}$ respectively denote the matrix and vector whose entries are given by
\begin{align}
  &\left[\boldsymbol{G}_{ij}(u)\right]_{1 \leq i,j \leq m} = \int_{\Omega}\langle g_{j}(u(\omega)),\nabla h_{i}(u(\omega)) \rangle d\rho(\omega), \\
  &[\boldsymbol{F}_{i}(u)]_{1 \leq i \leq m} = \int_{\Omega} \langle f(u(\omega)),\nabla h_{i}(u(\omega)) \rangle d\rho(\omega).
\end{align}
Assume that there exists a non-empty subset $\mathcal{U}$ of $\mathcal{C}(\Omega,\mathbb{R}^{d})$ such that, for any $u \in \mathcal{U}$, the 
matrix $\boldsymbol{G}(u)$ is full-rank. Finally, for any $u \in \mathcal{U}$, we denote $\boldsymbol{\alpha}(u)$ the solution to the linear system
\begin{equation}
  \boldsymbol{G}(u)\boldsymbol{\alpha}(u)+\boldsymbol{F}(u) = 0.
\end{equation}
With these definitions in hand, consider the auxiliary dynamical system
\begin{align}
  \label{eq:model_dyn_app}
  \partial_{t}z(t,\omega) &= \sum_{i=1}^{m}\alpha_{i}(t)g_{i}(z(t,\omega))+f(z(t,\omega)), \\
  z(0,\omega) &= z_{0}(\omega) \in \mathcal{C}_{c}(\Omega,\mathbb{R}^{d}) \cap \mathcal{U},
  \label{eq:model_dyn_app_init}
\end{align}
where we have respectively abbreviated $\boldsymbol{G}(z(t,.)),\boldsymbol{\alpha}(z(t,.)),\boldsymbol{F}(z(t,.))$ with $\boldsymbol{G}(t),\boldsymbol{\alpha}(t),\boldsymbol{F}(t)$.
Letting $T_{\mathcal{U}} = \inf \{t \geq 0 : z(t,.) \notin \mathcal{U}\}$, the Cauchy-Lipschitz theorem once again guarantees the existence, for each $\omega$, of 
a time horizon $\overline{T}_{0}(\omega)$ such that there exists a unique local solution $z(t,\omega) \in \mathcal{C}^{1}([0,\overline{T}_{0}(\omega)\wedge T_{\mathcal{U}}],\mathbb{R}^{d})$.
In what follows, we will denote $\overline{T}_{0} = \inf_{\omega \in \Omega} \overline{T}_{0}(\omega)$.
A short computation then shows that the coefficients $\alpha_{i}(t)$ impose the constraints 
\begin{align}
  z(t,.) &\in S, \\
  \partial_{t}z(t,\omega) &\in f(z(t,\omega))+\mbox{Vec}\left(\left\{g_{i}(z(t,\omega))\right\}_{i=1}^{m}\right), \quad \forall \omega \in \Omega.
\end{align}
We seek to bound the $\ell_{2}$ distance $\norm{z_{\varepsilon}(t,\omega)-z(t,\omega)}_{2}$, for each $\omega \in \Omega$, on a time horizon 
$T_{\mathcal{S}}$ such that $z_{\varepsilon}(t,.)$ is close to the set $\mathcal{S}$ in the following sense: for each $1 \leq i \leq m$, we assume that there exists a continuous, non-decreasing, positive function $u_{i}$ such that $u_{i}(0) = 0$ and, 
for any $t \leq T_{\mathcal{S}}$:
\begin{equation}
  \int_{\Omega}h_{i}(z_{\varepsilon}(t,\omega))d\rho(\omega) \leq u_{i}(\varepsilon).
\end{equation}
For any $z_{0} \in \mathbb{R}^{d}$ and positive constant $R$, we denote 
\begin{align}
  C_{f}(z_{0},R) &= \sup_{z \in \mathcal{B}(z_{0},R)} \norm{f(z)}_{2}, \\
  C'_{f}(z_{0},R) &= \sup_{z \in \mathcal{B}(z_{0},R)} \norm{\nabla f(z)}_{op},
\end{align}
along with the corresponding quantities for the $\{g_{i}\}_{1 \leq i \leq nm}$ and $h$. We also let, for any $1 \leq i \leq m$,
\begin{equation}
  C''_{h,i}(z_{0},R) = \sup_{z \in \mathcal{B}(z_{0},R)} \norm{\nabla^{2}h_{i}(z)}_{op},
\end{equation}
and introduce the stopping times, for any positive constant $a$ and any $\omega \in \Omega$
\begin{align}
  T_{a} &= \inf\{0 \leq t \leq T_{0} \wedge \overline{T}_{0} \wedge T_{\mathcal{U}} : \sigma_{min}\left(G(z_{\varepsilon}(t,.))\right) \wedge \sigma_{min}\left(G(z_{\varepsilon}(t,.))\right)\leq a\}, \\
  T_{z_{0},R}(\omega) &= \inf\{0 \leq t \leq T_{0} \wedge \overline{T}_{0} \wedge T_{\mathcal{U}}  : z_{\varepsilon}(t,\omega) \notin \mathcal{B}(z_{0}(\omega),R) \thickspace \mbox{or} \thickspace z(t,\omega) \notin \mathcal{B}(z_{0}(\omega),R)\},
\end{align}
and we will denote $T_{z_{0},R} = \inf_{\omega \in \Omega}T_{z_{0},R}(\omega)$.
Using the above notations, for any $\delta >0$, consider the following quantities 
\begin{align}
  L_{z_{\varepsilon}} &= \sum_{i=1}^{m}\frac{u_{i}(\varepsilon)}{v_{i}(\varepsilon)}C_{g_{i}}(z_{0},R)+C_{f}(z_{0},R), \\
  C_{\alpha} &= \frac{C_{f}(z_{0},R)}{a}\left(\sum_{i=1}^{m}C'_{h_{i}}(z_{0},R)^{2}\right)^{1/2},\\
  L_{z} &=  C_{\alpha}\left(\sum_{i=1}^{m}C_{g_{i}}(z_{0},R)^{2}\right)^{1/2}+C_{f}(z_{0},R), \\
  C'_{\alpha} &= \left(\frac{1}{a}\left(\sum_{i=1}^{m}\left(C'_{f}C'_{h_{i}}+C_{f}C''_{h_{i}}\right)^{2}\right)^{1/2}+\frac{m}{a}\sup_{1 \leq i,j \leq n}\left(C_{g_{i}}C''_{h_{j}}+C'_{g_{i}}C'_{h_{j}}\right)C_{\alpha}\right), \\
  K_{v_{z}} &= \frac{\delta^{2}}{\sqrt{3}}L_{z_{\varepsilon}}(z_{0},R)\left(\sum_{i=1}^{m}\frac{u_{i}(\varepsilon)}{v_{i}(\varepsilon)}C'_{g_{i}}(z_{0},R)+C'_{f}(z_{0},R)\right), \\
  K_{v_{z_{\varepsilon}}} &= \frac{\delta^{2}}{\sqrt{3}}C'_{f}(z_{0},R)L_{z}(z_{0},R)+\frac{\delta^{2}\sqrt{2m}}{\sqrt{3}}\left(C_{\alpha}^{2}L_{z}^{2}\left(\sup_{1 \leq i \leq m}C'_{g_{i}}\right)^{2}+\left(\sup_{1 \leq i \leq m}C_{g_{i}}\right)^{2}(C'_{\alpha}L_{z})^{2}\right)^{1/2},
\end{align}
as well as 
\begin{align}
  K &= \left(C'_{f}+C_{\alpha}\left(\sum_{i=1}^{m}(C'_{g_{i}})^{2}\right)^{1/2}+C'_{\alpha}\left(\sum_{i=1}^{m}C_{g_{i}}^{2}\right)^{1/2}\right), \\
  \ell(\delta,\varepsilon) &= K_{v_{z}}+K_{v_{z_{\varepsilon}}}+\frac{\left(\sum_{i=1}^{m}C_{g_{i}}^{2}\right)^{1/2}}{a}\left(\left(\sum_{i=1}^{m}\left(2u_{i}(\varepsilon)+C''_{h_{i}}L^{2}_{z}\delta^{2}\right)^{2}\right)^{1/2}+K_{v_{z_{\varepsilon}}}\left(\sum_{i}^{n}(C'_{h_{i}}(z_{0},R))^{2}\right)^{1/2}\right).
\end{align}
We then have the following result:
\begin{prop}
  \label{prop:approx_dyn_set}
  Using the notations and under the assumptions outlined above, fix $a$ and $R$ such that $T_{a} \wedge T_{z_{0},R}>0$,
  and denote $S = T_{\mathcal{S}}\wedge T_{a}\wedge T_{z_{0},R} \wedge T_{0} \wedge \overline{T}_{0} \wedge T_{\mathcal{U}}$. 
  Then, for any $\delta >0$ and $n \in \mathbb{N}$ such that $n\delta \leq S$,
  \begin{equation}
  \norm{\norm{z_{\varepsilon}(n\delta,.)-z(n\delta,.)}_{2}} \leq \left(\norm{\norm{z_{\varepsilon}(0,.)-z(0,.)}_{2}}+n\ell(\delta,\varepsilon)\right)\exp\left(Kn\delta\right).
\end{equation}
\end{prop}
\begin{proof}
  We begin by writing, for any $t <S$ and $\delta >0$ such that $t+\delta \leq S$,
  \begin{align}
    \label{eq:dev_z_e}
    z_{\varepsilon}(t+\delta,\omega) &= z_{\varepsilon}(t,\omega)+\int_{0}^{\delta}f(z_{\varepsilon}(t+s,\omega))ds+\sum_{i=1}^{m}\int_{0}^{\delta}g_{i}(z_{\varepsilon}(t+s,\omega))\int_{\Omega}\frac{h_{i}(z_{\varepsilon}(t+s,\nu))}{v_{i}(\varepsilon)}d\rho(\nu)ds, \\
    z(t+\delta,\omega) &= z(t,\omega)+\int_{0}^{\delta}f(z(t+s,\omega))ds+\sum_{i=1}^{m}\int_{0}^{\delta}\alpha_{i}(t+s)g_{i}(z(t+s,\omega))ds.
    \label{eq:dev_z}
  \end{align}
Let us introduce the vectors $v_{z_{\varepsilon}},v_{z},$ defined as 
\begin{align}
  \label{eq:def_z_e}
  v_{z_{\varepsilon}}(t,\omega) &= z_{\varepsilon}(t+\delta,\omega)-z_{\varepsilon}(t,\omega)-\delta f(z_{\varepsilon}(t,\omega))-\sum_{i=1}^{m}g_{i}(z_{\varepsilon}(t,\omega))\int_{0}^{\delta}\int_{\Omega}\frac{h_{i}(z_{\varepsilon}(t+s,\nu))}{v_{i}(\varepsilon)}d\rho(\nu)ds, \\
  v_{z}(t,\omega) &= z(t+\delta,\omega)-z(t,\omega)-\delta f(z(t,\omega))-\delta\sum_{i=1}^{m}\alpha_{i}(t)g_{i}(z(t,\omega)),
  \label{eq:def_z}
\end{align}
and turn to bounding their norms. 
\newline
\paragraph{\textbf{Bounding the norms of $v_{z_{\varepsilon}},v_{z}$}}
Combining Eq.\eqref{eq:dev_z_e} with Eq.\eqref{eq:def_z_e}, and Eq.\eqref{eq:dev_z} with Eq.\eqref{eq:def_z}, we find 
\begin{align}
  \label{eq:exp_vz_e}
  v_{z_{\varepsilon}}(t,\omega) &= \int_{0}^{\delta}f(z_{\varepsilon}(t+s,\omega))ds-\delta f(z_{\varepsilon}(t,\omega)) \notag \\ 
  &\hspace{-1cm}+\sum_{i=1}^{m}\left(\int_{0}^{\delta}g_{i}(z_{\varepsilon}(t+s,\omega))\int_{\Omega}\frac{h_{i}(z_{\varepsilon}(t+s,\nu))}{v_{i}(\varepsilon)}d\rho(\nu)ds-g_{i}(z_{\varepsilon}(t,\omega))\int_{0}^{\delta}\int_{\Omega}\frac{h_{i}(z_{\varepsilon}(t+s,\nu))}{v_{i}(\varepsilon)}d\rho(\nu)ds\right)
\end{align}
and 
\begin{align}
  \label{eq:exp_vz}
  v_{z}(t,\omega) &= \int_{0}^{\delta}f(z(t+s),\omega)ds-\delta f(z(t,\omega))+\sum_{i=1}^{m}\left(\int_{0}^{\delta}\alpha_{i}(t+s)g_{i}(z(t+s,\omega))ds-\delta\alpha_{i}(t)g_{i}(z(t,\omega))\right).
\end{align}
Starting with $v_{z_{\varepsilon}}$, a first order Taylor expansion then gives 
\begin{equation}
  f(z_{\varepsilon}(t+s),\omega) = f(z_{\varepsilon}(t,\omega))+sDf[z(\tilde{t},\omega)]\partial_{t}z(\tilde{t},\omega),
\end{equation}
for some $\tilde{t} \in [t,t+s]$. But then, owing to Eq.\eqref{eq:model_dyn} and the definition of $S$, 
\begin{equation}
  \label{eq:Lz_e}
  \sup_{0 \leq t \leq S, \thickspace \omega \in \Omega} \norm{\partial_{t}z_{\varepsilon}(t,\omega)}_{2} \leq \sum_{i=1}^{m}\frac{u_{i}(\varepsilon)}{v_{i}(\varepsilon)}C_{g_{i}}(z_{0},R)+C_{f}(z_{0},R),
\end{equation}
so that, letting $L_{z_{\varepsilon}}(z_{0},R)$ denote the right-hand-side of the above inequality, we find
\begin{align}
  \norm{\int_{0}^{\delta}f(z_{\varepsilon}(t+s,\omega))ds-\delta f(z_{\varepsilon}(t,\omega))}_{2} &\leq \sqrt{\delta}\left(\int_{0}^{\delta}s^{2}\norm{Df[z(\tilde{t}(s),\omega)]}_{op}^{2}\norm{\partial_{t}z(\tilde{t}(s),\omega)}_{2}^{2}ds\right)^{1/2} \\
  &\leq \frac{\delta^{2}}{\sqrt{3}}C'_{f}(z_{0},R)L_{z_{\varepsilon}}(z_{0},R).
  \label{eq:b1_vz_e}
\end{align}
Letting, for any $1 \leq i \leq m$, $\beta_{\varepsilon,i}(t) = \int_{\Omega}\frac{h_{i}(z_{\varepsilon}(t,\omega))}{v_{i}(\varepsilon)}d\rho(\omega)$, we turn to bounding the second term in Eq.\eqref{eq:exp_vz_e}. Starting again with a first order 
Taylor expansion, we find
\begin{equation}
  g_{i}(z_{\varepsilon}(t+s),\omega) = g_{i}(z_{\varepsilon}(t,\omega))+sDg_{i}[z_{\varepsilon}(\tilde{t},\omega)]\partial_{t}z_{\varepsilon}(\tilde{t},\omega),
\end{equation}
for some $\tilde{t}(s) \in [t,t+\delta]$. Then, 
\begin{align}
  &\norm{\int_{0}^{\delta}g_{i}(z_{\varepsilon}(t+s,\omega))\beta_{\varepsilon,i}(t+s)ds-g_{i}(z_{\varepsilon}(t,\omega))\int_{0}^{\delta}\beta_{\varepsilon,i}(t+s)ds}_{2}, \notag \\
  &\leq \norm{\int_{0}^{\delta}sDg_{i}[z_{\varepsilon}(\tilde{t}(s),\omega)]\partial_{t}z_{\varepsilon}(\tilde{t}(s),\omega)\beta_{\varepsilon,i}(t+s)ds}_{2}, \notag \\
  &\leq \frac{\delta^{2}}{\sqrt{3}}C'_{g_{i}}(z_{0},R)L_{z_{\varepsilon}}(z_{0},R)\frac{u_{i}(\varepsilon)}{v_{i}(\varepsilon)},
  \label{eq:b2_vz_e}
\end{align}
where we have used the upper bound \eqref{eq:Lz_e} and the fact that $\sup_{0 \leq t \leq S}\vert \beta_{\varepsilon,i}(t) \vert \leq \frac{u_{i}(\varepsilon)}{v_{i}(\varepsilon)}$.
Combining the bounds \eqref{eq:b1_vz_e} and \eqref{eq:b2_vz_e}, along with the triangle inequality, we reach  
\begin{equation}
  \label{eq:final_vz_e}
  \norm{v_{z_{\varepsilon}}(t,\omega)}_{2} \leq \frac{\delta^{2}}{\sqrt{3}}L_{z_{\varepsilon}}(z_{0},R)\left(\sum_{i=1}^{m}\frac{u_{i}(\varepsilon)}{v_{i}(\varepsilon)}C'_{g_{i}}(z_{0},R)+C'_{f}(z_{0},R)\right).
\end{equation}
Turning to $v_{z}$, we begin by noting that, for any $t \leq S$, 
\begin{align}
  \norm{\sum_{i=1}^{m}\alpha_{i}(t)g_{i}(z(t,\omega))}_{2} &\leq \sum_{i=1}^{m}\vert \alpha_{i}(t) \vert C_{g_{i}}(z_{0},R) \\
  &\leq \norm{\boldsymbol{\alpha}(t)}_{2} \left(\sum_{i=1}^{m}C_{g_{i}}(z_{0},R)^{2}\right)^{1/2},
\end{align} 
owing to the Cauchy-Schwarz inequality.
But then, recalling that $\boldsymbol{\alpha}(t) = -\boldsymbol{G}(t)^{-1}\boldsymbol{F}(t)$, we find 
\begin{equation}
  \label{eq:c_alpha}
  \norm{\boldsymbol{\alpha}(t)}_{2} \leq \frac{\norm{\boldsymbol{F}(t)}_{2}}{a},
\end{equation}
where, recalling the definition of $\boldsymbol{F}(t) := \boldsymbol{F}(z(t,.))$, we have the upper bound 
\begin{equation}
  \norm{\boldsymbol{F}(t)}_{2} \leq C_{f}(z_{0},R)\left(\sum_{i=1}^{m}C'_{h_{i}}(z_{0},R)^{2}\right)^{1/2}
\end{equation}
Letting $C_{F}$ denote the right-hand-side of the above inequality and $C_{\alpha} = \frac{C_{F}}{a}$, we reach
\begin{align}
  \label{eq:Lz}
  \sup_{0 \leq t \leq S, \omega \in \Omega} \norm{\partial_{t}z(t,\omega)}_{2} &\leq C_{\alpha}\left(\sum_{i=1}^{m}C_{g_{i}}(z_{0},R)^{2}\right)^{1/2}+C_{f}(z_{0},R).
\end{align}
Letting $L_{z}(z_{0},R)$ denote the right-hand-side of inequality \eqref{eq:Lz}, an argument similar to the one used to obtain the bound \eqref{eq:b1_vz_e} gives 
\begin{equation}
  \norm{\int_{0}^{\delta}f(z(t+s),\omega)ds-\delta f(z(t,\omega))}_{2} \leq \frac{\delta^{2}}{\sqrt{3}}C'_{f}(z_{0},R)L_{z}(z_{0},R).
\end{equation}
Now, for any $0 < t,t' \leq S$ , the triangle inequality gives 
\begin{align}
  \norm{\boldsymbol{\alpha}(t)-\boldsymbol{\alpha}(t')}_{2} \leq \norm{\boldsymbol{G}^{-1}(t)-\boldsymbol{G}^{-1}(t')}_{op}\norm{\boldsymbol{F}(t)}_{2}+\norm{\boldsymbol{G}^{-1}(t)}_{op}\norm{\boldsymbol{F}(t)-\boldsymbol{F}(t')}_{2}.
\end{align}
Then, for any $1 \leq i \leq m$, the fact that $f,\nabla h_{i}$ are bounded Lipschitz (on $\mathcal{B}(z_{0},R)$) gives,
\begin{align}
  \big\vert \left[\boldsymbol{F}(t)-\boldsymbol{F}(t')\right]_{i} \big \vert &\leq \int_{\Omega}\big \vert \left(\langle f(z(t,\omega)),\nabla h_{i}(z(t,\omega)) \rangle-\langle f(z(t',\omega)),\nabla h_{i}(z(t',\omega)) \rangle\right)\big \vert d\rho(\omega) \\
  &\leq \left(C'_{f}C'_{h_{i}}+C_{f}C''_{h_{i}}\right)L_{z}\vert t-t'\vert,
\end{align}
leading to 
\begin{equation}
  \norm{\boldsymbol{F}(t)-\boldsymbol{F}(t')}_{2} \leq \left(\sum_{i=1}^{m}\left(C'_{f}C'_{h_{i}}+C_{f}C''_{h_{i}}\right)^{2}\right)^{1/2}L_{z}\vert t-t' \vert.
\end{equation}
Then, noting that $\boldsymbol{G}^{-1}(t)-\boldsymbol{G}^{-1}(t') = \boldsymbol{G}^{-1}(t')\left(\boldsymbol{G}(t')-\boldsymbol{G}(t)\right)\boldsymbol{G}^{-1}(t)$
\begin{align}
  \norm{\boldsymbol{G}^{-1}(t)-\boldsymbol{G}^{-1}(t')}_{op} &\leq \frac{1}{a^{2}}\norm{\boldsymbol{G}(t)-\boldsymbol{G}(t')}_{op} \\
  &\leq \frac{m}{a^{2}}\sup_{1 \leq i,j \leq m}\vert\left[\boldsymbol{G}(t)-\boldsymbol{G}(t')\right]_{ij} \vert \\
  &\leq \frac{m}{a^{2}}\sup_{1 \leq i,j \leq m}\left(C_{g_{i}}C''_{h_{j}}+C'_{g_{i}}C'_{h_{j}}\right)L_{z}\vert t-t' \vert.
\end{align}
Letting
\begin{equation}
  \label{eq:bound_dalpha}
  C'_{\alpha} = \left(\frac{1}{a}\left(\sum_{i=1}^{m}\left(C'_{f}C'_{h_{i}}+C_{f}C''_{h_{i}}\right)^{2}\right)^{1/2}+\frac{m}{a^{2}}\sup_{1 \leq i,j \leq n}\left(C_{g_{i}}C''_{h_{j}}+C'_{g_{i}}C'_{h_{j}}\right)C_{F}\right),
\end{equation}
we reach
\begin{align}
  \label{eq:bound_dalpha_f}
\norm{\boldsymbol{\alpha}(t)-\boldsymbol{\alpha}(t')}_{2} \leq C'_{\alpha}L_{z}\vert t-t' \vert.
\end{align}
Then, for any $1 \leq i \leq m$, we write
\begin{align*}
  \alpha_{i}(t+s)g_{i}(z(t+s),\omega)-\alpha_{i}(t)g_{i}(z(t,\omega)) &= \alpha_{i}(t+s)g_{i}(z(t+s),\omega)-\alpha_{i}(t+s)g_{i}(z(t),\omega), \\
  &+\alpha_{i}(t+s)g_{i}(z(t,\omega))-\alpha_{i}(t)g_{i}(z(t,\omega)),
\end{align*}
so that 
\begin{equation}
\norm{\alpha_{i}(t+s)g_{i}(z(t+s),\omega)-\alpha_{i}(t)g_{i}(z(t,\omega))}_{2} \leq \vert \alpha_{i}(t+s) \vert C'_{g_{i}}L_{z}s+C_{g_{i}}\vert \alpha_{i}(t+s)-\alpha_{i}(t) \vert,
\end{equation}
and
\begin{align*}
  &\norm{\sum_{i=1}^{m}\left(\int_{0}^{\delta}\alpha_{i}(t+s)g_{i}(z(t+s,\omega))ds-\delta\alpha_{i}(t)g_{i}(z(t,\omega))\right)}_{2} \leq \\
  &\sqrt{2\delta}\sum_{i=1}^{m}\left(\int_{0}^{\delta}\left(\vert \alpha_{i}(t+s)|C'_{g_{i}}L_{z}s\right)^{2}ds +\int_{0}^{\delta}\left(C_{g_{i}}\vert \alpha_{i}(t+s)-\alpha_{i}(t) \vert\right)^{2}ds\right)^{1/2} \\
  &\leq \sqrt{2\delta m}\left(\frac{\delta^{3}}{3}L_{z}^{2}\left(\sup_{1 \leq i \leq n}C'_{g_{i}}\right)^{2}\sum_{i=1}^{m}\alpha_{i}(t+s)^{2}+\left(\sup_{1 \leq i \leq n}C_{g_{i}}\right)^{2}\int_{0}^{\delta}\norm{\boldsymbol{\alpha}(t+s)-\boldsymbol{\alpha}(t)}_{2}^{2}ds\right)^{1/2} \\
  &\leq \frac{\sqrt{2m}}{\sqrt{3}}\delta^{2}\left(C_{\alpha}^{2}L_{z}^{2}\left(\sup_{1 \leq i \leq n}C'_{g_{i}}\right)^{2}+\left(\sup_{1 \leq i \leq n}C_{g_{i}}\right)^{2}(C'_{\alpha})^{2}\right)^{1/2}
\end{align*}
We finally reach 
\begin{equation}
  \label{eq:final_vz}
  \norm{v_{z}(t,\omega)}_{2} \leq \frac{\delta^{2}}{\sqrt{3}}C'_{f}(z_{0},R)L_{z}(z_{0},R)+\frac{\delta^{2}\sqrt{2m}}{\sqrt{3}}\left(C_{\alpha}^{2}L_{z}^{2}\left(\sup_{1 \leq i \leq n}C'_{g_{i}}\right)^{2}+\left(\sup_{1 \leq i \leq n}C_{g_{i}}\right)^{2}(C'_{\alpha}L_{z})^{2}\right)^{1/2}.
\end{equation}
In what follows, we will respectively denote $K_{v_{z_{\varepsilon}}}$ and $K_{v_{z}}$ the bounds obtained at Eq.\eqref{eq:final_vz_e} and Eq.\eqref{eq:final_vz}. We now turn to showing that the vector $\boldsymbol{\alpha}(t)$ is close 
to the vector $\boldsymbol{\alpha}_{\varepsilon}(t)$ whose entries are given by
\begin{equation}
\alpha_{\varepsilon,i}(t) = \frac{1}{\delta}\int_{0}^{\delta}\int_{\Omega}\frac{h_{i}(z_{\varepsilon}(t+s),\omega)}{v_{i}(\varepsilon)}d\rho(\omega)ds, \quad \forall \thickspace 1 \leq i \leq m.
\end{equation}
\paragraph{\textbf{Approximating} \textbf{the} $\alpha_{\varepsilon,i}(t)$ \textbf{with the} $\boldsymbol{\alpha_{i}}(t)$}
Owing to the definitions of $v_{z_{\varepsilon}},v_{z}$, we may write
\begin{align}
  \label{eq:exp_z_e}
  z_{\varepsilon}(t+\delta,\omega) &= z_{\varepsilon}(t,\omega)+\delta f(z_{\varepsilon}(t,\omega))+\delta\sum_{i=1}^{m} \alpha_{\varepsilon,i}(t)g_{i}(z_{\varepsilon}(t,\omega))+v_{z_{\varepsilon}}(t,\omega), \\
  z(t+\delta,\omega) &= z(t,\omega)+\delta f(z(t,\omega))+\delta\sum_{i=1}^{m}\alpha_{i}(t)g_{i}(z(t,\omega))+v_{z}(t,\omega).
  \label{eq:exp_z}
\end{align}
Then, for any $1 \leq i \leq m$, an order two Taylor expansion gives that 
\begin{align}
    \label{eq:taylor_h_e}
    h_{i}(z_{\varepsilon}(t+\delta,\omega)) = h_{i}(z_{\varepsilon}(t,\omega))&+\langle \nabla h_{i}(z_{\varepsilon}(t,\omega)),z_{\varepsilon}(t+\delta,\omega)-z_{\varepsilon}(t,\omega) \rangle, \notag \\
    &+\frac{1}{2}\nabla^{2}h_{i}(\tilde{z}_{\varepsilon})\left[(z_{\varepsilon}(t+\delta,\omega)-z_{\varepsilon}(t,\omega))^{\otimes 2}\right],
  \end{align}
for some $\tilde{z}_{\varepsilon} \in \mathcal{B}(z_{\varepsilon}(t,\omega),\norm{z_{\varepsilon}(t+\delta,\omega)-z_{\varepsilon}(t,\omega)}_{2}) \subset \mathcal{B}(z_{\varepsilon,0}(\omega),3R)$.
Letting
\begin{equation}
  \label{eq:s}
  s_{i}(t,\omega) =  h_{i}(z_{\varepsilon}(t+\delta,\omega)) - h_{i}(z_{\varepsilon}(t,\omega)) - \frac{1}{2}\nabla^{2}h_{i}(\tilde{z}_{\varepsilon})\left[(z_{\varepsilon}(t+\delta,\omega)-z_{\varepsilon}(t,\omega))^{\otimes 2}\right]
\end{equation}
and replacing the increment $z_{\varepsilon}(t+\delta,\omega)-z_{\varepsilon}(t,\omega)$ using Eq.\eqref{eq:exp_z_e} in Eq.\eqref{eq:taylor_h_e}, and integrating with respect to $\omega$, we find, for any 
$1 \leq j \leq m$,
\begin{align}
  \label{eq:def_alpha_eps}
  \int_{\Omega}\langle \nabla h_{j}(z_{\varepsilon}(t,\omega)), \delta f(z_{\varepsilon}(t,\omega))+\delta \sum_{i=1}^{m}\alpha_{\varepsilon,i}(t)g_{i}(z_{\varepsilon}(t,\omega))+v_{z_{\varepsilon}}(t,\omega) \rangle d\rho(\omega) = \int_{\Omega} s_{j}(t,\omega)d\rho(\omega).
\end{align}
Now, defining the vectors $\boldsymbol{s}(t)$ and $\boldsymbol{V}_{z_{\varepsilon}}(t)$, whose entries are respectively given by 
\begin{align}
  &[\boldsymbol{s}(t)]_{1 \leq i \leq m} = \int_{\Omega}s_{i}(t,\omega)d\rho(\omega), \\
  &[\boldsymbol{V}_{z_{\varepsilon}}(t)]_{1 \leq i \leq m} = \int_{\Omega}\langle \nabla h_{i}(z_{\varepsilon}(t,\omega)),v_{z_{\varepsilon}}(t,\omega) \rangle d\rho(\omega),
\end{align}
we may rewrite Eq.\eqref{eq:def_alpha_eps} as 
\begin{equation}
  \delta \left(\boldsymbol{G}(z_{\varepsilon}(t,.))\boldsymbol{\alpha}_{\varepsilon}(t)+\boldsymbol{F}(z_{\varepsilon}(t,.))\right) = -\boldsymbol{V}_{z_{\varepsilon}}(t)+\boldsymbol{s}(t). 
\end{equation}
Abbreviating $\boldsymbol{G}(z_{\varepsilon}(t,.))$ with $\boldsymbol{G}^{\varepsilon}(t)$ and $\boldsymbol{F}(z_{\varepsilon}(t,.))$ with $\boldsymbol{F}^{\varepsilon}(t)$, and letting $\tilde{\boldsymbol{\alpha}}_{\varepsilon}(t) = - \boldsymbol{G}^{\varepsilon}(t)^{-1}\boldsymbol{F}^{\varepsilon}(t)$, we may write 
\begin{align}
\norm{\boldsymbol{\alpha}(t)-\boldsymbol{\alpha}_{\varepsilon}(t)}_{2} &\leq \norm{\boldsymbol{\alpha}(t)-\tilde{\boldsymbol{\alpha}}_{\varepsilon}(t)}_{2}+\frac{1}{\delta}\norm{\boldsymbol{G}^{\varepsilon}(t)^{-1}\left(-\boldsymbol{V}_{z_{\varepsilon}}(t)+\boldsymbol{s}(t)\right)}\notag \\
&\leq \norm{\boldsymbol{\alpha}(t)-\tilde{\boldsymbol{\alpha}}_{\varepsilon}(t)}_{2}+\frac{1}{\delta a}\left(\norm{\boldsymbol{V}_{z_{\varepsilon}}(t)}_{2}+\norm{\boldsymbol{s}(t)}_{2}\right).
\end{align}
Now, using the definition of $s_{i}(t)$ at Eq.\eqref{eq:s}, we find
\begin{equation}
  \sup_{0 \leq t \leq S} \vert \int_{\Omega}s_{i}(t,\omega)d\rho(\omega) \vert \leq 2u_{i}(\varepsilon)+C''_{h_{i}}L^{2}_{z}\delta^{2},
\end{equation}
and thus 
\begin{equation}
  \norm{\boldsymbol{s}(t)}_{2} \leq \left(\sum_{i=1}^{m}\left(2u_{i}(\varepsilon)+C''_{h_{i}}L^{2}_{z}\delta^{2}\right)^{2}\right)^{1/2}.
\end{equation}
Also,
\begin{equation}
  \norm{\boldsymbol{V}_{z_{\varepsilon}}(t)}_{2} \leq K_{v_{z_{\varepsilon}}}\left(\sum_{i}^{m}(C'_{h_{i}}(z_{0},R))^{2}\right)^{1/2}.
\end{equation}
Furthermore, the argument employed to establish Eq.\eqref{eq:bound_dalpha_f} also implies that
\begin{equation}
  \norm{\boldsymbol{\alpha}(t)-\tilde{\boldsymbol{\alpha}}_{\varepsilon}(t)}_{2} \leq C'_{\alpha}\norm{z(t)-z_{\varepsilon}(t)}_{2}.
\end{equation}
In summary, we reach 
\begin{align}
  \norm{\boldsymbol{\alpha}(t)-\boldsymbol{\alpha}_{\varepsilon}(t)}_{2} \leq &\frac{1}{\delta a}\left(\left(\sum_{i=1}^{m}\left(2u_{i}(\varepsilon)+C''_{h_{i}}L^{2}_{z}\delta^{2}\right)^{2}\right)^{1/2}+K_{v_{z_{\varepsilon}}}\left(\sum_{i}^{m}(C'_{h_{i}}(z_{0},R))^{2}\right)^{1/2}\right) \notag\\
  &+C'_{\alpha}\norm{z(t,\omega)-z_{\varepsilon}(t,\omega)}_{2}.
\end{align}
\paragraph{\textbf{Concluding the proof}}
To conclude the proof, we use Eq.\eqref{eq:exp_z_e},\eqref{eq:exp_z} to obtain that, for all $\omega \in \Omega$,
\begin{align}
  \norm{z_{\varepsilon}(t+\delta,\omega)-z(t+\delta,\omega)}_{2} &\leq \left(1+\delta C'_{f}\right)\norm{z_{\varepsilon}(t,\omega)-z(t,\omega)}_{2}+K_{v_{z}}+K_{v_{z_{\varepsilon}}},\notag \\
  &\hspace{2cm}+\delta \norm{\sum_{i=1}^{m}\left(\alpha_{\varepsilon,i}(t)g_{i}(z_{\varepsilon}(t,\omega))-\alpha_{i}(t)g_{i}(z(t,\omega))\right)}_{2}.
\end{align}
But then,
\begin{align*}
 &\norm{\sum_{i=1}^{m}\left(\alpha_{\varepsilon,i}(t)g_{i}(z_{\varepsilon}(t,\omega))-\alpha_{i}(t)g_{i}(z(t,\omega))\right)}_{2} \leq \sum_{i=1}^{m} \norm{\vert \alpha_{\varepsilon,i}(t)-\alpha_{i}(t)\vert g_{i}(z_{\varepsilon}(t,\omega))}_{2} \\
  &\hspace{7cm}+\sum_{i=1}^{m} \norm{\alpha_{i}(t)\left(g_{i}(z_{\varepsilon}(t,\omega))-g_{i}(z(t,\omega))\right)}_{2} \\
  &\leq \left(\sum_{i=1}^{m}C_{g_{i}}^{2}\right)^{1/2}\norm{\boldsymbol{\alpha}(t)-\boldsymbol{\alpha}_{\varepsilon}(t)}_{2}+C_{\alpha}\left(\sum_{i=1}^{m}(C'_{g_{i}})^{2}\right)^{1/2}\norm{z(t,\omega)-z_{\varepsilon}(t,\omega)}_{2}.
\end{align*}
Gathering the above bounds, we introduce the quantities 
\begin{align}
  K &= \left(C'_{f}+C_{\alpha}\left(\sum_{i=1}^{m}(C'_{g_{i}})^{2}\right)^{1/2}+C'_{\alpha}\left(\sum_{i=1}^{m}C_{g_{i}}^{2}\right)^{1/2}\right), \\
  \ell(\delta,\varepsilon) &= K_{v_{z}}+K_{v_{z_{\varepsilon}}}+\frac{\left(\sum_{i=1}^{m}C_{g_{i}}^{2}\right)^{1/2}}{a}\left(\left(\sum_{i=1}^{m}\left(2u_{i}(\varepsilon)+C''_{h_{i}}L^{2}_{z}\delta^{2}\right)^{2}\right)^{1/2}+K_{v_{z_{\varepsilon}}}\left(\sum_{i}^{m}(C'_{h_{i}}(z_{0},R))^{2}\right)^{1/2}\right),
\end{align}
so that we finally obtain the following control:
\begin{align}
  &\norm{z_{\varepsilon}(t+\delta,\omega)-z(t+\delta,\omega)}_{2} \leq \left(1+K\delta\right)\norm{z_{\varepsilon}(t,\omega)-z(t,\omega)}_{2}+\ell(\delta,\varepsilon).
\end{align}
Applying the discrete Gronwall inequality then gives that, for any $n \in \mathbb{N}$ such that $n\delta \leq S$, 
\begin{equation}
  \norm{z_{\varepsilon}(n\delta,\omega)-z(n\delta,\omega)}_{2} \leq \left(\norm{z_{\varepsilon}(0,\omega)-z(0,\omega)}_{2}+nl(\delta,\varepsilon)\right)\exp\left(Kn\delta\right).
\end{equation}
\end{proof}
We now specialize this result to a family of dynamical systems related to \eqref{eq:main_mf_dyn1}-\eqref{eq:main_mf_dyn2}, using the same notations. For any integer $D$, consider the following evolution:
  \begin{align}
    \label{eq:gen_dyn1}
    \partial_{t}a(t,\omega) &= \sum_{k=0}^{D-1}\sigma_{k}s^{k}(t,\omega)\frac{\left(\varphi_{k}-\sigma_{k}\overline{as^{k}}(t)\right)}{v_{k}(\varepsilon)}+\sigma_{D}\varphi_{D}s^{D}(t,\omega), \\
    \partial_{t}s(t,\omega) &= \sum_{k=1}^{D-1}k\sigma_{k}a(t,\omega)s^{k-1}(t,\omega)\frac{\left(\varphi_{k}-\sigma_{k}\overline{as^{k}}(t)\right)}{v_{k}(\varepsilon)}+D\sigma_{D}\varphi_{D}a(t,\omega)s^{D-1}(t,\omega),
    \label{eq:gen_dyn2}
  \end{align}
  with initial condition $a(0,.),s(0,.) \in \mathcal{C}_{M}(\Omega,\mathbb{R})$. For any pair $z_{1}(t,.),z_{2}(t,.) \in \mathcal{C}(\Omega,\mathbb{R})$,
  define the p-dimensional vector $\boldsymbol{\alpha}(z_{1}(t,.),z_{2}(t,.))$ as any solution to the linear system 
  \begin{align}
    \label{eq:alpha_nn}
    &\forall \thickspace 0 \leq p \leq D-1 : \notag\\
    &\sum_{k=0}^{D-1}\alpha_{k}(t)\left(\sigma_{k}\overline{z_{2}^{p+k}}(t)+pk\sigma_{k}\overline{z_{1}^{2}z_{2}^{p+k-2}}(t)\right) = -\sigma_{D}\varphi_{D}\left(\overline{z_{2}^{p+D}}(t)+Dp\overline{z_{1}^{2}z_{2}^{D+p-2}}(t)\right).
  \end{align}
  Assume that there exists a pair $z_{1}(0,.),z_{2}(0,.)$ such that the system \eqref{eq:alpha_nn} has a unique solution, and consider the dynamical system 
  \begin{align}
    \partial_{t}z_{1}(t,\omega) &= \sum_{k=0}^{D-1}\sigma_{k}z_{2}^{k}(t,\omega)\alpha_{k}(t)+\sigma_{D}\varphi_{D}z_{2}^{D}(t,\omega), \\
    \partial_{t}z_{2}(t,\omega) &= \sum_{k=1}^{D-1}k\sigma_{k}z_{1}(t,\omega)z_{2}^{k-1}(t,\omega)\alpha_{k}(t)+D\sigma_{D}\varphi_{D}z_{1}(t,\omega)z_{2}^{D-1}(t,\omega),
  \end{align}
  initialized with $z_{1}(0,.),z_{2}(0,.)$. Finally, define $z_{\varepsilon}(t,\omega) = (a(t,\omega),s(t,\omega))$ and $z(t,\omega) = (z_{1}(t,\omega),z_{2}(t,\omega))$.
\begin{cor}
  \label{cor:nn_approx_dyn}
  With the same notations as in Proposition \ref{prop:approx_dyn_set}, assume that there exists a constant $C_{0}$ (independent on $\varepsilon$) such that, for each $0 \leq k \leq D-1$:
  \begin{equation}
    u_{k}(\varepsilon) \leq C_{0}v_{k}(\varepsilon),
  \end{equation}
  and denote $u^{*}(\varepsilon) = \sup_{k}\{u_{k}(\varepsilon)\}_{k=0}^{D-1}$.
  Then, there exist constants $C_{1},C_{2}$ depending only on $K,R,C_{0}$ and the coefficients $\{\sigma_{k},\varphi_{k}\}_{k=0}^{D}$ such that, for any $0 \leq t \leq S$,
  \begin{equation}
  \norm{z_{\varepsilon}(t,.)-z(t,.)} \leq \left(\norm{z_{\varepsilon}(0,.)-z(0,.)}+C_{1}\sqrt{u^{*}(\varepsilon)}\right)\exp\left(C_{2}t\right).
\end{equation}
\end{cor}
\begin{proof}
  We apply Proposition \ref{prop:approx_dyn_set} to the evolution 
  \eqref{eq:gen_dyn1}-\eqref{eq:gen_dyn2}, the form of which implies that there exist constants $C_{1},C_{2},C_{3}$ depending only on $D,R,\{\sigma_{k},\varphi_{k}\}_{k=0}^{D}$ 
  such that $K\leq C_{1}$ and $l(\delta,\varepsilon) \leq C_{2}\delta^{2}+C_{3}u^{*}(\varepsilon)$. Choosing $n = \lfloor \frac{t}{\delta} \rfloor$, we find, for any $\delta<1$,
  \begin{equation}
    \norm{z_{\varepsilon}(t,.)-z(t,.)} \leq \left(\norm{z_{\varepsilon}(0,.)-z(0,.)}+\frac{t}{\delta}(C_{1}\delta^{2}+C_{3}u^{*}(\varepsilon))\right)\exp\left(C_{2}t\right).
  \end{equation}
  Optimizing over $\delta$ then gives the desired result.
\end{proof}
\section{Auxiliary lemmas}
\begin{lem}
  \label{lem:Lip_sup}
  Let $f : \mathbb{R}_{+} \to \mathbb{R}$ be a Lipschitz continuous function and let $F(t) = \sup_{r \leq t} \vert f(r) \vert$. Then $F$ is Lipschitz continuous. 
  Furthermore, if there exists a positive, increasing function $g$ such that, for every $t \geq 0$, $\vert \partial_{t}f(t) \vert \leq g(t)$, then $\partial_{t} F(t) \leq g(t)$ for 
  (Lebesgue) almost every $t$.
\end{lem}
\begin{proof}
  For any $t \geq 0$ and since $f$ is continuous, we may define $r_{t} = \argsup_{r \leq t} \vert f(r) \vert$. Furthermore, let 
  \begin{align}
    \tilde{r} = \begin{cases}
                &r_{t+s} \quad \mbox{if} \thickspace r_{t+s} \leq t \\
                &t \quad \mbox{otherwise}
        \end{cases}
  \end{align}
  Then, $F(t) \geq \vert f(\tilde{r}) \vert$ and therefore 
  \begin{align}
    F(t) \geq \vert f(r_{t+s})\vert+\vert f(\tilde{r}) \vert- \vert f(r_{t+s}) \vert.
  \end{align}
  By definition $\vert f(r_{t+s})\vert = F(t+s)$ and $r_{t+s}-\tilde{r} \leq s$, and, by the Lipschitz assumption, $ \big \vert \vert f(\tilde{r}) \vert -\vert f(r_{t+s}) \vert \big \vert \leq Ls$ for some constant $L$. Upon noting 
  that $F(t+s) \geq F(t)$, this completes the proof of the first part of the statement. Now, since $F$ is Lipschitz continuous, its derivative exists (Lebesgue) almost everywhere and, 
  we have 
  \begin{align}
    \partial_{t} F(t) = \begin{cases}
                &0 \quad \mbox{if} \thickspace r_{t} < t \\
                &\partial_{t}f(t) \quad \mbox{if} \thickspace r_{t} = t.
        \end{cases}
  \end{align}
\end{proof}
\begin{lem}
  \label{lem:Ber_1/2}
  For any positive real numbers $\alpha, \beta$, consider the ODE 
  \begin{align}
    &\dot{x}(t) = -\alpha x(t)+\beta\sqrt{x}(t), \\
    &x(0) > 0.
  \end{align}
  Then, for all $t \geq 0$,
  \begin{equation}
    x(t) = \left(\sqrt{x(0)}\exp(-\frac{\alpha}{2}t)+\frac{\beta}{\alpha}\left(1-\exp(-\frac{\alpha}{2}t)\right)\right)^{2}.
  \end{equation}
\end{lem}
\begin{proof} 
  This is a Bernoulli equation, see e.g.\cite{teschl2012ordinary}, and, clearly, $x(t)>0$ for any $t>0$, so that existence and uniqueness of maximal solutions follows straightforwardly. Making the change of 
  variable $y(t) = \sqrt{x(t)}$, we find 
  \begin{equation}
    \dot{y}(t) = -\frac{\alpha}{2} y(t)+\frac{\beta}{2},
  \end{equation}
  so that 
  \begin{equation}
    y(t) = y(0)\exp(-\frac{\alpha}{2}t)+\frac{\beta}{\alpha}\left(1-\exp(-\frac{\alpha}{2}t)\right).
  \end{equation}
\end{proof}
\begin{lem}
  Let $z : \mathbb{R}_{+} \to \mathbb{R}^{d}$ be a continuous function. Then, for any $t \geq 0$,
  \begin{align}
    \norm{\int_{0}^{t}z(s)ds}_{2} &\leq \int_{0}^{t}\norm{z(s)}_{2}ds, \\
    \norm{\int_{0}^{t}z(s)ds}_{2} & \leq \sqrt{t}\left(\int_{0}^{t}\norm{z(s)}_{2}^{2}ds\right)^{1/2}.
  \end{align}
\end{lem}
\begin{proof}
  The first inequality is a consequence of the convexity of the Euclidean norm and Jensen's inequality. For the second inequality, we apply the Cauchy-Schwarz inequality to obtain
  \begin{align}
    \norm{\int_{0}^{t}z(s)ds}_{2}^{2} \leq \sum_{i=1}^{d}t\int_{0}^{t}z^{2}_{i}(s)ds.
  \end{align}
\end{proof}
\bibliographystyle{siam}
\bibliography{refs}
\end{document}